\documentclass[10pt,a4paper,twocolumn]{article} \pdfoutput = 1
\usepackage{geometry}
\usepackage[breaklinks=true,colorlinks=true,urlcolor=blue,linkcolor=blue,citecolor=blue, pdfborder={0 0 0}]{hyperref}

\usepackage{graphicx}
 \usepackage{mathptmx}      
\usepackage{amsmath}
\usepackage{amssymb}
\usepackage{amsthm}
\usepackage{multirow,mdwlist}
\usepackage{tikz}
\usetikzlibrary{arrows,automata}
\usepackage{graphicx}
\usepackage{wrapfig}
\usepackage[scriptsize]{subfigure} 
\usepackage{algorithm}
\usepackage{algorithmic}
\usepackage{breakcites}
\usepackage{cuted}
\usepackage{todonotes}
\usepackage{url}
\usepackage{authblk}

\newtheorem{assumption}{Assumption}

\newtheorem{definition}{Definition}
\newtheorem{proposition}{Proposition}
\newtheorem{theorem}{Theorem}
\newtheorem{corollary}{Corollary}
\newenvironment{remark}[1][Remark]{\begin{trivlist}
\item[\hskip \labelsep {\bfseries #1}]}{\end{trivlist}}

\title{An Efficient Decomposition Framework for Discriminative Segmentation with Supermodular Losses
}

\author[1]{Jiaqian~Yu\thanks{jiaqian.yu@centralesupelec.fr}}
\author[2]{Matthew~B.~Blaschko\thanks{matthew.blaschko@esat.kuleuven.be}}
\affil[1]{{\small Center for Visual Computing, CentraleSup\'{e}lec \& Inria, Universit\'{e} Paris-Saclay}}
\affil[2]{{\small Center for Processing Speech and Images, Dept.\ Elektrotechniek, KU Leuven}}
\date{}

\begin{document}
\maketitle

\begin{abstract}
Several supermodular losses have been shown to improve the perceptual quality of image segmentation in a discriminative framework such as a structured output support vector machine (SVM).  These loss functions do not necessarily have the same structure as the one used by the segmentation inference algorithm, and in general, we may have to resort to generic submodular minimization algorithms for loss augmented inference.  Although these come with polynomial time guarantees~\cite{doi:10.1137/S0097539701397813,fujishige1980,fujishige2005submodular}, 
they are not practical to apply to image scale data.  Many supermodular losses come with strong optimization guarantees, but are not readily incorporated in a loss augmented graph cuts procedure.  This motivates our strategy of employing the alternating direction method of multipliers (ADMM) decomposition for loss augmented inference.  In doing so, we create a new API for the structured SVM that separates the maximum a posteriori (MAP) inference of the model from the loss augmentation during training.  In this way, we gain computational efficiency, making new choices of loss functions practical for the first time, while simultaneously making the inference algorithm employed during training closer to the test time procedure. 
   We show improvement both in accuracy and computational performance on the Microsoft Research Grabcut database and a brain structure segmentation task, empirically validating the use of several supermodular loss functions during training, and the improved computational properties of the proposed ADMM approach over the Fujishige-Wolfe minimum norm point algorithm.

\end{abstract}

\section{Introduction}

Discriminative structured prediction is a valuable tool in computer vision that has been applied to a wide range of application areas, and in particular object detection and segmentation \cite{anguelov2005discriminative,blaschko2008learning,nowozin2011structured,osokin2014perceptually,pletscher2012learning,szummer2008learning}.  It is frequently applied using variants of the structured output support vector machine (SVM) \cite{taskar03max,tsochantaridis2005large} in which a domain specific discrete loss function is upper bounded by a piecewise linear surrogate.  In the case of image segmentation, this discrete loss function has frequently been taken to be the Hamming loss, which simply counts the number of incorrect pixels \cite{anguelov2005discriminative,szummer2008learning}.  Following the principle of empirical risk minimization, one might expect that minimization of the desired loss at training time would lead to the best performing loss at test time.  However, it has recently been shown that in the finite sample regime, minimizing a different loss can lead to better performance even when measured using Hamming loss \cite{osokin2014perceptually}.  In that work, a supermodular loss function was employed, and a custom graph cuts solution was found to the loss augmented inference problem necessary for computation of a subgradient or cutting plane of the learning objective \cite{joachims2009cutting}.

Several non-modular loss functions have been considered in the context of image segmentation, e.g.\ the intersection over union loss in the context of a Bayesian framework \cite{nowozin2014optimal}, an area/volume based label-count loss that enforces high-order statistics \cite{pletscher2012learning}, or a layout-aware loss function that takes into account the topology/structure of the object \cite{osokin2014perceptually}.
A message passing based optimization scheme is proposed for optimizing several families of structured loss functions \cite{tarlow2010hop,tarlow2012structured}, which assumes the loss function is constructed by a grammar for which the productions specify function composition \cite{tarlow2010hop}. 
By contrast, we provide a generic framework for decomposing the loss function from model inference that assumes a custom solver for the loss, but that does not assume the loss belongs to a specific compositional grammar.
We concern ourselves primarily with supermodular loss functions in this work as they lead to provable polynomial time loss augmented inference problems (an essential step in training structured output SVMs), while non-supermodular loss functions lead to NP-hard optimization in general.

In general, it is a time consuming process to develop custom loss-augmented solvers for different combinations of loss functions and inference procedures.  We show in this work a direct combination of two submodular graph cuts procedures may in fact lead to a non-submodular minimization problem, and reparametrizations or novel graph constructions may be necessary.  Furthermore, if we attempt to solve a non-submodular minimization problem approximately, this may lead to poor convergence of the learning procedure and catastrophic failure of the learning algorithm \cite{finley2008training}.

An alternative approach is to resort to generic submodular optimization algorithms, such as that of Iwata \cite{doi:10.1137/S0097539701397813} which has complexity $\mathcal{O}(n^4 T + n^5 \log M)$, or
Orlin \cite{orlin2009faster} with complexity $\mathcal{O}(n^6 + n^5 T)$, where $T$ is the time for a single function evaluation and $M$ is an upper bound on the absolute value of the function.  Although these optimization algorithms are polynomial, the exponent is sufficiently large as to render them infeasible for images of even less than one megapixel.  In practice, the Fujishige-Wolfe
minimum norm algorithm \cite{fujishige1980,fujishige2005submodular} is empirically faster 
\cite{chakrabarty2014provable}.  However, we will show that even this state of the art  optimization strategy is infeasible for relatively small consumer images.

Specific subclasses of submodular functions come with lower complexity optimization algorithms, and we should be able to exploit these known classes in a general learning framework.  Examples include decomposable submodular functions \cite{stobbe2010efficient,nishihara2014convergence}, several notions of symmetry \cite{kolmogorov2012minimizing,queyranne1998minimizing}, and graph partition problems \cite{kolmogorov2004energy,charpiat:graphcut}.  A problem with the current API for loss augmented inference is that it is assumed that the loss function will decompose with a structure compatible to that of the inference problem.  We address the case that this assumption does not hold and that separate efficient optimization procedures are available for the loss and for inference.

We propose to use Lagrangian splitting techniques to separate loss maximization from the inference problem.  Strategies such as dual decomposition have become popular in Markov Random Field (MRF)  inference \cite{komodakis2007mrf}, while later developments such as the alternating direction method of multipliers (ADMM) \cite{bertsekas1999nonlinear,boyd2011distributed} have improved convergence guarantees. Other strategies involving a quadratic penalty term have also been proposed in the literature, although still with the assumption that the loss decomposes as the inference \cite{meshi2015efficient}. We make use of ADMM to separate these inference problems and apply them to several supermodular loss functions that cannot be straightforwardly incorporated in a submodular graph partition problem for loss augmented inference.  Instead we allow separate optimization strategies for the loss maximization and inference procedures yielding substantially improved computational performance, while making feasible the application of a wide range of supermodular loss functions by changing a single line of code.

This article is an extended version of \cite{Yu2016b} with additional theoretical contributions and experimental results. 
 Specifically, we have added:
 \begin{enumerate}
\item Section~\ref{sec:square}: a supermodular loss function, the square loss function;
\item Section~\ref{sec:Biconvex}: a new section on supermodular loss functions through biconvexity, with the definition of biconvexity, a proposition that biconvexity characterizes supermodularity for an important class of loss functions, and its proof; 
\item Section~\ref{sec:biconvexloss}: a novel supermodular loss function from biconvexity;
\item Section~\ref{sec:admmconverge}: a new section on ADMM convergence theorems;
\item Section~\ref{sec:optimization}: a new section on the optimization algorithm related to the novel supermodular loss functions in Sections~\ref{sec:square} and~\ref{sec:biconvexloss};
\item Tables~\ref{tab:loss} and~\ref{tab:pvalueErr}: additional results with more parameter values;
\item Section~\ref{sec:experimentNew}: new experimental results with the supermodular loss functions introduced in Sections~\ref{sec:square} and~\ref{sec:biconvexloss};
\item Figures~\ref{fig:segmentation2} to \ref{fig:segDiffBiconvex}: additional qualitative segmentation results.
 \end{enumerate}

\section{Methods}\label{sec:method}
We discriminatively train a graph cuts based segmentation system using a structured SVM \cite{tsochantaridis2005large}.  We first construct a supermodular loss function that is solvable with graph cuts, but that when incorporated in a joint loss-augmented inference leads to non-submodular potentials, which causes graph cuts based optimization to fail.  We therefore use an ADMM based decomposition strategy to perform loss augmented inference.  This strategy consists of alternatingly optimizing the loss function and performing maximum
\emph{a posteriori} (MAP) inference, with each process augmented by a quadratic term enforcing the labeling determined by each to converge to the optimum of the sum.

The structured output SVM is a discriminative learning framework that has been applied in diverse computer vision applications \cite{anguelov2005discriminative,blaschko2008learning,nowozin2011structured,osokin2014perceptually,pletscher2012learning,szummer2008learning}.
Given a training set of labeled images $\{(x_1,y^*_1), \dots ,$ $(x_n,y^*_n) \}$ $\in \left(\mathcal{X} \times \mathcal{Y}\right)^{n}$, where $\mathcal{Y}=\{-1,1\}^{p}$ for a binary segmentation problem with $p$ pixels, it optimizes a regularized convex upper bound to a structured loss function, $\Delta : \mathcal{Y} \times \mathcal{Y} \to \mathbb{R}_{+}$. $\Delta$ measures the mismatch between a ground truth labeling and a hypothesized labeling.  
With $\Delta$ provided as an input, the structured SVM with margin rescaling minimizes \cite{tsochantaridis2005large}:
\begin{align}
\min_{w,\xi} \ \  \frac{1}{2} \|w\|^{2} +  C \sum_{i=1}^{n} \xi_i \qquad   &\text{s.t. }\forall i, \tilde{y}_i \in \mathcal{Y},\label{eq:svmmargin} \\ 
 \langle w, \phi(x_i,y_i^{*}) - \phi(x_i,\tilde{y}_i) \rangle& \geq \Delta(y_i^{*},\tilde{y}_i) - \xi_i
\end{align}
In the case of image segmentation, we may interpret $\langle w, \phi(x,y) \rangle$ as a function that is monotonic in the log probability of the joint configuration of observed and unobserved variables $(x,y)$ as determined by a CRF \cite{Lafferty:2001:CRF:645530.655813}.  Under this interpretation, a standard definition of $\phi$ is
\begin{equation}\label{eq:feature}
\phi(x,y) := \begin{pmatrix} \sum_{j=1}^{p} \phi_u(x,y^{j}) \\ \sum_{(k,l) \in \mathcal{E}} \phi_p(x,y^{k},y^{l}) \end{pmatrix}
\end{equation}
where $\phi_u$ determines a vector of features, a linear combination of which form the unary potentials of the CRF, and $\phi_p$ determines the pairwise potentials over a model specific edge set $\mathcal{E}$.  In this work, we have set $\phi_p(x,\cdot,\cdot) : \{-1,1\}^2 \to \{0,1\}^{3}$ to map to an indicator vector of three cases: (i) $y^k = y^l = -1$, (ii) $y^k \neq y^{l}$, or (iii) $y^k = y^l = +1$, and have placed hard constraints on the corresponding entries of $w$ in the optimization of the structured SVM to ensure that the pairwise potentials in the 
energy minimization problem remain submodular \cite{Zaremba2016a}.

During training of the structured SVM, we must perform \emph{loss augmented inference} in order to compute a subgradient of the loss.
In the case of margin rescaling, this consists of computing
\begin{equation}\label{eq:maxenergy}
\arg\max_{\tilde{y} \in \mathcal{Y}} \langle w, \phi(x,\tilde{y}) \rangle + \Delta(y^{*},\tilde{y}) .
\end{equation}
If $\mathcal{Y}$ is isomorphic to $\{-1,1\}^{p}$ for some $p$, $\Delta(y^*,\cdot)$ will be isomorphic to a set function $\ell : \mathcal{P}(V) \to \mathbb{R}_{+}$ where $\mathcal{P}(V)$ is the power set of a base set with $|V| = p$.  
This allows us to discuss the properties of such loss functions $\Delta$ in terms of the language of set functions as occurs in real analysis \cite{bKolmogorovFomin1975} and discrete optimization \cite{Schrijver2004}.  
In particular, we are interested in $\Delta$ corresponding to a supermodular set function $\ell$ \cite{Schrijver2004,yu:hal-01151823}:
\begin{definition}[Supermodular set function \cite{fujishige2005submodular}]\label{def:supermodular}
A supermodular set function is a set function $\ell : \mathcal{P}(V) \to \mathbb{R}$ which satisfies:
for every $A,B \subseteq V$ with $A \subseteq B$ and every $v \in V \setminus B$ we have that
\begin{equation}
\ell(A \cup \{v\}) - \ell(A) \leq \ell(B \cup \{v\}) - \ell(B).\label{eq:supermodularDefinition}
\end{equation}
\end{definition}
A function is submodular if its negative is supermodular.
Given the definition of supermodularity, we may now define when a loss function $\Delta : \mathcal{Y} \times \mathcal{Y} \rightarrow \mathbb{R}_+$ is supermodular.
\begin{definition}[Supermodular loss function \cite{yu:hal-01151823}]\label{def:SupermodularLossFunction}
A loss function $\Delta : \mathcal{Y} \times \mathcal{Y} \rightarrow \mathbb{R}_+$ is called supermodular if, for every $y \in \mathcal{Y}$, the unique set function $\ell$ such that $\Delta(y, \tilde{y}) \mapsto \ell(\{j | y^j \neq \tilde{y}^j\})$ is supermodular.
\end{definition}
We note that in Definition~\ref{def:SupermodularLossFunction} the mapping to the set function $\ell$ has an explicit dependency on the ground truth labeling $y$ and varies per training image.

Necessary to the sequel of the article, we introduce also the definition of a symmetric set function:
\begin{definition}[Symmetry]
A set function $\ell:\mathcal{P}(V)\mapsto\mathbb{R}$ is symmetric if $\ell(A) = c(|A|)$ for some function $c : \mathbb{Z}_+ \mapsto \mathbb{R} $. ($\mathbb{Z}$ is the set of integers and $\mathbb{Z}_+$ is the set of non-negative integers.)
\end{definition}
\begin{theorem}[Cardinality-based set function \cite{Bach2013submodular}]\label{th:cardinalSupermodular}
If $\ell:\mathcal{P}(V)\mapsto\mathbb{R}$ and there exist a function $c:\mathbb{Z}_+\mapsto\mathbb{R}$ such that $\ell(A) = c(|A|)$, where $|\cdot|$ is the cardinality of $A$. Then $\ell$ is supermodular if and only if $c$ is convex.
\end{theorem}

As we have guaranteed that maximization of $\langle w, \phi(x_i,\tilde{y}_i) \rangle$ with respect to $\tilde{y}$ corresponds to a submodular minimization problem, the loss augmented inference as in Equation~\eqref{eq:maxenergy} remains a submodular minimization when $\Delta$ is supermodular and can be aligned with the inference, and therefore polynomial time solvable.  By contrast, non-supermodular $\Delta$ result in NP-hard optimization problems in general.

Modular loss functions, such as Hamming loss, can be incorporated into the unary potentials in a graph cuts optimization framework for loss augmented inference.  However, the formulation of loss augmented inference with supermodular losses as a graph cuts problem is not straightforward, despite previous work (in which a custom graph cuts formulation was derived for a specific family of supermodular losses) that indicated a supermodular loss can lead to improved segmentation quality \cite{osokin2014perceptually}.
Moreover, while supermodular loss functions guarantee polynomial time solvability, they do not do so with low order polynomial guarantees in general.
We have observed that the Fujishige-Wolfe algorithm is infeasible to apply even in the case of sub-megapixel images, and scales poorly for useful supermodular loss functions.
Consequently, we develop a general framework for decomposing loss augmented inference based on ADMM. This framework solely relies on a loss function being able to be efficiently optimized in isolation using a specialized solver specific to the loss function. 

\subsection{A supermodular loss function for binary image segmentation}

\begin{figure}
\subfigure[An 8-connected neighborhood is used in the construction of the loss function.]{\label{fig:8connectedLoss}

\begin{minipage}{0.3\columnwidth}
\resizebox{.8\textwidth}{!}{
\tikzset{
  p/.style = {fill = none, draw=black, text = black}              
}{\centering
\begin{tikzpicture}[->,>=stealth',shorten >=.5pt,auto,node distance=1cm,inner sep=1pt,minimum size=5pt,semithick,every node/.style = {circle}]
  \tikzstyle{every state}=[fill=none,draw=black,text=black]

  \node[p] (X1)               { };
  \node[p] (X2) [below of=X1] { };
  \node[p] (X3) [below of=X2] { };
  
  \node[p] (Y1) [right of=X1] { };
  \node[p] (Y2) [below of=Y1] { };
  \node[p] (Y3) [below of=Y2] { };
  
  \node[p] (Z1) [right of=Y1] { };
  \node[p] (Z2) [below of=Z1] { };
  \node[p] (Z3) [below of=Z2] { };
  
  \node[p] (A1) [right of=Z1] { };
  \node[p] (A2) [below of=A1] { };
  \node[p] (A3) [below of=A2] { };
  
  \path[-] (X1) edge node {} (Y1)
  		(X2) edge node {} (Y2)
  		(X3) edge node {} (Y3)
  		
  		(Y1) edge node {} (Z1)
  		(Y2) edge node {} (Z2)
  		(Y3) edge node {} (Z3)
		
		(Z1) edge node {} (A1)
  		(Z2) edge node {} (A2)
  		(Z3) edge node {} (A3)
		
  		(X1) edge node {} (X2)
  		(X2) edge node {} (X3)
  		
  		(Y1) edge node {} (Y2)
  		(Y2) edge node {} (Y3)
  		
  		(Z1) edge node {} (Z2)
  		(Z2) edge node {} (Z3)
  		
  		(A1) edge node {} (A2)
  		(A2) edge node {} (A3) 
  		
  		(X1) edge node {} (Y2)
  		(X2) edge node {} (Y3)
  		(X3) edge node {} (Y2)
  		(X2) edge node {} (Y1)
  		
  		(Y1) edge node {} (Z2)
  		(Y2) edge node {} (Z3)
		(Y3) edge node {} (Z2)
		(Y2) edge node {} (Z1)
		
  		(Z1) edge node {} (A2)
  		(Z2) edge node {} (A3)
		
  		(Z3) edge node {} (A2)
  		(Z2) edge node {} (A1)
  		;
\end{tikzpicture}
}}
\end{minipage}
}\hfill
\subfigure[Pairwise potential construction for an edge with $y^{*k}=+1$ and $y^{*l}=-1$ following the loss function in Equation~\eqref{eq:supermodularLoss}.]{
\begin{minipage}{0.6\columnwidth}
\resizebox{0.98\columnwidth}{!}{\centering
$E = - \!\!\!\! \overbrace{\begin{pmatrix} w_{00} & w_{01}\\w_{10} & w_{11} \end{pmatrix}}^{\text{inference pairwise potential}} \!\!\!\! - \!\!\!\! \underbrace{\begin{pmatrix} 0 & \gamma \\ 0 & 0 \end{pmatrix}}_{\text{loss pairwise potential}}$
}
\end{minipage}
}
\caption{Non-submodularity of the joint loss augmented inference procedure using the same mapping to a set function for inference and loss functions.The inference procedure can be solved by graph cuts when the sum of the diagonal elements of $E$ is less than the sum of the off diagonal elements.  While it is enforced during optimization that $w_{00}+w_{11}-w_{01}-w_{10} \geq 0$, the presence of $\gamma$ in the off diagonal, the exact position depending on the value of $y^*$,   removes the guarantee of a resulting submodular minimization problem.}\label{fig:notSubmodularJointGraphcuts}
\end{figure}
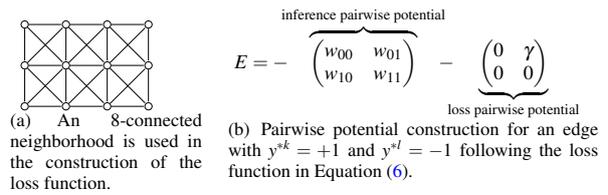

As a first example, we propose a loss function that is itself optimizable with graph cuts.  The loss simply counts the number of incorrect pixels plus the number of pairs of neighboring pixels that both have incorrect labels
\begin{equation}\label{eq:supermodularLoss}
\Delta_8(y^*,\tilde{y}) = \sum_{j=1}^{p} [y^{*j} \neq \tilde{y}^j] + \sum_{(k,l) \in \mathcal{E}_\ell} \gamma [y^{*k} \neq \tilde{y}^k \wedge y^{*l} \neq \tilde{y}^l]
\end{equation}
where $[\cdot]$ is Iverson bracket notation, $\mathcal{E}_\ell$ is a loss specific edge set and   $\gamma$ is a positive weight.  We have used 8-connectivity for the loss function in the experiments (Figure~\ref{fig:8connectedLoss}), referred to as ``8-connected loss'' in the sequel. We may identify this function with a set function to which the argument is the set of mispredicted pixels.

\begin{proposition}
Maximization of the loss function in Equation~\eqref{eq:supermodularLoss} is equivalent
to a supermodular function maximization problem.
\end{proposition}
\begin{proof}[Proof sketch]
  Equation~\eqref{eq:supermodularLoss} is isomorphic to a binary random field model for which label is $1$ iff a pixel has a different label from the ground truth.  
  Neighboring pixels that both have label $1$ contribute a positive amount to the energy, while all other configurations contribute zero.  This corresponds to a supermodular function following Definition~\ref{def:supermodular}.
  \qed
\end{proof}

This loss function emphasizes the importance of correctly predicting adjacent groups of pixels, e.g.\ those present in thin structures more than one pixel wide.  While the pairwise potential in $\langle w , \phi(x,y) \rangle$ has a tendency to reduce the perimeter of the segment, the loss strongly encourages the correct identification of adjacent pixels. 
We will observe in the experimental results that the use of this loss function during training improves the test time prediction accuracy, even when measuring in terms of Hamming loss.

It may appear at first glance that the structure of this loss function is aligned with that of the inference, and that we can therefore jointly optimize the loss augmented inference with a single graph cuts procedure.  Indeed, the loss function is isomorphic to a supermodular set function, and the inference is isomorphic to a supermodular set function, both of which can be solved by graph cuts.
However, the isomorphisms are not the same.  
The loss function maps to a set function by considering the set of pixels that are incorrectly labeled, while the inference maps to a set function by considering the set of pixels that are labeled as foreground.
Shown in Figure~\ref{fig:notSubmodularJointGraphcuts} is the pairwise potential for an edge with $y^{*k}=+1$ and $y^{*l}=-1$.

If we apply a single mapping, the inference procedure can be solved by graph cuts when the sum of the diagonal elements of $E$ is less than the sum of the off diagonal elements.  While it is enforced during optimization that $w_{00}+w_{11}-w_{01}-w_{10} \geq 0$, the presence of $\gamma$ in the off diagonal, for which the exact position depends on the value of $y^*$, removes the guarantee of a resulting submodular minimization problem.
We therefore consider a Lagrangian based splitting method to solve the loss augmented inference problem in Section~\ref{sec:ADMMlossaugmented}

\subsection{Symmetric supermodular loss function: square loss}\label{sec:square}

As a second example, we consider the following loss which simply takes the square of the number of mis-predictions.  This function is not readily incorporated in graph-cuts, as the square induces a pairwise dependency between all pixels.
\begin{equation}
\Delta_{\text{S}}(y^*,\tilde{y}) = \left(  \frac{\sum_{j=1}^{p} [y^{*j} \neq \tilde{y}^j]}{\alpha}   \right)^2 \label{eq:squareloss}
\end{equation}
where $\alpha >0$ is a scale factor to prevent the value to be too large in an image scale problem. We used $\alpha=\sqrt{|y^*|}$ in our setting, where $|\cdot|$ is the number of positive labels of $y^*$. This is a function on the misprediction set which only depends on the size of the input set i.e.\ $\ell(A) = \left(\frac{|A|}{\alpha}\right)^2$. As the square function is a convex function, $\Delta_{\text{S}}$ is a supermodular loss w.r.t.\ the misprediction set. Then maximizing the loss itself is a supermodular maximization i.e.\ a submodular minimization problem, following Theorem~\ref{th:cardinalSupermodular}. 

\subsection{Supermodular Loss Functions Through Biconvexity}\label{sec:Biconvex}

In this section, we develop a family of supermodular loss functions based on biconvex functions of the number of false positives and false negatives.  
\begin{definition}[Biconvexity \cite{Gorski2007}]\label{def:Biconvex}
  A function $f : \mathcal{A}\times \mathcal{B} \rightarrow \mathbb{R}$ is called a \emph{biconvex function} if
  \begin{equation}
    f_a(\cdot) := f(a,\cdot) : \mathcal{B} \rightarrow \mathbb{R}
  \end{equation}
  is a convex function on $\mathcal{B}$ for every fixed $a\in \mathcal{A}$ and
  \begin{equation}
    f_b(\cdot) := f(\cdot, b) : \mathcal{A} \rightarrow \mathbb{R}
  \end{equation}
  is a convex function on $\mathcal{A}$ for every fixed $b\in \mathcal{B}$.
\end{definition}
\begin{remark}
  The usual definition of biconvexity specifies that the function be defined over a biconvex set \cite[Definition~1.1]{Gorski2007}, but for the purpose of this section we will restrict ourselves to $\mathcal{A}$ and $\mathcal{B}$ being convex sets so that biconvexity of the domain of $f$ follows trivially.
\end{remark}

Denote by $m$ the number of positive labels in the ground truth labeling $y^*$,
\begin{equation}
m:= |y^*| = \sum_{i=1}^{p} \left[ y^*_i = +1\right].\label{eq:groundtruthPos}
\end{equation}

\begin{proposition}[Biconvexity characterizes supermodularity]\label{biconvexitySupermodular}
For a given ground truth labeling $y^*$ and a given prediction $\tilde{y}$,
 let $e_{-}$ denote the number of false negatives, and $e_{+}$ denote the number of false positives:
\begin{align}
  e_{-}:=& \sum_{i=1}^{p} \left[ y^*_i = +1 \wedge \tilde{y}_i = -1  \right] \label{eq:falsenegative} \\
  e_{+}:=& \sum_{i=1}^{p} \left[ y^*_i = -1 \wedge \tilde{y}_i = +1  \right] \label{eq:falsepositive}
\end{align}
where $[\cdot]$ is Iverson bracket notation.
The following holds
  \begin{equation}
    \Delta(y^*,\tilde{y}) := \ell(e_{-},e_{+})
  \end{equation}
  is a supermodular loss function iff $\exists \hat{\ell} : [0,m] \times [0,p-m] \mapsto \mathbb{R}_+$ that is a biconvex function and
  \begin{equation}
    \hat{\ell}(e_{-},e_{+}) = \ell(e_{-},e_{+}) \quad \forall e_{-},e_{+}  .\label{eq:hatElltoEll}
  \end{equation}
  In particular, we may select $\hat{\ell}$ to be the convex closure of $\ell$ \cite[Section~5.1]{Bach2013submodular}.
\end{proposition}
\begin{proof}
%
We first show that supermodularity implies the existence of a corresponding biconvex function.
  \cite[Proposition~B.2]{Bach2013submodular} indicates that for a function to be supermodular, all contractions of that function must be supermodular, in particular the contractions achieved by fixing a set of false positives and fixing a set of false negatives.

  For the contraction obtained by fixing false positives to be supermodular, we have from Theorem~\ref{th:cardinalSupermodular} that there exists a convex function that specifies the contraction.  Similarly for the contraction obtained by fixing false negatives to be supermodular, there exists a (different) convex function that specifies the contraction.  Combining all such contractions and convex functions yields the conditions in Definition~\ref{def:Biconvex} for integral points.  The existence of a function satisfying these conditions for non-integral points is obtained by noting that the convex closure satisfies the required properties.

  It now remains to show that a biconvex function yields a supermodular set function. 
  Given a biconvex function of the number of false negatives and the number of false positives, by Definition~\ref{def:Biconvex}, the function obtained by fixing the number of false positives is convex in the number of false negatives, then we have from Theorem~\ref{th:cardinalSupermodular} that the set function restricted to the set of foreground pixels is supermodular. A symmetric argument gives that the restriction to the set of background pixels is also supermodular.  
  \qed
\end{proof}

Several popular loss functions such as Intersection over Union loss \cite[Equation~(43)]{Yu2015b} or S{\o}rensen-Dice loss \cite[Definition~11]{Yu2016a} can be specified as functions of the number of false positives and false negatives, but both have been shown to be non-supermodular.  In the next section we will develop a novel supermodular loss function by specifying an increasing biconvex function of $e_{-}$ and $e_{+}$.

\subsection{A novel loss function from biconvexity}\label{sec:biconvexloss}

The Intersection over Union loss \cite[Equation~(7)]{blaschko2008learning} has been shown to be non-supermodular \cite[Proposition~10]{Yu2015b}:
\begin{equation}\label{eq:IoUloss}
  \Delta_{IoU}(y^*,\tilde{y}) = 1 - \frac{|y^* \cap \tilde{y}|}{|y^* \cup \tilde{y}|}.
\end{equation}
We develop here a novel loss function that is similar in flavor to Equation~\eqref{eq:IoUloss} but we will see that it is supermodular: 
\begin{equation}
\Delta_C(y^*,\tilde{y}) = \frac{|y^*|+|\tilde{y}|-2|y^*\cap \tilde{y}|}{|y^*\cap \tilde{y}|+1} .\label{eq:biconvexloss}
\end{equation}
We can verify that $0\leq \Delta(y^*,\tilde{y})$, $\forall \tilde{y}\in\mathcal{Y}$, and $\Delta(y^*,y^*) = 0$.

Given a ground truth labeling $y^*$, we can consider $m$ to be a constant. With the notation in Equation~\eqref{eq:falsenegative} and Equation~\eqref{eq:falsepositive}, we can write $\Delta_C$ as a function of $e_{-}$ and $e_{+}$, denoted $\ell_C$:
\begin{equation}
\ell_C(e_{-},e_{+}) = \frac{e_{-}+e_{+}}{m-e_{-}+1}. \label{eq:l_c_biconvex}
\end{equation}

\begin{proposition}
There exits a function $\hat{\ell}_C : \mathbb{R}_+ \times \mathbb{R}_+ \mapsto \mathbb{R}_+$ that is biconvex and
  \begin{equation}
    \hat{\ell}_C(e_{-},e_{+}) = \ell_C(e_{-},e_{+}) \quad \forall e_{-},e_{+} \in \mathbb{Z}_+ .
  \end{equation}
\end{proposition}
\begin{proof}
We set
\begin{equation}
\hat{\ell}_C( e_1, e_2) = \frac{e_1+e_2}{m-e_1+1}, \forall e_1, e_2 \in \mathbb{R}_+ .
\end{equation}
It is straightforward that it satisfies Equation~\eqref{eq:hatElltoEll} in Proposition~\ref{biconvexitySupermodular}.
We now prove that this is a biconvex function. We note that with $e_1$ fixed, $\hat{\ell}_C$ is linear in $e_2$ and therefore $\hat{\ell}_C(e_1,\cdot)$ is convex.
  
Now to show that $\hat{\ell}_{C}(\cdot, e_2)$ is convex,  we calculate its first and second derivatives with respect to $e_1$:
\begin{align}
& \frac{\partial \hat{\ell}}{\partial e_1} 
= \frac{m+e_2+1}{(m-e_1+1)^2}, \\
& \frac{\partial^2 \hat{\ell}}{\partial e_1^2} 
= \frac{2(m+e_2+1)(m-e_1+1)}{(m-e_1+1)^4} \geq 0 \label{eq:seconddiff}
\end{align}
Given the fact that $e_1$ is the number of false negatives and $m$ is the number of ground truth positive labels, we have $e_1\leq m$.  All parenthesized terms of Equation~\eqref{eq:seconddiff} must therefore be strictly positive. 
As $\hat{\ell}_{C}(\cdot, e_2)$ is twice differentiable everywhere and its second derivative is non-negative, $\hat{\ell}_c$ is convex wrt $e_1$.
\qed
\end{proof}
Following Proposition~\ref{biconvexitySupermodular}, we then have the following corollary:
\begin{corollary}
$\Delta_C$ in Equation~\eqref{eq:biconvexloss} is supermodular.
\end{corollary}

\subsection{ADMM algorithm for loss augmented inference}\label{sec:ADMMlossaugmented}
Several Lagrangian based decomposition frameworks have been proposed, such as dual decomposition 
and ADMM \cite{boyd2011distributed}, with the latter having improved convergence guarantees. We have also observed a substantial improvement in performance using ADMM over dual decomposition in our own experiments. Here we consider a splitting method to optimize the minimization of the negative of Equation~\eqref{eq:maxenergy}, which is equivalent to finding the most violated constraint in cutting plane optimization:
\begin{equation}\label{eq:argmin}
\arg\min_{y_a,y_b} -\langle w,\phi(x,y_a) \rangle - \Delta(y^*,y_b)\quad
\textrm{s.t. }y_a=y_b .
\end{equation} 
and we form the augmented Lagrangian as
\begin{align}
\mathcal{L}(y_a,y_b,\lambda)
=& -\langle w,\phi(x,y_a) \rangle - \Delta(y^*,y_b) \nonumber \\
&+\lambda^T(y_a-y_b)+\frac{\rho}{2}\|y_a-y_b\|_2^2 \label{eq:lagADMM}
\end{align}
where $\rho>0$. \eqref{eq:lagADMM} can be optimized in an iterative fashion by Algorithm~\ref{alg:admm} \cite{boyd2011distributed}. 

\begin{algorithm}[!htb]
\caption{ADMM in scaled form for finding a saddle point of the Lagrangian in Equation~\eqref{eq:lagADMM} }\label{alg:admm}
\begin{algorithmic}[1]
\STATE {Initialization $u^0=0$}
\REPEAT  
\STATE \label{alg:ADMM:subproblemInference}  
$y_a^{t+1} = \arg\min_{y_a} -\langle w,\phi(x,y_a) \rangle 
+ \frac{\rho}{2}(\|y_a-y_b^t+u^t\|_2^2)$
\STATE \label{alg:ADMM:subproblemLoss}
$y_b^{t+1} = \arg\min_{y_b}  - \Delta(y^*,y_b) + \frac{\rho}{2}(\|y_a^{t+1}-y_b+u^t\|_2^2)$
\STATE $u^{t+1}\ =\ u^t+(y_a^{t+1}-y_b^{t+1})$
\STATE $t=t+1$ 
\UNTIL{stopping criterion satisfied}  
\end{algorithmic}
\end{algorithm}

The saddle point of the Lagrangian will correspond to an optimal solution over a convex domain, while we are optimizing w.r.t.\ binary variables.  Strictly speaking, we may therefore consider the linear programming (LP) relaxation of our loss augmented inference problem, followed by a rounding post-processing step. We use a standard stopping criterion as in \cite{boyd2011distributed}: the primal and dual residuals must be small with an absolute criterion $\epsilon^{\text{abs}}=10^{-4}$ and a relative criterion $\epsilon^{\text{rel}}=10^{-2}$. In practice, we have found that discretizing the quadratic terms and incorporating them into the unary potentials of the respective graph cuts problems is more computationally efficient, while yielding results that are nearly identical with exact optimization with a primal-dual gap of 0.01\%.
We show in the experimental results that this strategy yields results almost identical to those of an LP relaxation. 


In general, we simply need task-specific solvers for Line~\ref{alg:ADMM:subproblemInference} and Line~\ref{alg:ADMM:subproblemLoss} 
of Algorithm~\ref{alg:admm}.  These solvers need not use a single graph cut algorithm, and can therefore exploit any available structure even though it may not be present, or aligned, between the two subproblems.  Although we have used this framework for the specific supermodular loss functions described in the previous subsection, we note that this provides an API for the structured output SVM framework alternate to that provided by SVMstruct \cite{tsochantaridis2005large}.  We have released our structured prediction toolbox as an open source project, enabling the application of this strategy to diverse structured prediction problems with non-modular loss functions.

\subsection{ADMM Convergence}\label{sec:admmconverge}

Consider the standard form of the problem solved by ADMM:
\begin{align}
&\min \qquad f(x)+g(x) \\
&\text{s.t.\ } Ax+Bz = c
\end{align}
with variables $x \in \mathbb{R}^n$ and $z\in\mathbb{R}^m$, where $A\in\mathbb{R}^{p\times n}$, $B\in\mathbb{R}^{p\times m}$, and 
$c\in \mathbb{R}^p$. Following \cite{boyd2011distributed}, some general convergence results for ADMM are considered in this section.

\begin{assumption}\label{th:admmconverge1}
The (extended-real-valued) functions $f:\mathbb{R}^n \mapsto \mathbb{R}\cup\{+\infty\}$ and $g:\mathbb{R}^m \mapsto \mathbb{R}\cup\{+\infty\}$ are closed, proper, and convex.
\end{assumption}

\begin{assumption}\label{th:admmconverge2}
The unaugmented Lagrangian of the problem has a saddle point.
\end{assumption}

If Assumption~\ref{th:admmconverge1} and Assumption~\ref{th:admmconverge2} hold, 
the ADMM algorithm guarantees:
(1) the residual convergence: $r^k\to 0$ as $k \to \infty$, i.e., the iterates approach feasibility;
(2) the objective convergence: $f(x^k) + g(z^k) \to p$ as $k\to\infty$, i.e., the objective function of the iterates approaches the optimal value;
(3) and the dual variable convergence: $y^k\to y$ as $k\to\infty$, where $y$ is a dual optimal point.
Proofs of the residual and objective convergence results are given in \cite{boyd2011distributed}.

\subsection{Optimization}\label{sec:optimization}

As shown in Line~\ref{alg:ADMM:subproblemLoss} in Algorithm~\ref{alg:admm}, we need to solve the subproblem that minimizes the negative of the loss function augmented by a term from the ADMM iteration. It is equivalent to maximizing the sum of the loss function and the negative of the ADMM term.
Among the three examples of supermodular loss functions we proposed, maximizing the  8-connected loss in Equation~\eqref{eq:supermodularLoss}, augmented by a modular term from the ADMM iteration, can be solved by a modified graph-cut. Maximizing the square loss (Equation~\eqref{eq:squareloss}) and the biconvex loss (Equation~\eqref{eq:biconvexloss}) can also be solved efficiently, as we will show in this section. 

Explicitly, we maximize over the sum of a supermodular loss function and a modular function:
\begin{equation}
y_b = \arg\max_{y}  \Delta(y^*,y) + r(y) \label{eq:lossAugmented}
\end{equation}
where $r(y) = - \frac{\rho}{2}(\|y_a-y+u\|_2^2)$ is an asymmetric modular function wrt the misprediction set $\{j | y^{*j} \neq y^j\}$ for a given $y_a$ and $u$ at the current iteration (we discard the supercript $t$ for simplicity). We know that any modular function can be written as 
\begin{equation}\label{eq:modularIsLinear}
r(A) = \sum_{j \in A} w^j
\end{equation}
for some coefficient vector $w \in \mathbb{R}^{|V|}$. 
In our case, we have
\begin{equation}
w^j = - \frac{\rho}{2}(y_a^j+y^{*j}+u^j)^2 + \frac{\rho}{2}(y_a^j-y^{*j}+u^j)^2, \forall j\in V
\end{equation}
Under the assumption that $\Delta(y^*,y_b)$ is a symmetric loss function, such as the square loss in Equation~\eqref{eq:squareloss}, Algorithm~\ref{alg:symmetric} solves the required optimization efficiently.
\begin{algorithm}[htb]
\caption{\label{alg:symmetric} Maximization of Equation~\eqref{eq:lossAugmented} with a symmetric loss function.} 
\begin{algorithmic}[1]
\STATE Sort the vector $w = \langle w^{1},\dots,w^{j},\dots w^{|V|}\rangle$ in decreasing order, denoted $w^{\pi} = \langle w^{\pi^1},\dots,w^{\pi^j},\dots w^{\pi^{|V|}}\rangle$ with $\pi$ the permutation that achieves this sorting;
\FOR{$j = 1$ \TO $|V|$} \label{line:forAllj}
\STATE \label{line:marginalvalues} Calculate the marginal values 
\begin{displaymath}
\ell_{\text{marginal}}(j) = \ell(j)-\ell(j-1);
\end{displaymath}
\STATE Calculate the augmented marginal values by adding the sorted modular vector $w^{\pi}$
\begin{displaymath}
\ell_{\text{augmented}}(j)= \ell_{\text{marginal}}(j) + w^{\pi^j};
\end{displaymath}
\STATE Calculate the loss augmented values, which are the cumulative sum of the marginal values
\begin{displaymath}
\ell_{\text{all}}(j) = \sum_{k=1}^{j} \ell_{\text{augmented}}(k);
\end{displaymath}
\ENDFOR
\STATE Find the maximum of $\ell_{\text{all}}(j)$ wrt $j \leq |V|$, denote $A_{\text{opt}}$
\RETURN $y$ such that $\{j|y^{*j}\neq y^j\} = A_{\text{opt}}$.
\end{algorithmic}
\end{algorithm}
By exploiting the symmetry properties of the loss function, all operations in Algorithm~\ref{alg:symmetric} are linear except for the first sorting operation.

For $\Delta(y^*,y_b)$ biconvex, as in Equation~\eqref{eq:biconvexloss}, we analyze the problem wrt to false positives and false negatives separately. For a given ground truth labeling $y^*$, we note that the subset  $M:=\{ j|  y^{*j} = +1\}$ is the set of all possible false negatives. $|M| = m$ following Equation~\eqref{eq:groundtruthPos}, $\{ j|  y^{*j} = -1\}$ is the set of all possible false positives, i.e.\ $V\setminus M$.
We first rewrite the modular function as a coefficient vector of ground truth positive entries and ground truth negative entries separately,
\begin{align}
& r(A) = \sum_{j \in A} w_{\text{neg}}^j, \forall A\subseteq M,\label{eq:modularFalseNeg} \\
& r(A) = \sum_{j \in A} w_{\text{pos}}^j, \forall A\subseteq V\setminus M, \label{eq:modularFalsePos}
\end{align}
for two coefficient vectors $w_{\text{neg}}, w_{\text{pos}} \in \mathbb{R}^{|V|}$.

Under the assumption that $\Delta(y^*,y_b)$ is a biconvex function, as in Equation~\eqref{eq:biconvexloss}, Algorithm~\ref{alg:biconvex} is an efficient solver for the resulting optimization.
\begin{algorithm}[]
\caption{\label{alg:biconvex} Maximization of Equation~\eqref{eq:lossAugmented} with biconvex-supermodular loss function} 
\begin{algorithmic}[1]
\STATE Sort $w_{\text{neg}}$ and $w_{\text{pos}}$ in decreasing order, denoted $w_{\text{neg}}^{\pi_{-}}$ and $w_{\text{pos}}^{\pi_{+}}$;
\FOR{$j=0$ \TO $|M|$}\label{line:forAlln}
\FOR{$k = 1$ \TO $|V\setminus M|$} \label{line:forAllp}
\STATE Calculate the marginal values wrt one false positive
\begin{displaymath}
\ell_{\text{marginal}}(j,k) = \ell(j,k)-\ell(j,k-1);
\end{displaymath}
\STATE Calculate the augmented marginal values by adding the sorted modular vector $w_{\text{pos}}^{\pi_{+}}$
\begin{displaymath}
\ell_{\text{augmented}}(j,k)= \ell_{\text{marginal}}(j,k) + w_{\text{pos}}^{\pi_{+}^k};
\end{displaymath}
\STATE Calculate the loss augmented values, which are the cumulative sums of the marginal values
\begin{displaymath}
\ell_{\text{pos}}(j,k) = \sum_{l=1}^{k} \ell_{\text{augmented}}(j,l);
\end{displaymath}
\ENDFOR
\STATE Find the maximum of $\ell_{\text{pos}}(j,k)$ wrt $k\subseteq (V\setminus M)$, denote $k_{\text{opt}}$ for the current $j$.
\ENDFOR
\FOR{each pair $(k_{\text{opt}},j\neq 0)$}
\STATE Calculate the marginal values wrt one false negative 
\begin{displaymath}
\ell_{\text{marginal}}(j,k_{\text{opt}}) = \ell(j,k_{\text{opt}})-\ell(j-1,k_{\text{opt}});
\end{displaymath}
\STATE Calculate the augmented marginal values by adding the sorted modular vector $w_{\text{neg}}^{\pi_{-}}$
\begin{displaymath}
\ell_{\text{augmented}}(j,k_{\text{opt}})= \ell_{\text{marginal}}(j,k_{\text{opt}}) + w_{\text{neg}}^{\pi_{-}^j};
\end{displaymath}
\STATE Calculate the loss augmented values, which are the cumulative sums of the marginal values
\begin{displaymath}
\ell_{\text{neg}}(j,k_{\text{opt}}) = \sum_{l=1}^{j} \ell_{\text{augmented}}(l,k_{\text{opt}});
\end{displaymath}
\ENDFOR
\STATE Find the maximum of $\ell_{\text{neg}}(j,k_{\text{opt}})$ wrt $j$, denote $j_{\text{opt}}$, along with the $k_{\text{opt}}$ for this $j_{\text{opt}}$.
\RETURN $y$ such that $\{i|y^{*i}\neq y^i\} = j_{\text{opt}}\cup k_{\text{opt}}$ .
\end{algorithmic}
\end{algorithm}

\begin{figure}[t]  \centering
\subfigure[]{\label{fig:rgb}\includegraphics[width=0.2\linewidth]{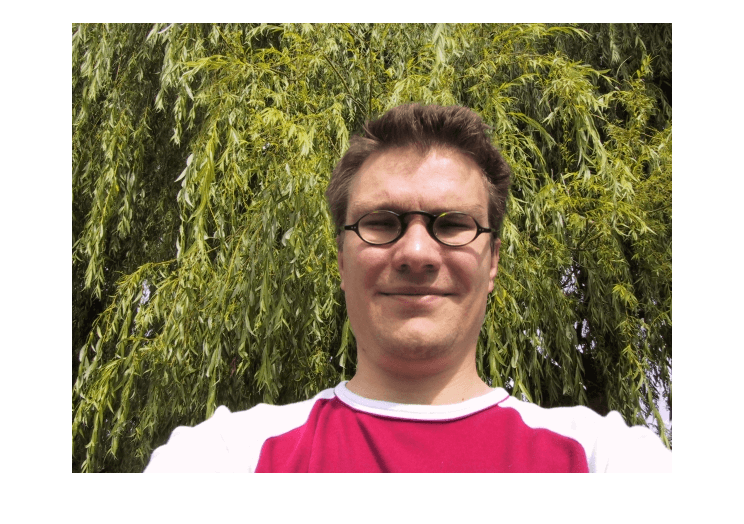}}
\subfigure[]{\label{fig:gt}\includegraphics[width=0.2\linewidth]{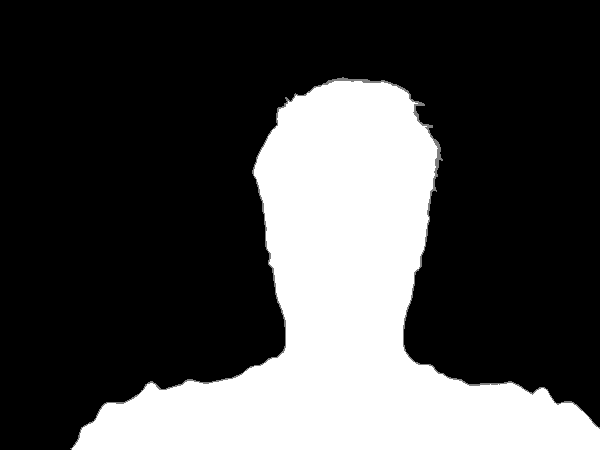}}
\subfigure[]{\label{fig:label}\includegraphics[width=0.2\linewidth]{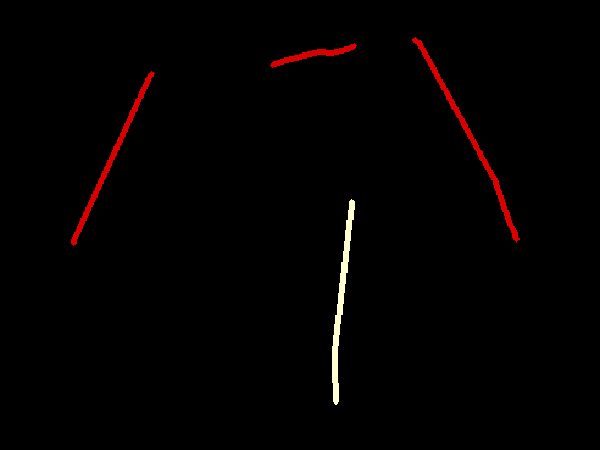}}
\subfigure[]{\label{fig:ext}\includegraphics[width=0.2\linewidth]{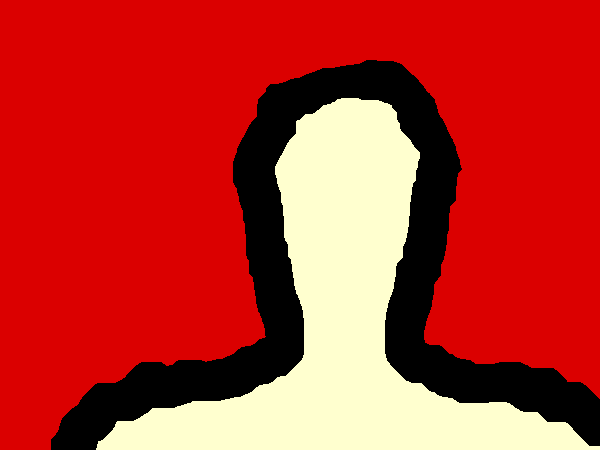}}\\
\subfigure[]{\label{fig:rgbFG}\includegraphics[width=0.2\linewidth]{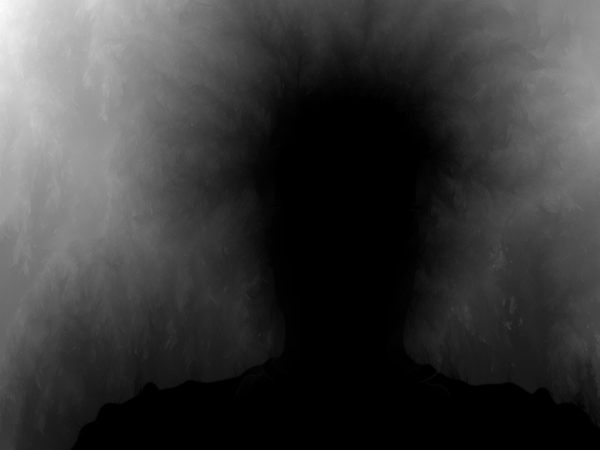}}
\subfigure[]{\label{fig:rgbBG}\includegraphics[width=0.2\linewidth]{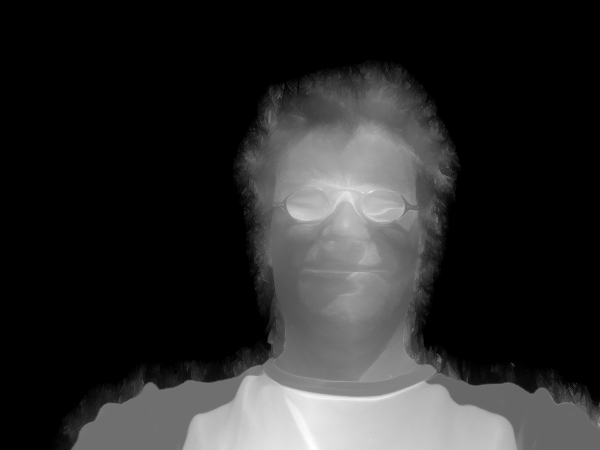}}
\subfigure[]{\label{fig:rgbGMM}\includegraphics[width=0.2\linewidth]{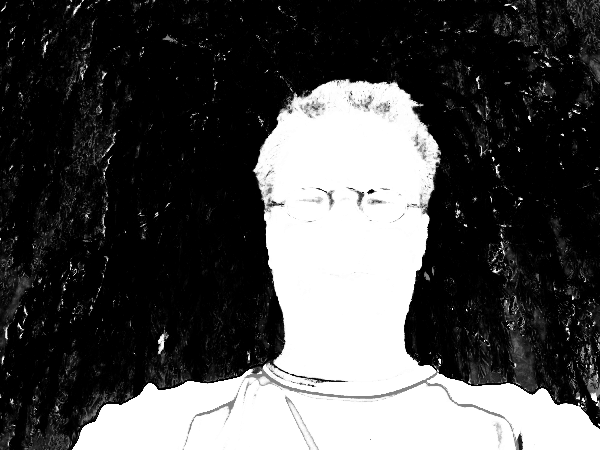}}
\subfigure[]{\label{fig:rgbGMMFG}\includegraphics[width=0.2\linewidth]{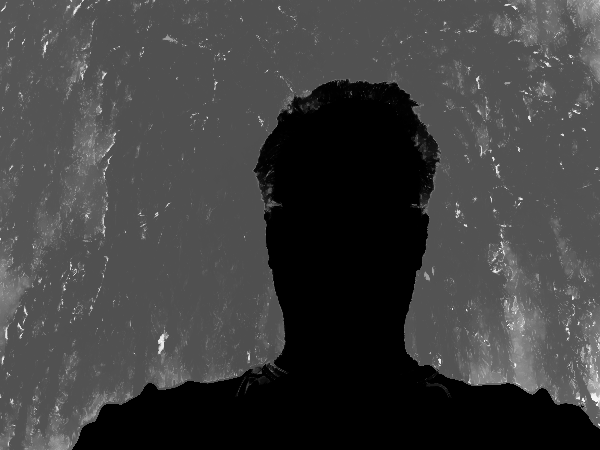}}
\caption{\label{fig:examples} Example training images and the extracted features.  \ref{fig:rgb} original RGB image; \ref{fig:gt} groundtruth; \ref{fig:label} the user-labelled seeds; \ref{fig:ext} the extended seeds; \ref{fig:rgbFG} the distance features to foreground seed based on RGB space; \ref{fig:rgbBG}  the distance features to background seed based on RGB space; \ref{fig:rgbGMM}  the GMM appearance model based on RGB space; \ref{fig:rgbGMMFG}  the distance features to foreground seed based on the RGB-space GMM appearance.}
\end{figure}

\section{Experimental Results}
In this section, we consider a foreground/background segmentation task. We compare the prediction using our proposed supermodular loss functions with the prediction using Hamming loss.  We show that: (i) our proposed splitting strategy is orders of magnitude faster than the 
minimum norm point algorithm; (ii) our strategy yields results nearly identical to a LP-relaxation while being much faster in practice; and (iii) training with the same supermodular loss as during test time yields better performance.  

\paragraph{Datasets} The dataset provided  by \cite{gulshan2010geodesic,blake2004interactive} contains 
color images in RGB space, ground truth foreground/background segmentations, and user-labelled seeds (see Figure~\ref{fig:rgb}, Figure~\ref{fig:gt}, and Figure~\ref{fig:label}, respectively). As we are discriminatively training a class specific segmentation system in our experiments, we focus on the images in which the foreground objects are \textit{people}. 
We compute in total 18 unary features following \cite{osokin2014perceptually}.  
Figure~\ref{fig:rgbFG} to Figure~\ref{fig:rgbGMMFG} show examples of the extracted features. 

\paragraph{IBSR Dataset}
We additionally utilise the Internet Brain Segmentation Repository (IBSR) dataset \cite{braindata}, which consists of T1-weighted MR images. Images and masks have been linearly registered and cropped to $145\times158\times123$. We choose one horizontal slice within each volume and we follow the feature extraction procedure as in \cite{stavros2014discrete}.

\subsection{Training with the 8-connected loss function}
We use the ADMM splitting strategy to solve the minimization problem in Equation~\eqref{eq:argmin}.
We use the GCMex - MATLAB wrapper for the Boykov-Kolmogorov graph cuts algorithm \cite{fulkerson2009,boykov2004,boykov2001,kolmogorov2004energy} to solve the optimization problems on Line~\ref{alg:ADMM:subproblemInference} for the inference. 
Results computed with different values of $\gamma>0$ are shown in Table~\ref{tab:loss} and Table~\ref{tab:brain}. 
During the training stage, we use $\rho = 0.1$ for the ADMM step-size parameter. The regularization parameter $C$ in Equation~\eqref{eq:svmmargin} is chosen by cross-validation in the range $\{ 10^i | -2\leq i \leq 2 \}$. We additionally train and test with Hamming loss as a comparison. 

At test time, we have computed the unnormalized Hamming loss, the intersection over union loss (IoU), and our 8-connected loss for each training scenario.  We have performed several random train-test splits in order to compute error bars on the loss estimates.
During testing stage, we evaluate one prediction as the average loss value for all images in the testing set. We compare different loss functions during training and during testing and measure the empirical loss values.
We  randomly split the data into training and testing sets five times to obtain an estimate of the average performance.

\begin{figure*}[tbp] \centering
{\includegraphics[width=0.15\textwidth]{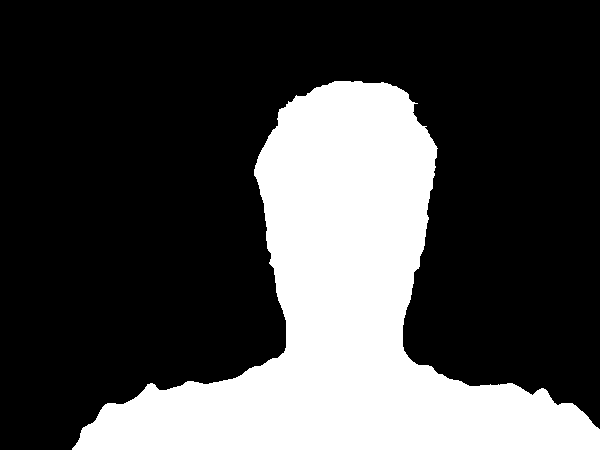}}
{\includegraphics[width=0.15\textwidth]{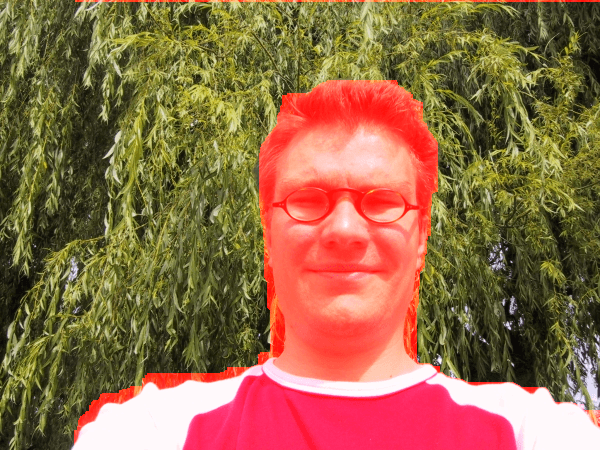}}
{\label{fig:1admm}\includegraphics[width=0.15\textwidth]{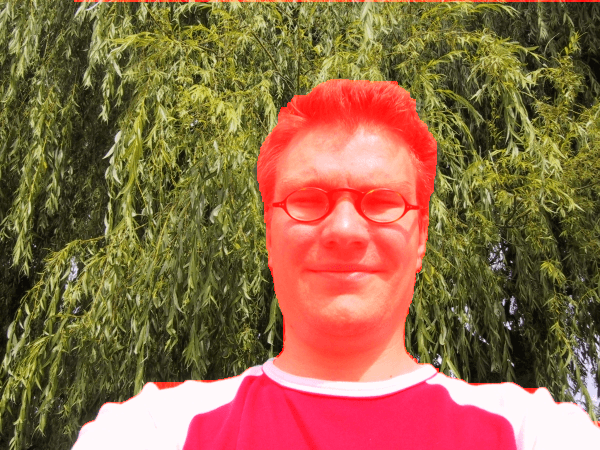}}
{\includegraphics[width=0.15\textwidth]{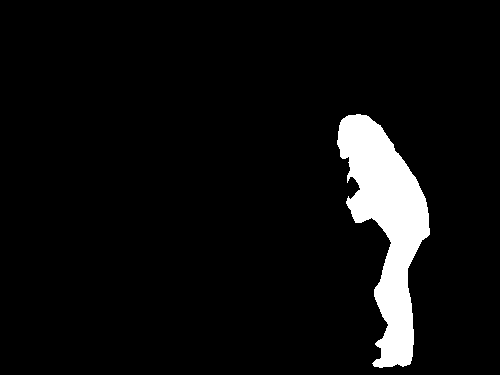}}
{\includegraphics[width=0.15\textwidth]{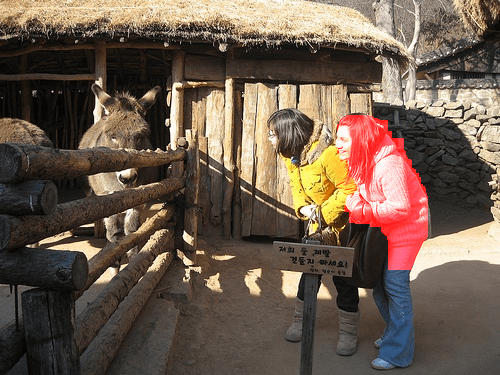}}
{\includegraphics[width=0.15\textwidth]{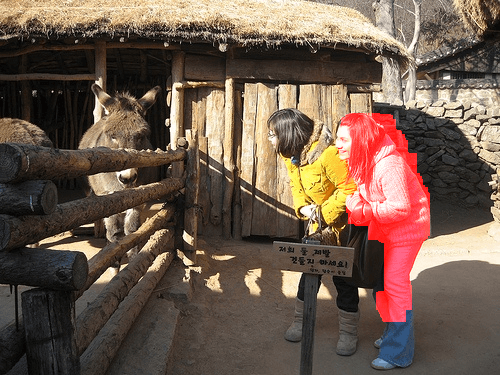}}
{\includegraphics[width=0.15\textwidth]{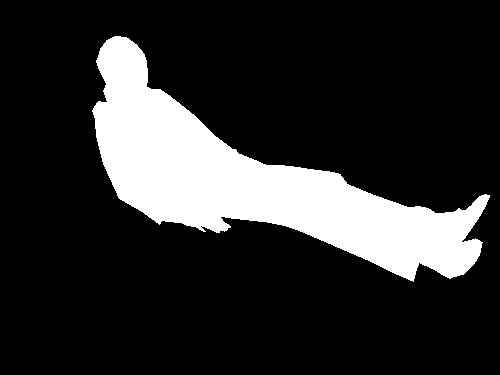}}
{\includegraphics[width=0.15\textwidth]{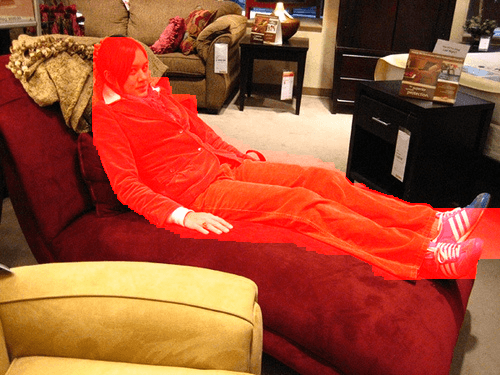}}
{\includegraphics[width=0.15\textwidth]{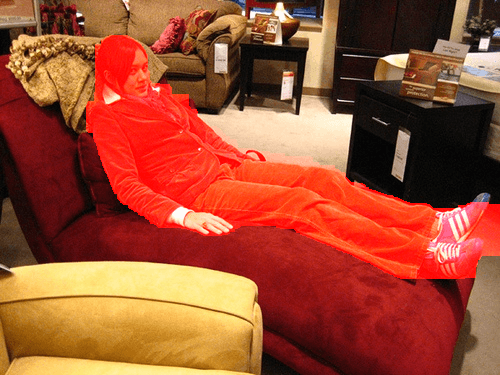}}
{\includegraphics[trim={0 2.5cm 0 0},clip,width=0.15\textwidth]{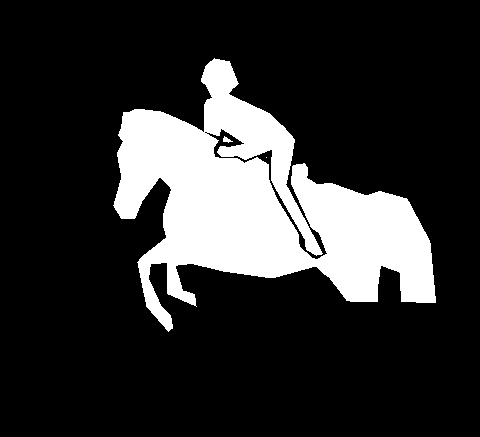}}
{\includegraphics[trim={0 2.5cm 0 0},clip,width=0.15\textwidth]{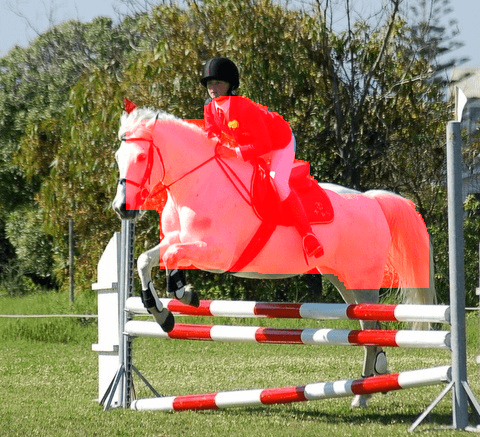}}
{\includegraphics[trim={0 2.5cm 0 0},clip,width=0.15\textwidth]{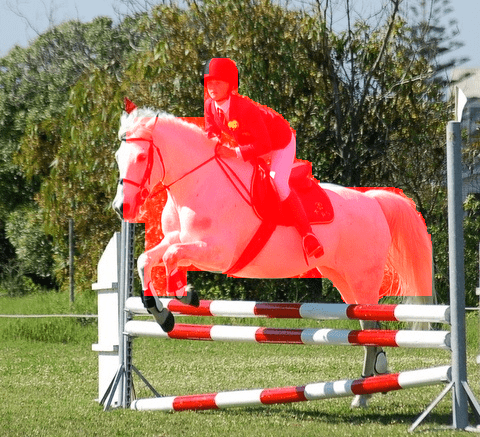}}
{\includegraphics[trim={0 2.5cm 0 2.5cm},clip, width=0.15\textwidth]{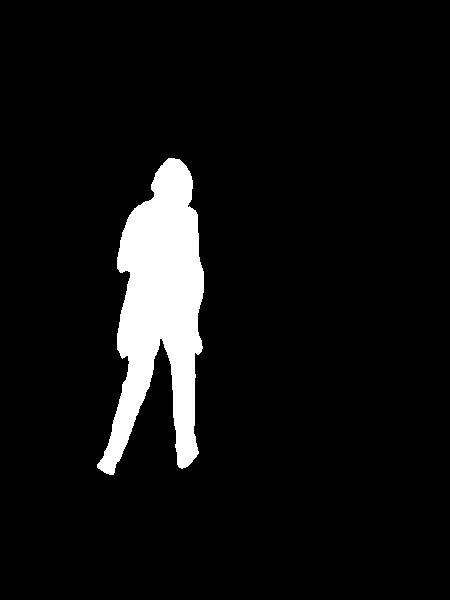}}
{\includegraphics[trim={0 2.5cm 0 2.5cm},clip,width=0.15\textwidth]{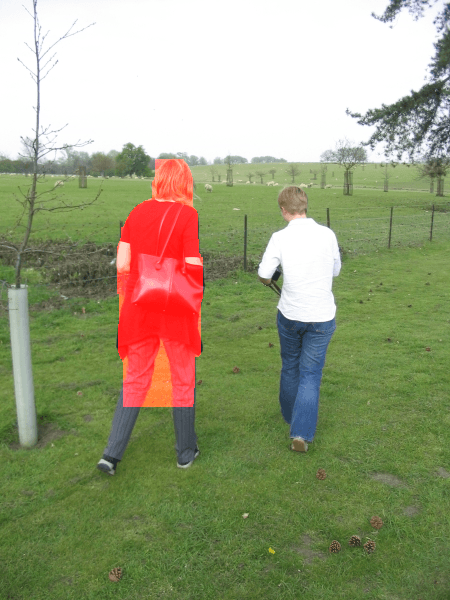}}
{\includegraphics[trim={0 2.5cm 0 2.5cm},clip,width=0.15\textwidth]{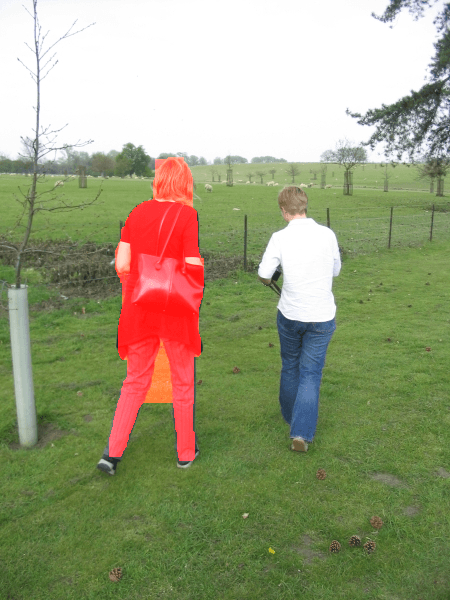}}
{\includegraphics[trim={0 2.5cm 0 2.5cm},clip,width=0.15\textwidth]{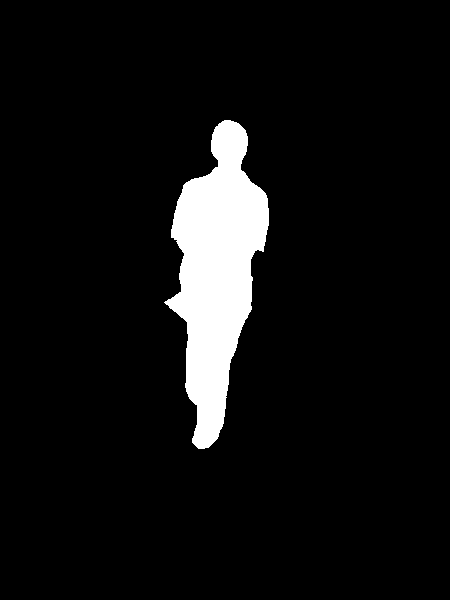}}
{\includegraphics[trim={0 2.5cm 0 2.5cm},clip,width=0.15\textwidth]{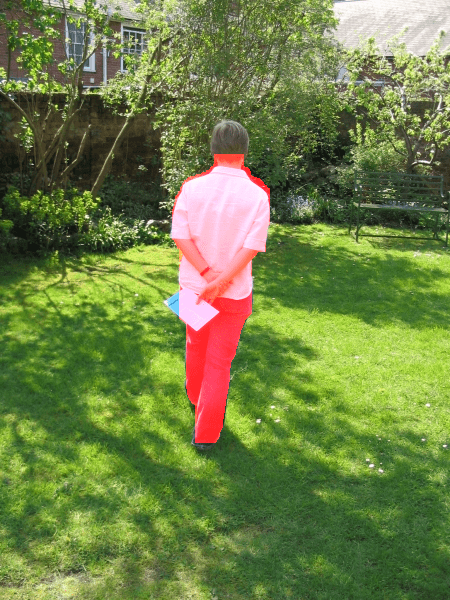}}
{\includegraphics[trim={0 2.5cm 0 2.5cm},clip,width=0.15\textwidth]{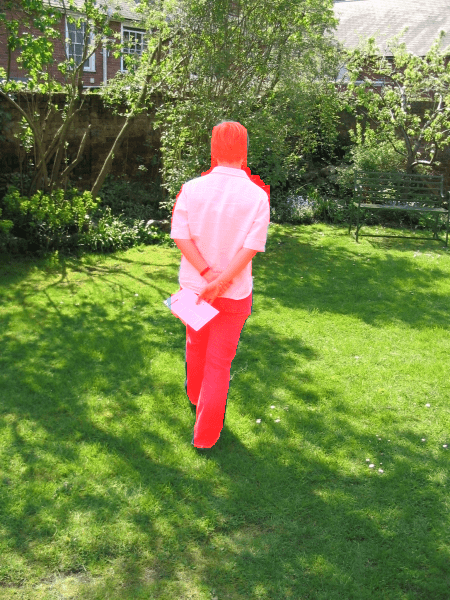}}
\subfigure[groundtruth]{\includegraphics[trim={0 2.5cm 0 2.5cm},clip,width=0.15\textwidth]{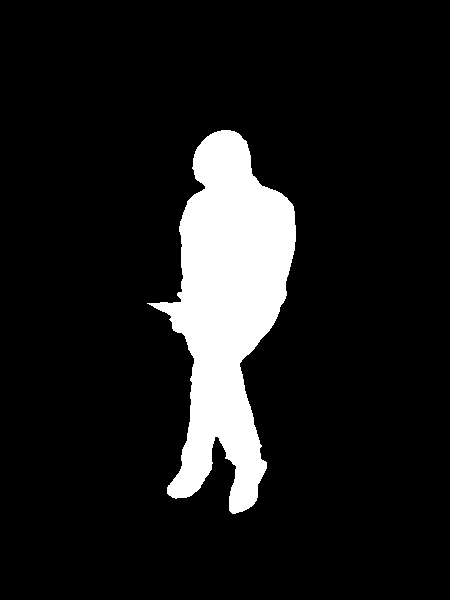}}
\subfigure[Hamming]{\includegraphics[trim={0 2.5cm 0 2.5cm},clip,width=0.15\textwidth]{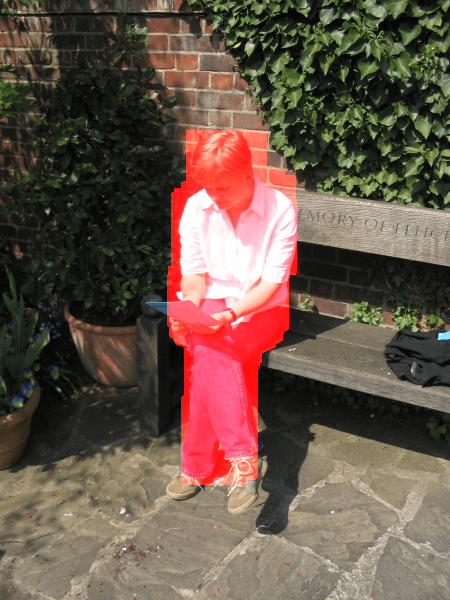}}
\subfigure[8-connected]{\includegraphics[trim={0 2.5cm 0 2.5cm},clip,width=0.15\textwidth]{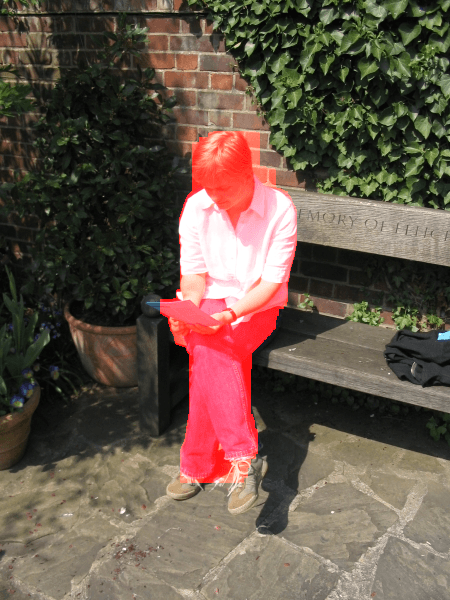}}
\subfigure[groundtruth]{\includegraphics[width=0.15\textwidth]{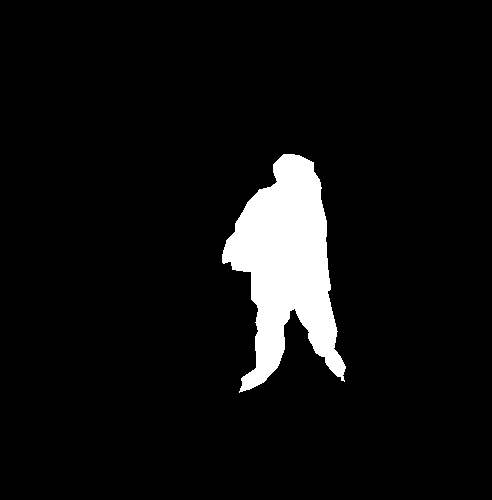}}
\subfigure[Hamming]{\includegraphics[width=0.15\textwidth]{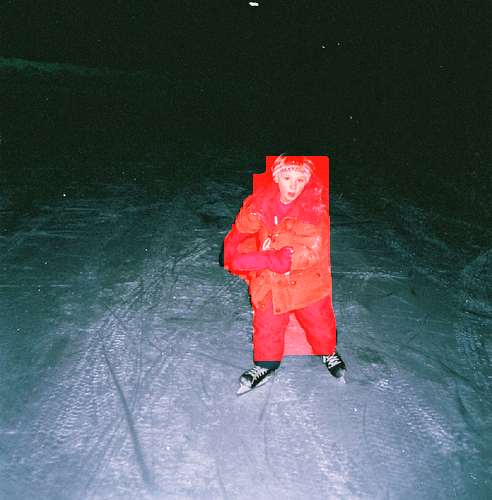}}
\subfigure[8-connected]{\includegraphics[width=0.15\textwidth]{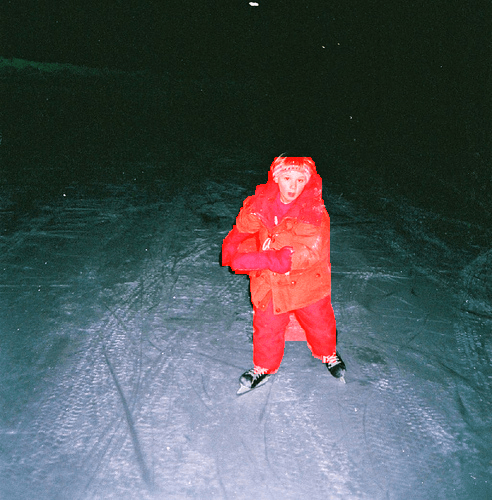}}
\caption{\label{fig:segmentation} The segmentation results of prediction trained with Hamming loss (columns 2 and 5) and the 8-connected loss (columns 3 and 6). The supermodular loss performs better on foreground object boundary than Hamming loss does, and it achieves better prediction on the elongated structures of the foreground object e.g.\ the heads and the legs.
}
\end{figure*}

\paragraph{Empirical Results} 

We show in Table~\ref{tab:loss} and Table~\ref{tab:brain} the empirical error values by training with the 8-connected supermodular loss compared with training with the Hamming loss (labeled 0-1).
In Table~\ref{tab:loss} we show the results by using different values of $\gamma$ for the 8-connected loss. We notice that in all cases, training with the same supermodular loss as used for testing has achieved the best performance, i.e.\ lower error values. Training with the supermodular loss even outperforms training with Hamming loss when measured by Hamming loss on the test set.
Wilcoxon sign rank tests are shown in Table~\ref{tab:pvalueErr}, which shows that training with the supermodular loss functions gives significantly better results in nearly all cases. 

We have additionally tried training with a joint graph cuts loss augmented inference using the pairwise potentials illustrated in Figure~\ref{fig:notSubmodularJointGraphcuts}.  However, due to the non-submodular potentials, the graph cuts procedure does not correctly minimize the energy resulting in incorrect cutting planes that causes optimization to fail after a small number of iterations.  The performance of this system was effectively random, and we have not included these values in Table~\ref{tab:loss}.

Qualitative segmentation results are shown in Figure~\ref{fig:segmentation}. 
In Figure~\ref{fig:segDiff} and \ref{fig:segDiff2} we show a pixelwise comparison of the predictions. The 8-connected loss achieves better performance on the foreground/background boundary, as well as on elongated structures of the foreground object, such as the head and legs, especially when the appearance of the foreground is similar to the background.

\begin{table}[]
\resizebox{0.99\linewidth}{!}{
\begin{tabular}{c|c|c|c|c}
\hline
\multicolumn{2}{c|}{} & \multicolumn{3}{|c}{Eval.} \\
\cline{3-5}
\multicolumn{2}{c|}{$\gamma = 0.25$} & 0-1(1e3) & $\Delta_8$(1e3) & IoU \\
\hline
\multirow{2}{*}{\rotatebox{90}{\small{Train.}}}
 & 0-1		&	$ 3.245\pm 0.137$ &  $6.097\pm 0.267$ & $0.2209\pm 0.0075$ \\
 & $\Delta_8$ &	$ 3.097\pm 0.141$ &  $5.807\pm 0.274$ & $0.2166\pm 0.0086$\\
\hline
\end{tabular}
}
\resizebox{0.99\linewidth}{!}{
\begin{tabular}{c|c|c|c|c}
\hline
\multicolumn{2}{c|}{} & \multicolumn{3}{|c}{Eval.} \\
\cline{3-5}
\multicolumn{2}{c|}{$\gamma = 0.5$} & 0-1(1e3) & $\Delta_8$(1e3) & IoU \\
\hline
\multirow{2}{*}{\rotatebox{90}{\small{Train.}}}
 & 0-1		&	$3.245\pm 0.137$ &  $8.950\pm 0.398$ & $0.2209\pm 0.0075$ \\
 & $\Delta_8$ &	$3.032\pm 0.149$ &  $8.329\pm 0.426$ & $0.2123\pm 0.0071$ \\
\hline
\end{tabular}
}

\resizebox{0.99\linewidth}{!}{
\begin{tabular}{c|c|c|c|c}
\hline
\multicolumn{2}{c|}{} & \multicolumn{3}{|c}{Eval.} \\
\cline{3-5}
\multicolumn{2}{c|}{$\gamma = 0.75$} & 0-1(1e3) & $\Delta_8$(1e3) & IoU \\
\hline
\multirow{2}{*}{\rotatebox{90}{\small{Train.}}}
 & 0-1		&	$ 3.245\pm 0.137$ &  $11.802\pm 0.528 $ & $0.2209\pm 0.0075$ \\
 & $\Delta_8$ &	$ 2.841\pm 0.138$ &  $10.250\pm 0.519 $ & $0.2054\pm 0.0066$ \\
\hline
\end{tabular}
}
\resizebox{0.99\linewidth}{!}{
\begin{tabular}{c|c|c|c|c}
\hline
\multicolumn{2}{c|}{} & \multicolumn{3}{|c}{Eval.} \\
\cline{3-5}
\multicolumn{2}{c|}{$\gamma = 1.0$} & 0-1(1e3 & $\Delta_8$(1e3)) & IoU \\
\hline
\multirow{2}{*}{\rotatebox{90}{\small{Train.}}}
 & 0-1		&	$ 3.245\pm 0.137$ &  $14.655\pm 0.659$ & $0.2209\pm 0.0075$ \\
 & $\Delta_8$ &	$ 2.863\pm 0.124$ &  $12.822\pm 0.585$ & $0.2065\pm 0.0075$ \\
\hline
\end{tabular}
}
\caption{\label{tab:loss}The cross comparison of average loss values (with standard error) using the 8-connected loss function ($\Delta_8$) and Hamming loss (labeled 0-1) during training.  During testing, we evaluate with the Hamming loss, the 8-connected loss and the Intersection over union loss (labeled IoU). Training with the same supermodular loss functions as used during testing yields the best results. Training with supermodular losses even outperforms the Hamming loss in terms of evaluating by Hamming loss. 
}
\end{table}

\begin{table}\centering
\begin{tabular}{c|c|c|c}
\hline
								& \multicolumn{3}{|c}{Eval.} \\
\cline{2-4}
								& 0-1		& $\Delta_8$ & IoU \\
\hline
$\Delta_8, \gamma = 0.25$ vs 0-1 	& $0.0195$ & $0.0195$ & $0.1055$ \\
\hline
$\Delta_8,\gamma = 0.5$ vs 0-1 	& $0.0371$ & $0.0371$ & $0.0273$\\
\hline
$\Delta_8,\gamma = 0.75$ vs 0-1 	& $0.0020$ & $0.0020$ & $0.0020$ \\
\hline
$\Delta_8,\gamma = 1.0$ vs 0-1 	& $0.0195$ & $0.0273$ & $0.0488$\\
\hline
\end{tabular}
\caption{\label{tab:pvalueErr}Wilcoxon sign rank test on the error values in Table~\ref{tab:loss} comparing training with Hamming loss (labeled 0-1) and the 8-connected loss $\Delta_8$ (with different values of $\gamma$).}
\end{table}

\begin{table}[]\centering
\resizebox{0.99\linewidth}{!}{
\begin{tabular}{c|c|c|c|c}
\hline
\multicolumn{2}{c|}{} & \multicolumn{3}{|c}{Eval.} \\
\cline{3-5}
\multicolumn{2}{c|}{$\gamma = 0.5$} & $\Delta_8$(1e3) & 0-1(1e3) & IoU \\
\hline
\multirow{2}{*}{\rotatebox{90}{\small{Train.}}}
 & $\Delta_8$ &	$ \mathbf{2.616 \pm 0.612} $  &  $1.297\pm0.224$ & $0.169\pm 0.018$ \\
 & 0-1       &	$2.885\pm0.765$            &   $1.393\pm0.279$ & $0.173\pm 0.019 $ \\
\hline
\end{tabular}
}
\caption{The cross comparison of average loss values on IBSR dataset (cf.\ comments for Table~\ref{tab:loss}).\label{tab:brain} }
\end{table}

We also ran a baseline comparing non-submodular loss augmented inference with the QPBO approach \cite{QPBOpaper}.  We computed pairwise energies as in Figure~\ref{fig:8connectedLoss}.  QPBO found loss augmented energies across the dataset of $1.1\times10^6 \pm 3\times10^5$ while ADMM found loss augmented energies of $3.7\times10^6 \pm 8 \times 10^5$, a substantial improvement.

\subsection{Training with the square loss and the biconvex loss}\label{sec:experimentNew}
We show in Table~\ref{tab:SquareAndBiconvex} the empirical error values by training with the square loss (labeled $\Delta_S$), and with the biconvex loss (labled $\Delta_C$), compared to training with the Hamming loss (labeled 0-1).
We can see that training with the same supermodular loss during test time yields better performance than training with the Hamming loss, which validates the correctness of the ADMM splitting strategy with more loss/inference combinations.
\begin{table*}[ht]\centering
\begin{tabular}{c|c|c|c|c|c}
\hline
\multicolumn{2}{c|}{} & \multicolumn{4}{|c}{Eval.} \\
\cline{3-6}
\multicolumn{2}{c|}{} & 0-1(1e3) & Square loss(1e3)  $\Delta_{\text{S}}$ & Biconvex loss $\Delta_{\text{C}}$  & IoU loss\\
\hline
\multirow{3}{*}{\rotatebox{90}{\small{Train.}}}
 & 0-1 					& $3.245\pm 0.137$ & $0.257\pm 0.019$ & $0.217\pm 0.012$ & $0.221\pm 0.007$ \\
 & $\Delta_{\text{S}}$	& $2.928\pm 0.418$ & $0.251\pm 0.040$ & $0.176\pm 0.022$ & $0.196\pm 0.016$ \\
 & $\Delta_{\text{C}}$	& $2.394\pm 0.166$ & $0.202\pm 0.019$ & $0.149\pm 0.011$ & $0.179\pm 0.007$ \\
\hline
\end{tabular}
\caption{The cross comparison of average loss values (with standard error) using supermodular losses and Hamming loss (labeled 0-1) during training and test time). We additionally evaluate on the intersection over union loss.}\label{tab:SquareAndBiconvex}
\end{table*}

Qualitative segmentation results are shown in Figure~\ref{fig:segmentation2}. Pixelwise comparison of the segmentation results using the square loss and the biconvex loss are shown in Figure~\ref{fig:segDiffSquare} and Figure~\ref{fig:segDiffBiconvex}, respectively.

\subsection{Computation Time} 
In addition, when using the 8-connected loss, we compare the time of one calculation of the loss augmented inference by the ADMM algorithm and by the minimum norm point algorithm \cite{fujishige2005submodular} (MinNorm). For MinNorm, we use the implementation provided in the SFO toolbox \cite{krause2010sfo}. Although it has been proven that in $t$ iterations, the MinNorm returns an $O(1/t)$-approximate solution \cite{chakrabarty2014provable}, the first step of this algorithm is to find a point in the submodular polytope, which alone is computationally intractable even for small $600\times400$ pixel images. Therefore, we measure the computation time on downsampled images, showing 
the growth in computation as a function of image size (Figure~\ref{fig:timehist} and Figure~\ref{fig:timeBox}). The running times are recorded on a machine with a 3.20GHz CPU.  Similarly, a dual-decomposition baseline took orders of magnitude longer computation than the ADMM approach, following known convergence results \cite{boyd2011distributed}.

We measure the computation time for 120 calculations of the loss augmented inference by ADMM and MinNorm on different sized images. 
From Figure~\ref{fig:timehist} and Figure~\ref{fig:timeBox} we can see that ADMM is always faster than the MinNorm by a substantial margin, and around 100 times faster when the problem size reaches $10^3$. The computing time for both ADMM and MinNorm vary approximately linearly in log-log scale, while MinNorm has a higher slope, suggesting a worse big-$\mathcal{O}$ computational complexity.  We note that theoretical bounds on MinNorm are currently weak and the exact complexity is unknown \cite{chakrabarty2014provable}. 

Although it is immediately clear from Figure~\ref{fig:timeBox} that ADMM is substantially faster than the minimum norm point algorithm, we have performed Wilcoxon sign rank tests that show this difference is significant with $p<10^{-20}$ in all settings.
\begin{figure}[tbp]
\centering
\subfigure[Size=600]{\label{fig:timehist600}\includegraphics[width=0.224\textwidth]{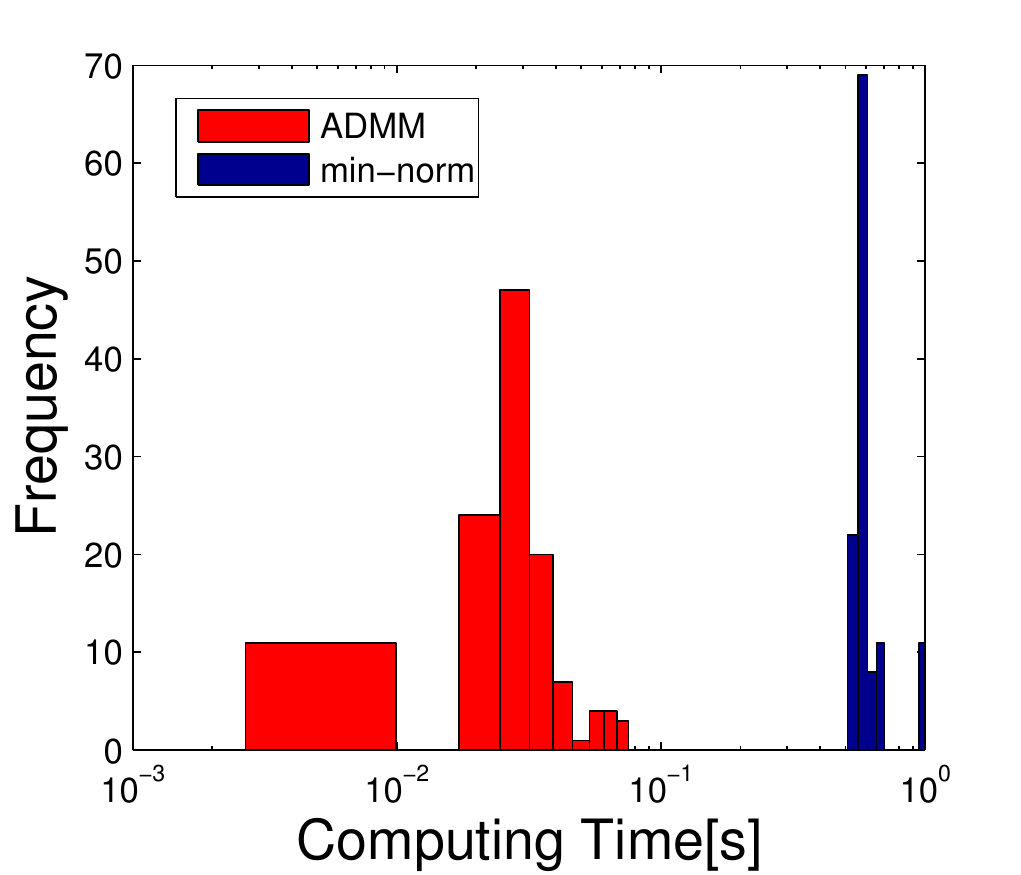}} 
\subfigure[Size=1200]{\label{fig:timehist1200}\includegraphics[width=0.224\textwidth]{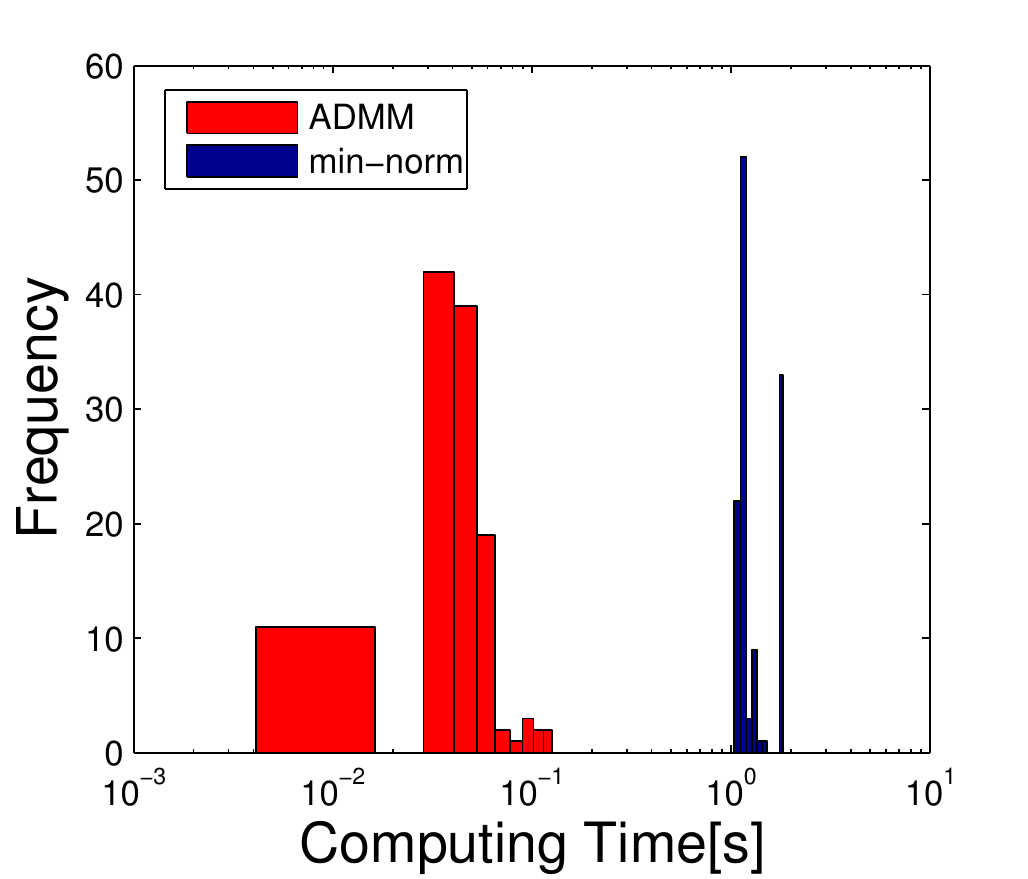}}
\subfigure[Size=2400]{\label{fig:timehist2400}\includegraphics[width=0.224\textwidth]{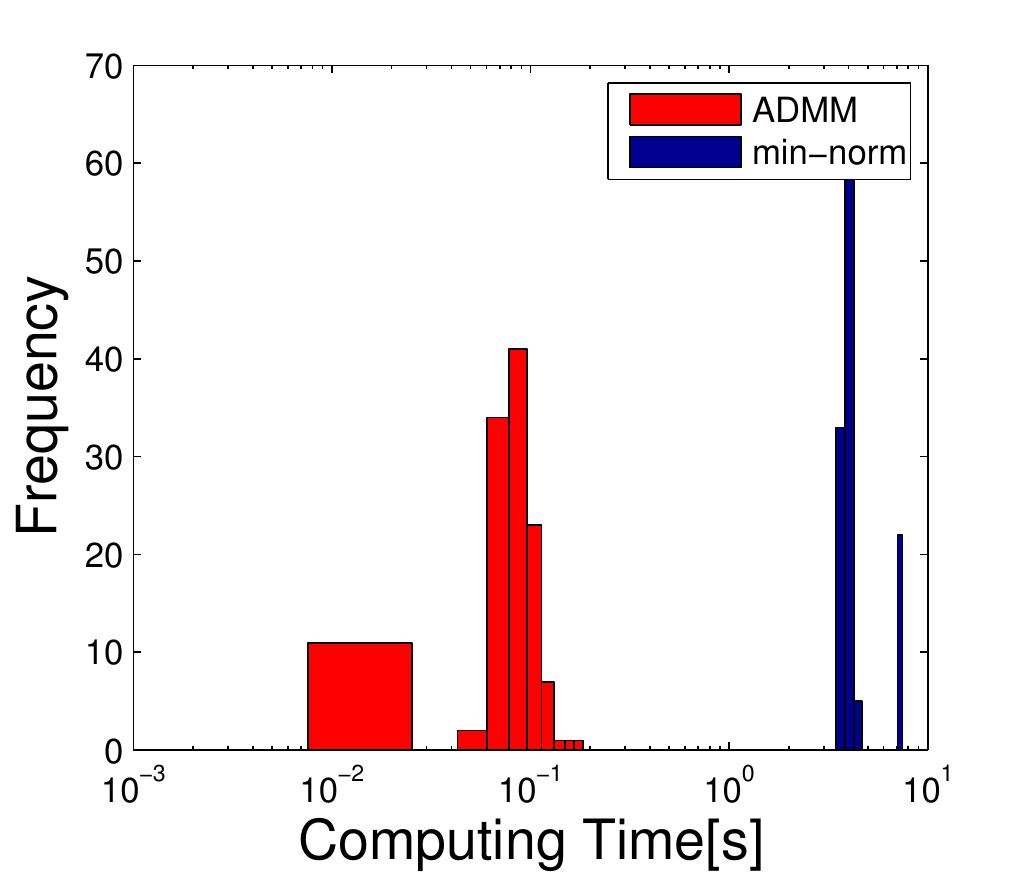}}
\subfigure[Size=4800]{\label{fig:timehist4800}\includegraphics[width=0.224\textwidth]{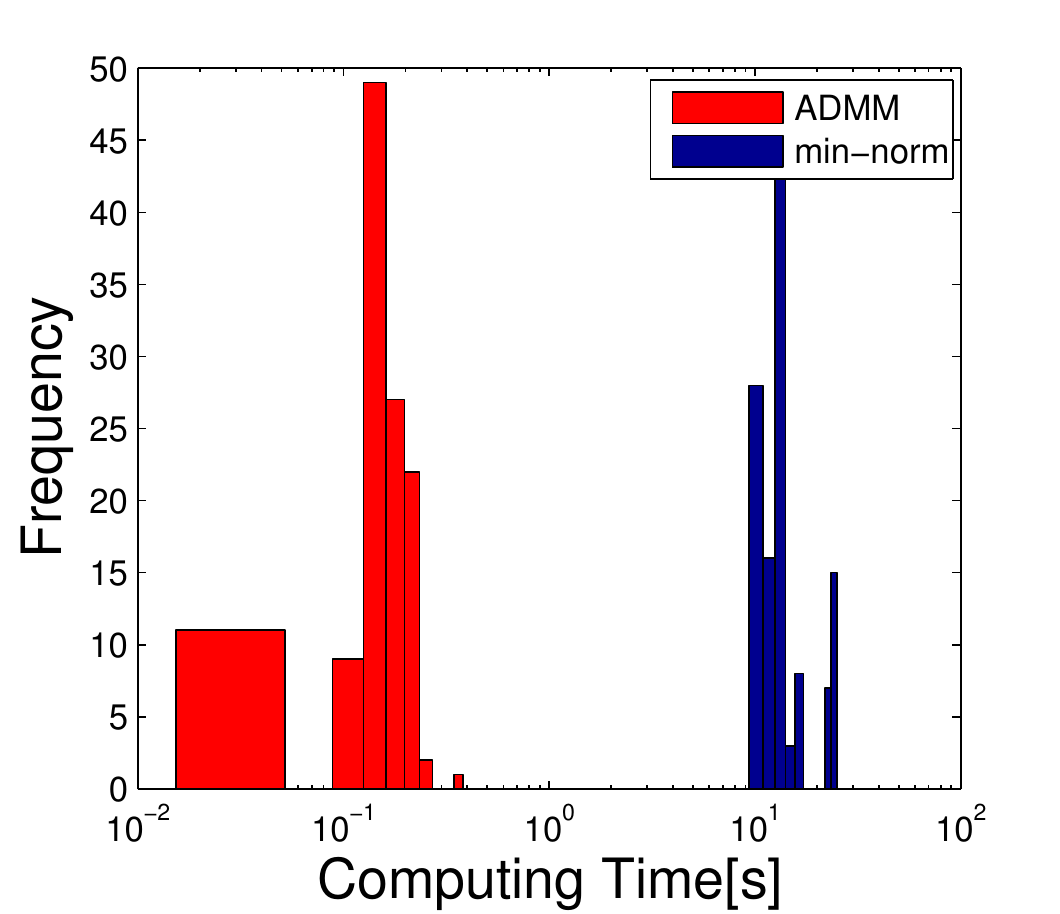}}
\caption{\label{fig:timehist} The computing time for the loss augmented inference, on different problem sizes. The {\color{red}red} histograms stands for ADMM and the {\color{blue}blue} for MinNorm. The calculation by ADMM is always faster than by MinNorm, and there is no overlap between the computing time by the two methods.}
\end{figure} 
\begin{figure}[]
\centering
\includegraphics[width=0.85\linewidth]{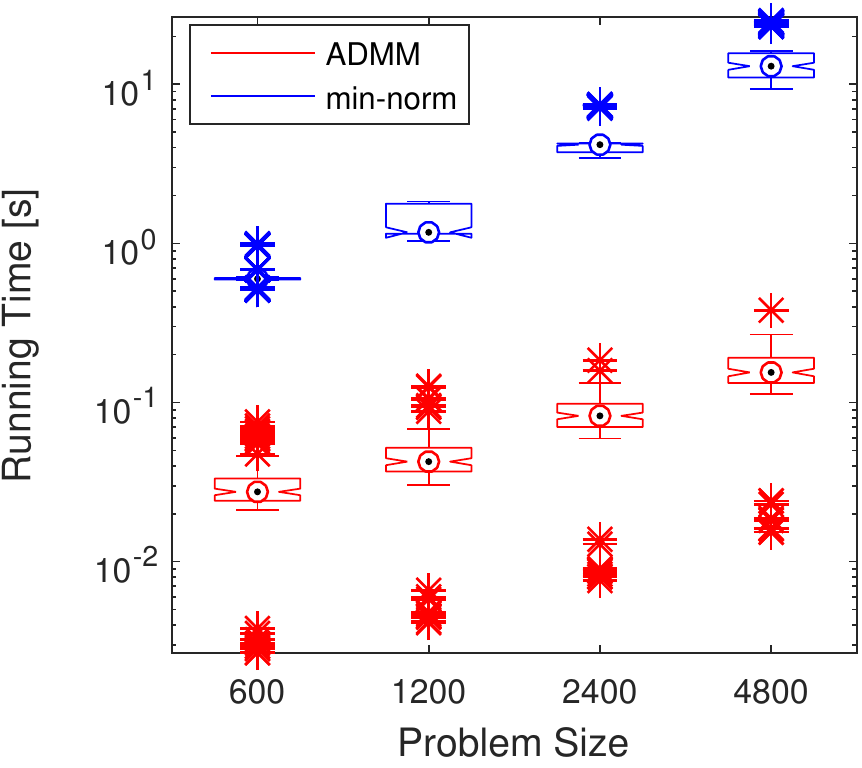}
\caption{\label{fig:timeBox} The running time increase along with the problem size. Both algorithm increase linearly in log scale while the ADMM has a time reduction from $10$ times to $10^2$ times along with the increase of the problem size.}
\end{figure}

\subsection{Comparison to LP-relaxation} 

We additionally compare ADMM to an LP relaxation procedure for the loss augmented inference to determine the accuracy of our optimization in practice, with using the 8-connected loss function and the Hamming loss (0-1).
For the implementation of the LP relaxation, we use the \texttt{UGM} toolbox \cite{schmidt2012ugm}.
We show in Table~\ref{tab:ADMMvsLP} the comparison between using ADMM and the LP relaxation. The first column represents the energy achieved by the loss augmented inference (Equation~\eqref{eq:maxenergy}). We observe that the (maximal) energy achieved by ADMM is almost the same as the LP relaxation: a difference of $0.4\%$. Columns 2--4 show the computing time for one calculation of the loss augmented inference on the downsampled images. Using an LP relaxation, the computation time is orders of magnitude slower, growing as a function of the image size. ADMM provides a more efficient strategy without loss of performance.

\begin{table}\centering
\resizebox{0.99\linewidth}{!}{
\begin{tabular}{c|c|c|c|c}
\hline
		& $-E$ & size $=600$   & size $=1200$   & size $=2400$  \\ 
\hline 
ADMM		& $2.28\pm 0.58$ & $0.035\pm 0.002$ & $0.051\pm 0.002$ & $0.864\pm 0.476$ \\
\hline
LP		& $2.29\pm 0.57$ & $1.857\pm 0.128$ & $3.946\pm 0.286$ & $13.57\pm 1.359$\\
\hline
\end{tabular}}
\caption{\label{tab:ADMMvsLP} The comparison between ADMM and an LP relaxation for solving the loss augmented inference. The 1st column shows the optimal energy values ($10^3$) (Equation~\eqref{eq:maxenergy}); columns 2--4 show the computation time (s) for one calculation on downsampled images of varying size.}
\end{table}

\begin{figure*}[]
\centering
\subfigure[]{\includegraphics[width=0.225\linewidth]{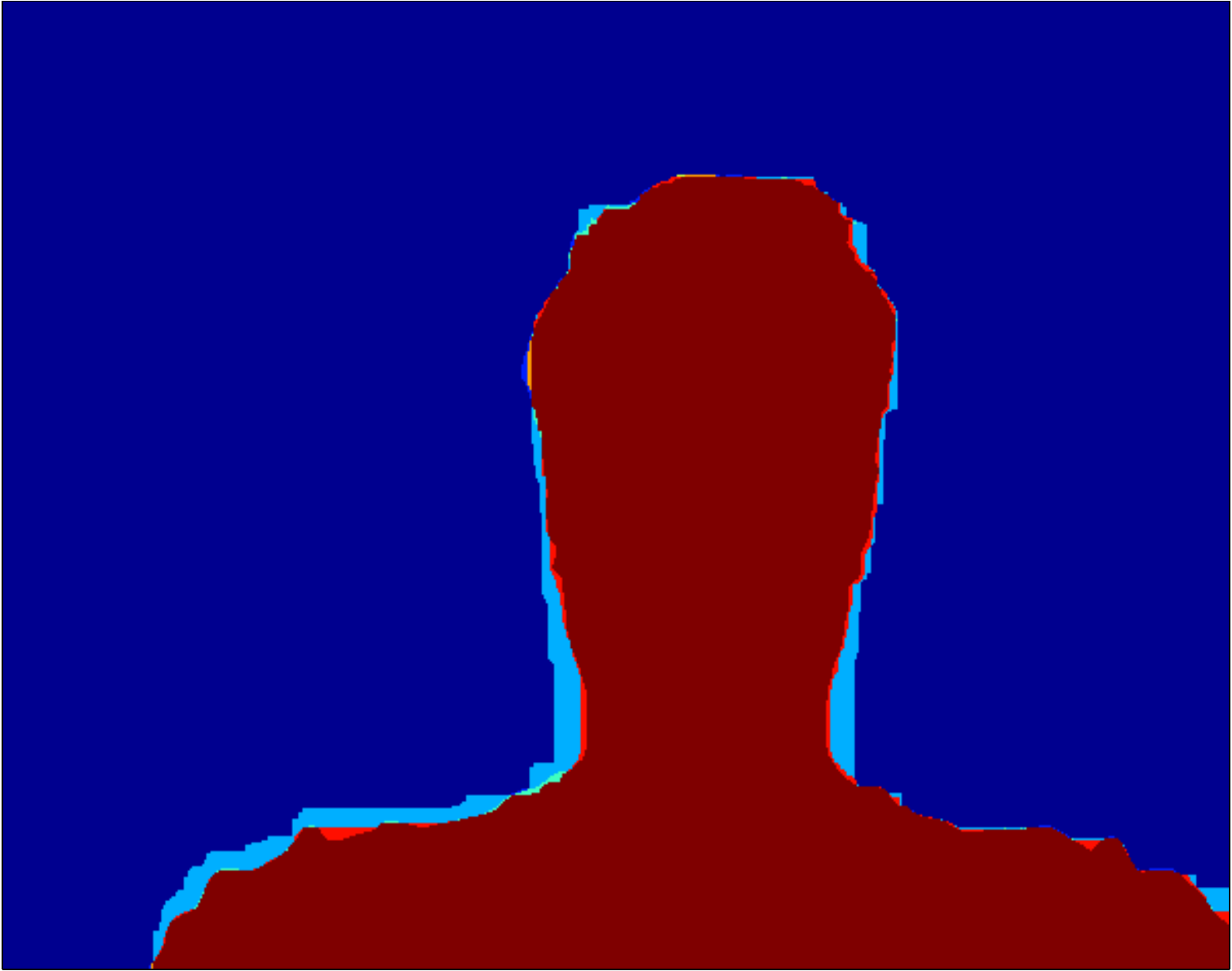}\label{fig:SegmentationComparison1}}
\subfigure[]{\includegraphics[width=0.225\linewidth]{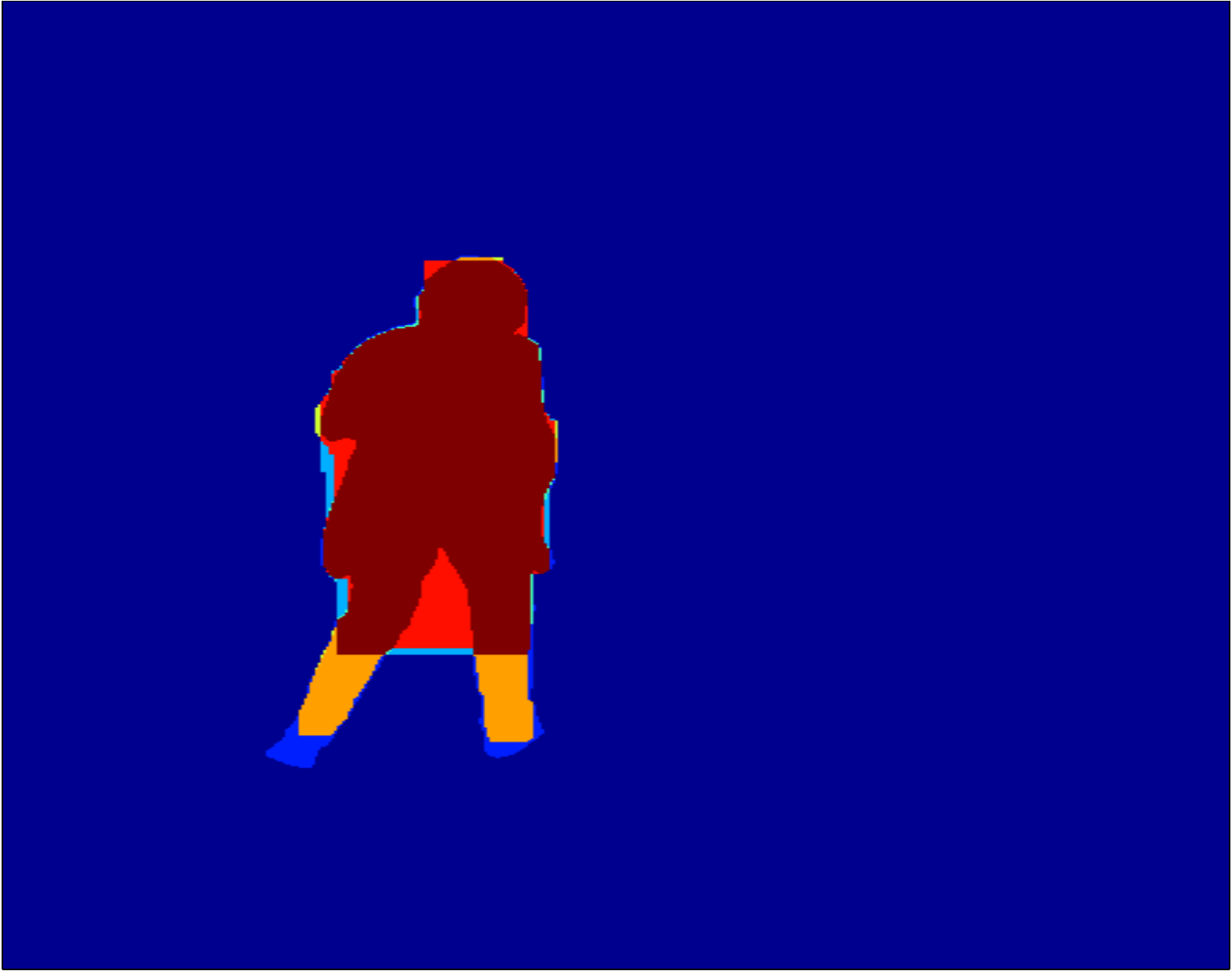}\label{fig:SegmentationComparison2}}
\subfigure[]{\includegraphics[width=0.225\linewidth]{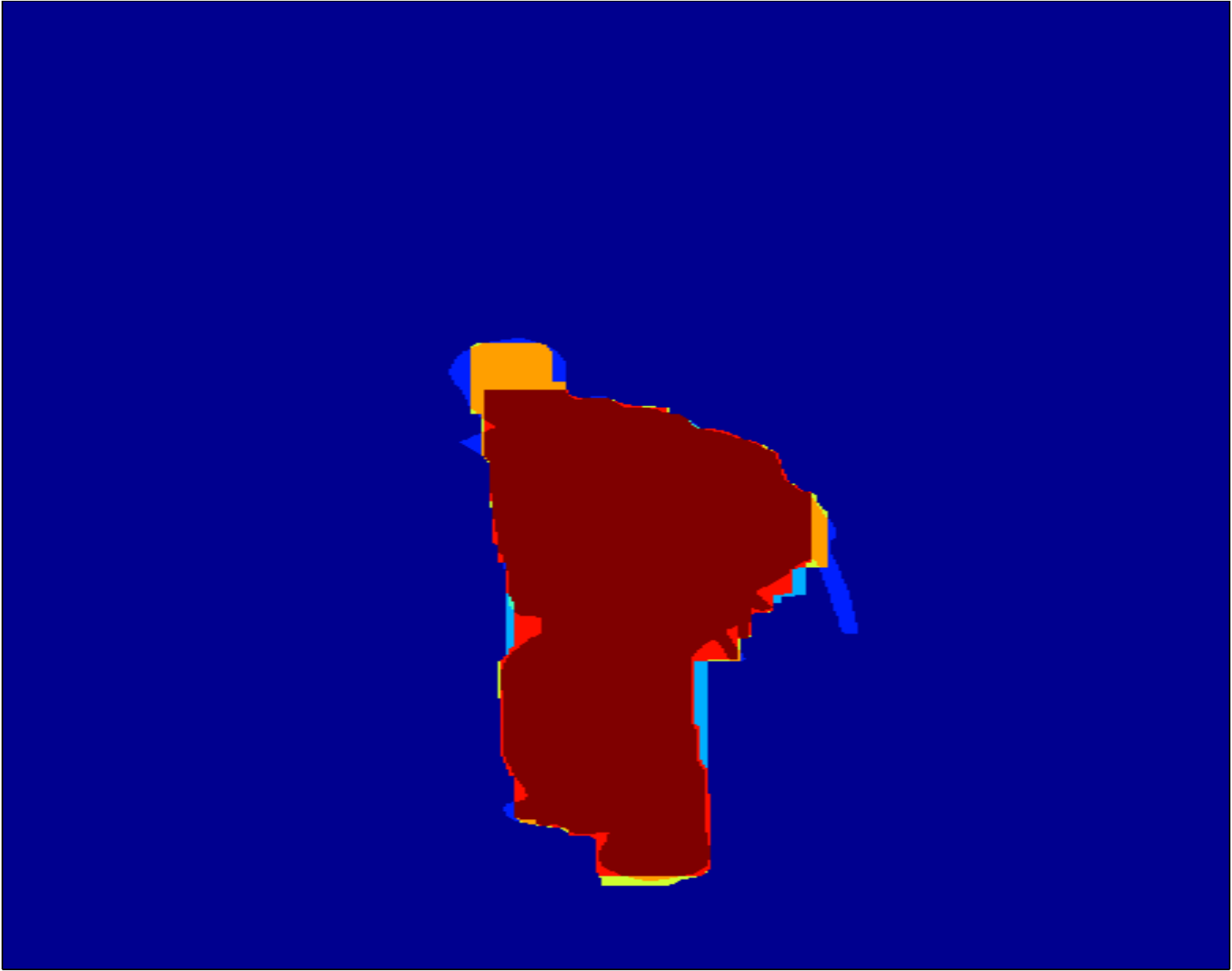}\label{fig:SegmentationComparison4}}
\subfigure[]{\includegraphics[width=0.225\linewidth]{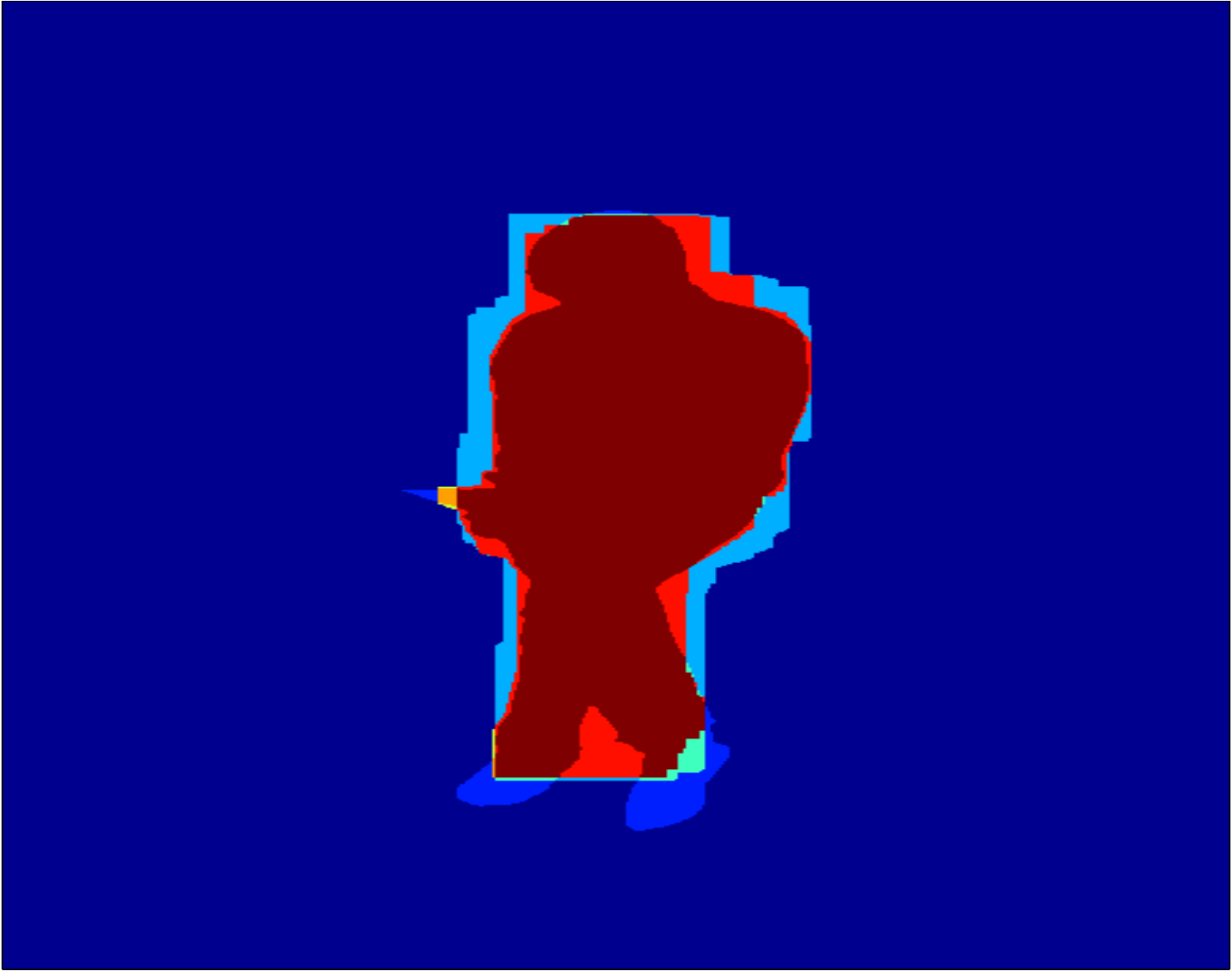}\label{fig:SegmentationComparison5}}
\subfigure[]{\includegraphics[width=0.225\linewidth]{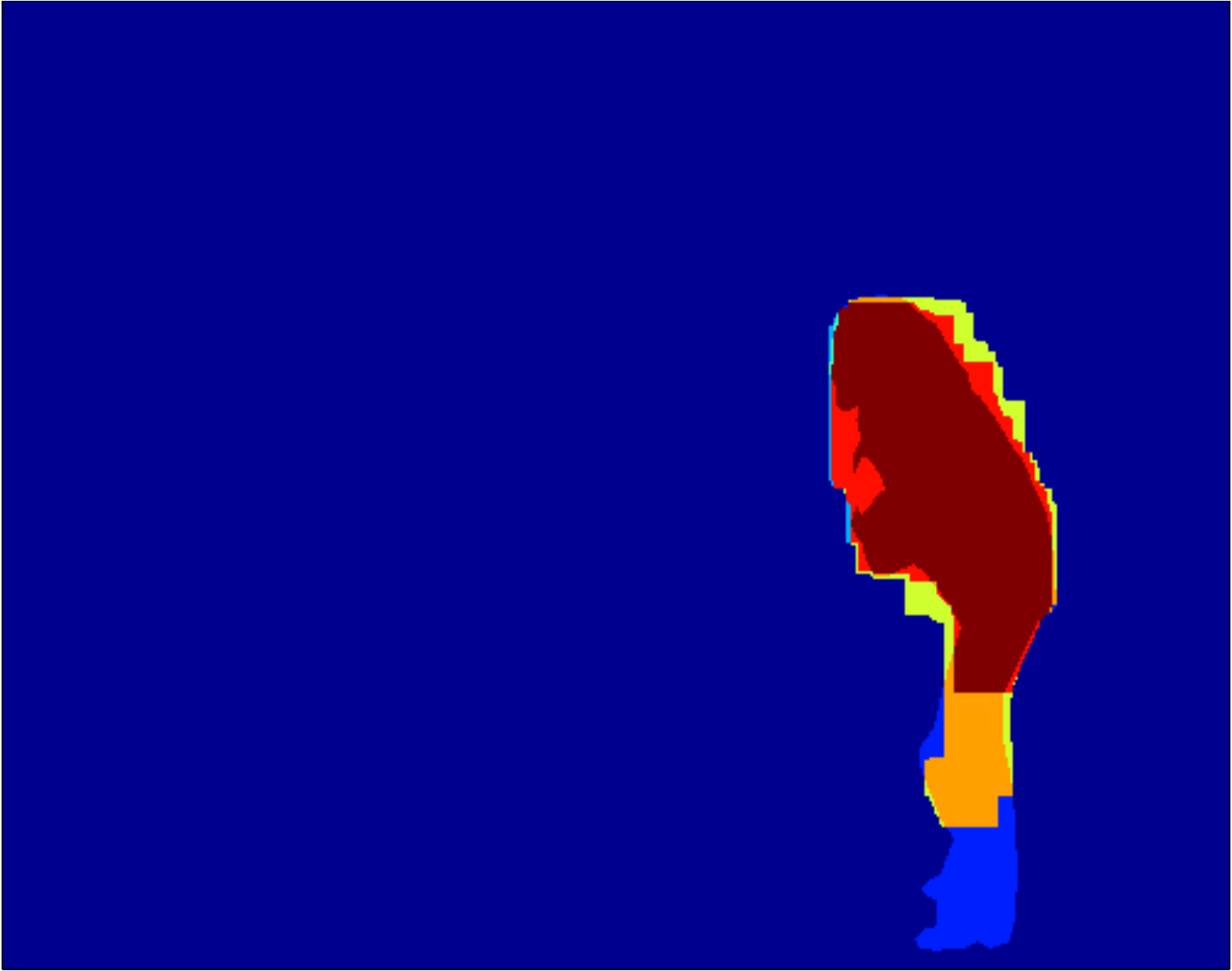}\label{fig:SegmentationComparison6}}
\subfigure[]{\includegraphics[width=0.225\linewidth]{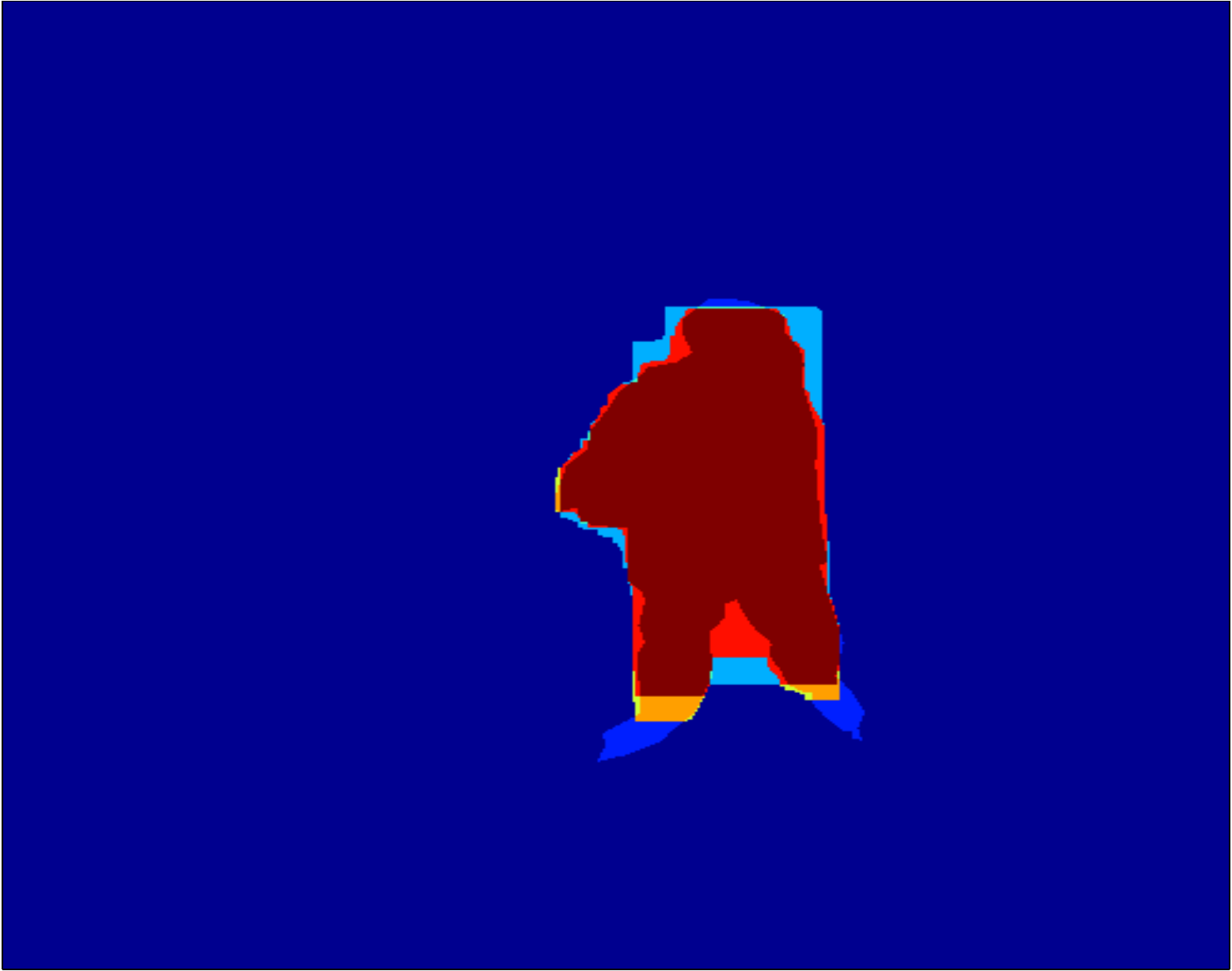}\label{fig:SegmentationComparison7}}
\subfigure[]{\includegraphics[width=0.225\linewidth]{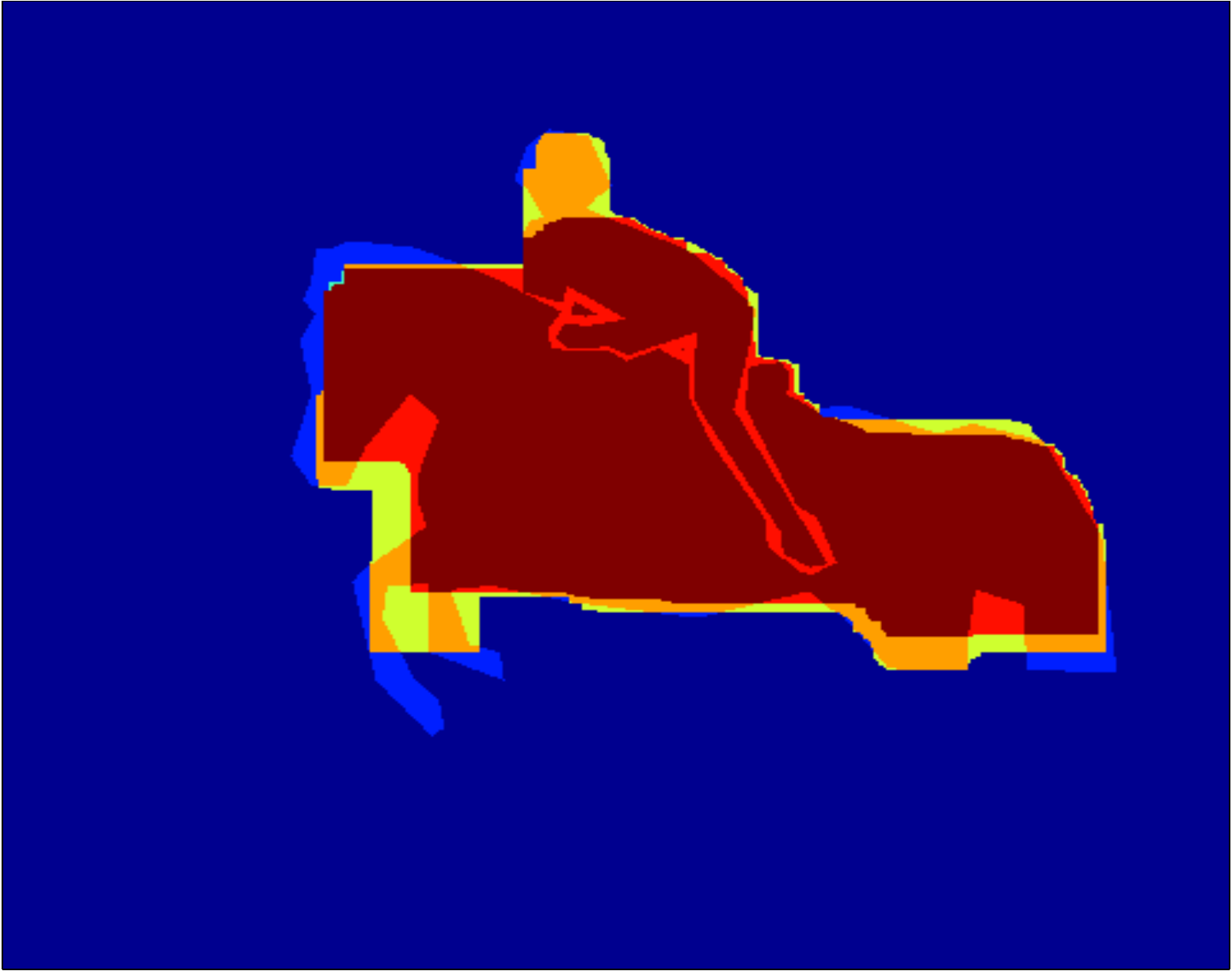}\label{fig:SegmentationComparison8}}
\subfigure[]{\includegraphics[width=0.225\linewidth]{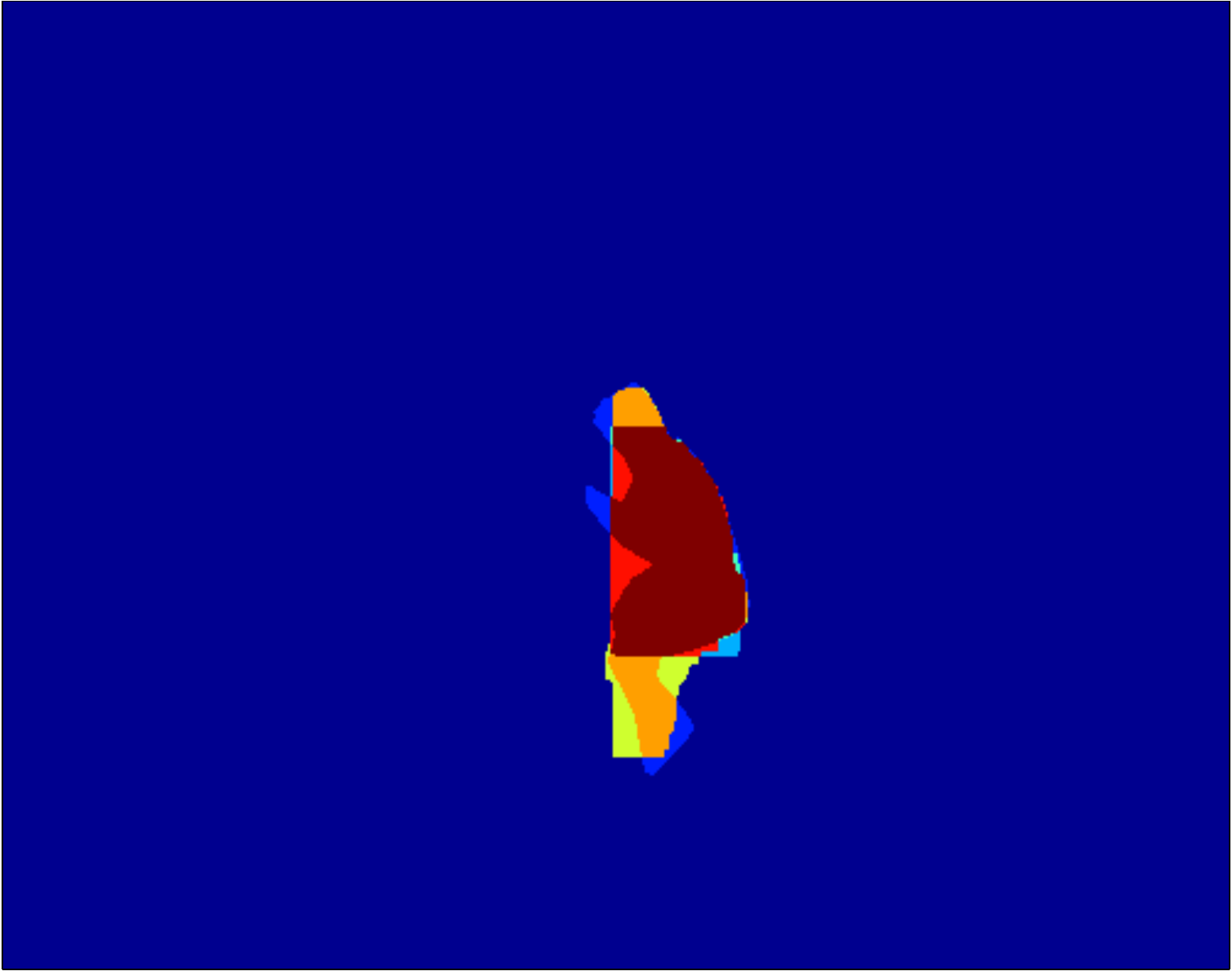}\label{fig:SegmentationComparison9}}
{\tiny
\begin{tabular}{p{0.05\linewidth}p{0.05\linewidth}p{0.05\linewidth}p{0.05\linewidth}p{0.05\linewidth}p{0.05\linewidth}p{0.05\linewidth}p{0.05\linewidth}}
\includegraphics[width=1\linewidth]{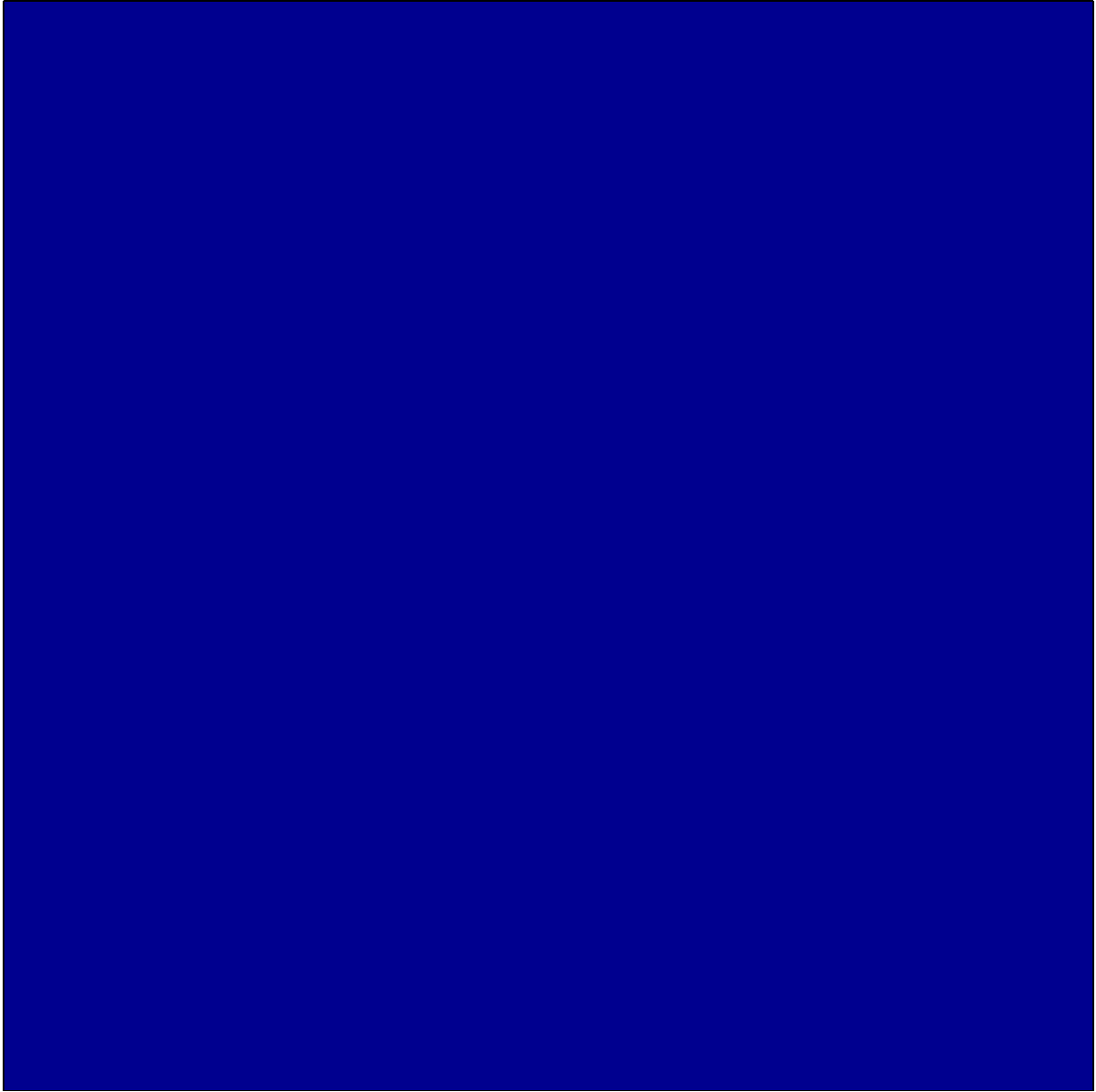} &
\includegraphics[width=1\linewidth]{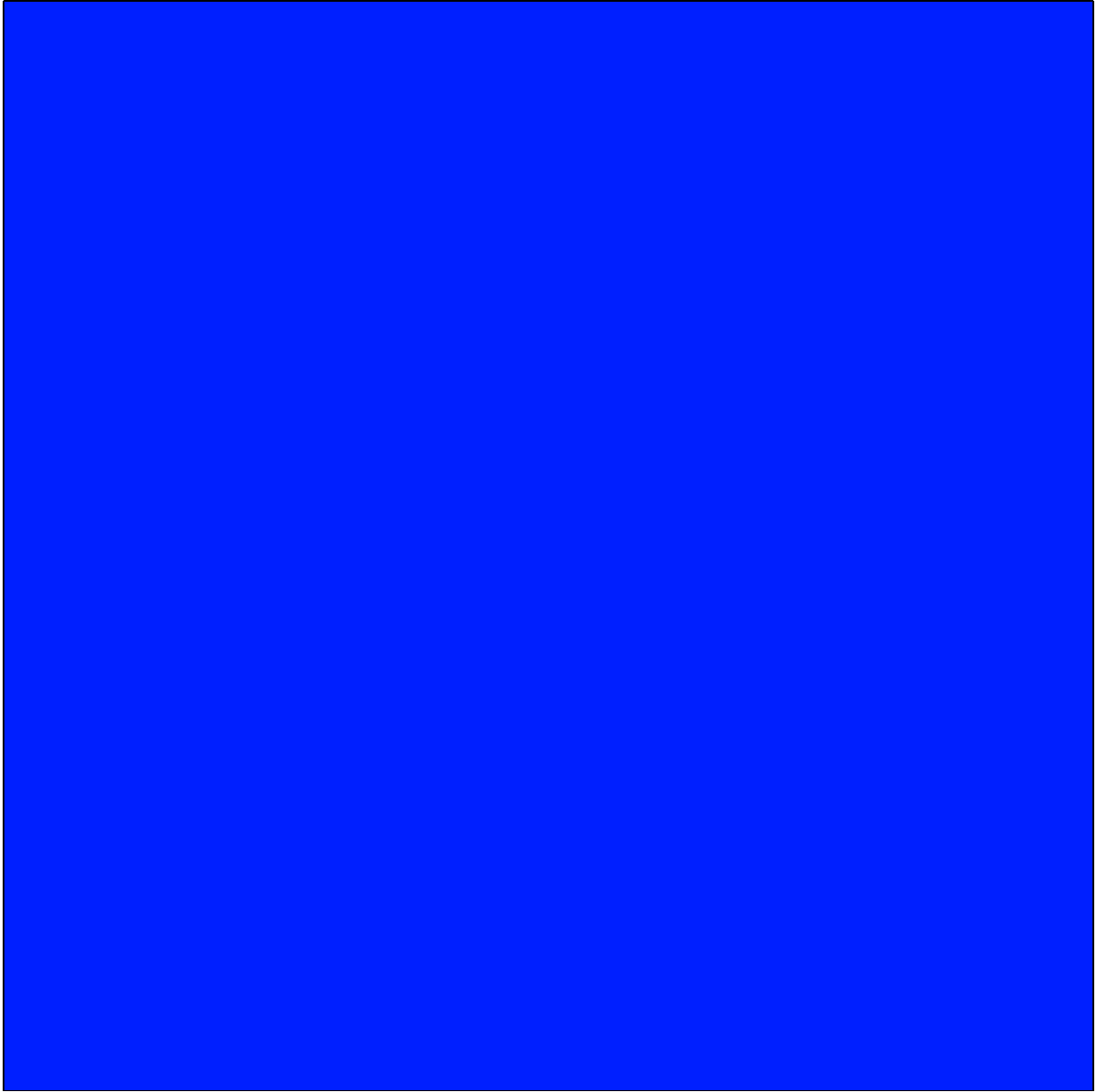} &
\includegraphics[width=1\linewidth]{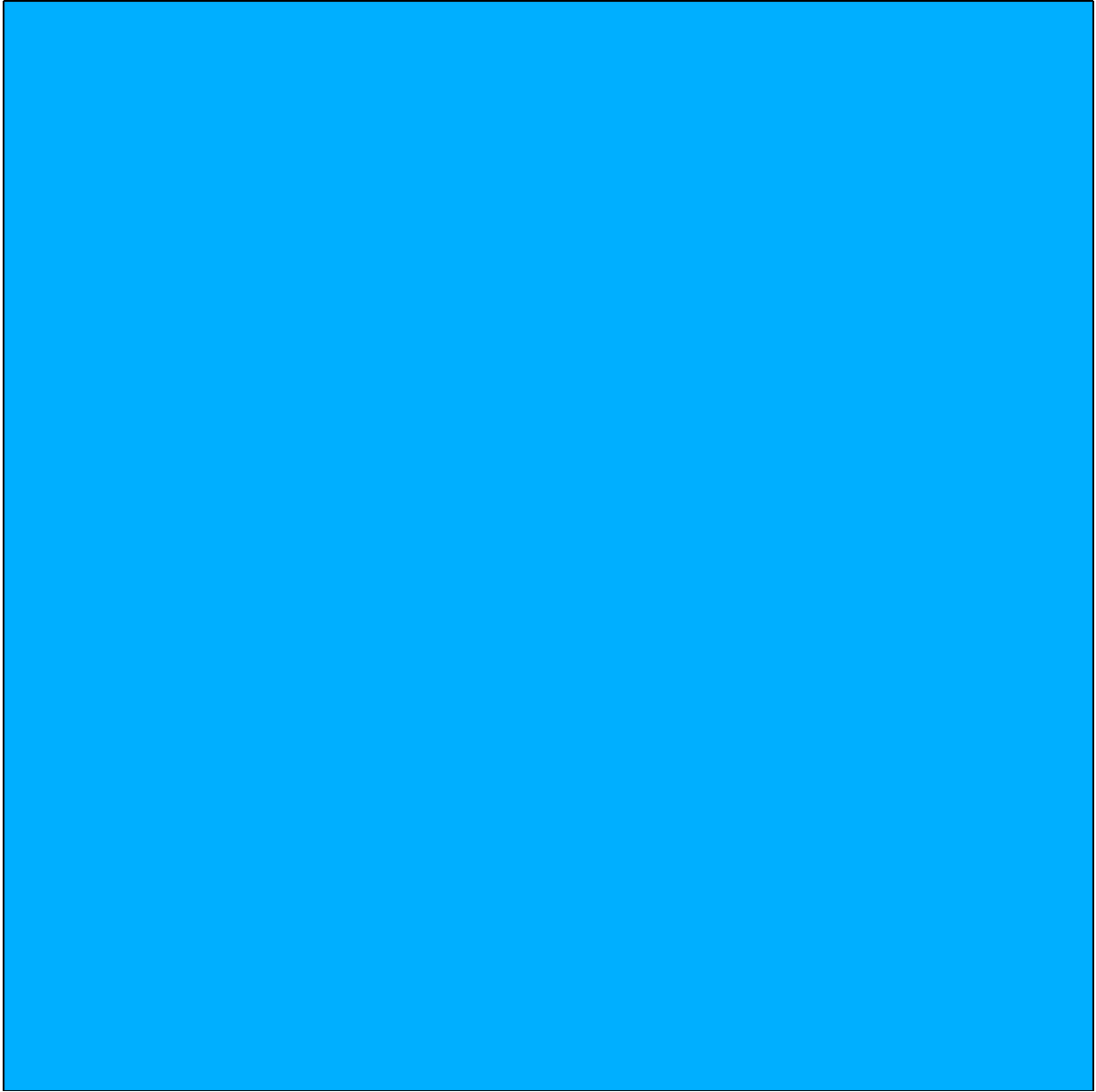} &
\includegraphics[width=1\linewidth]{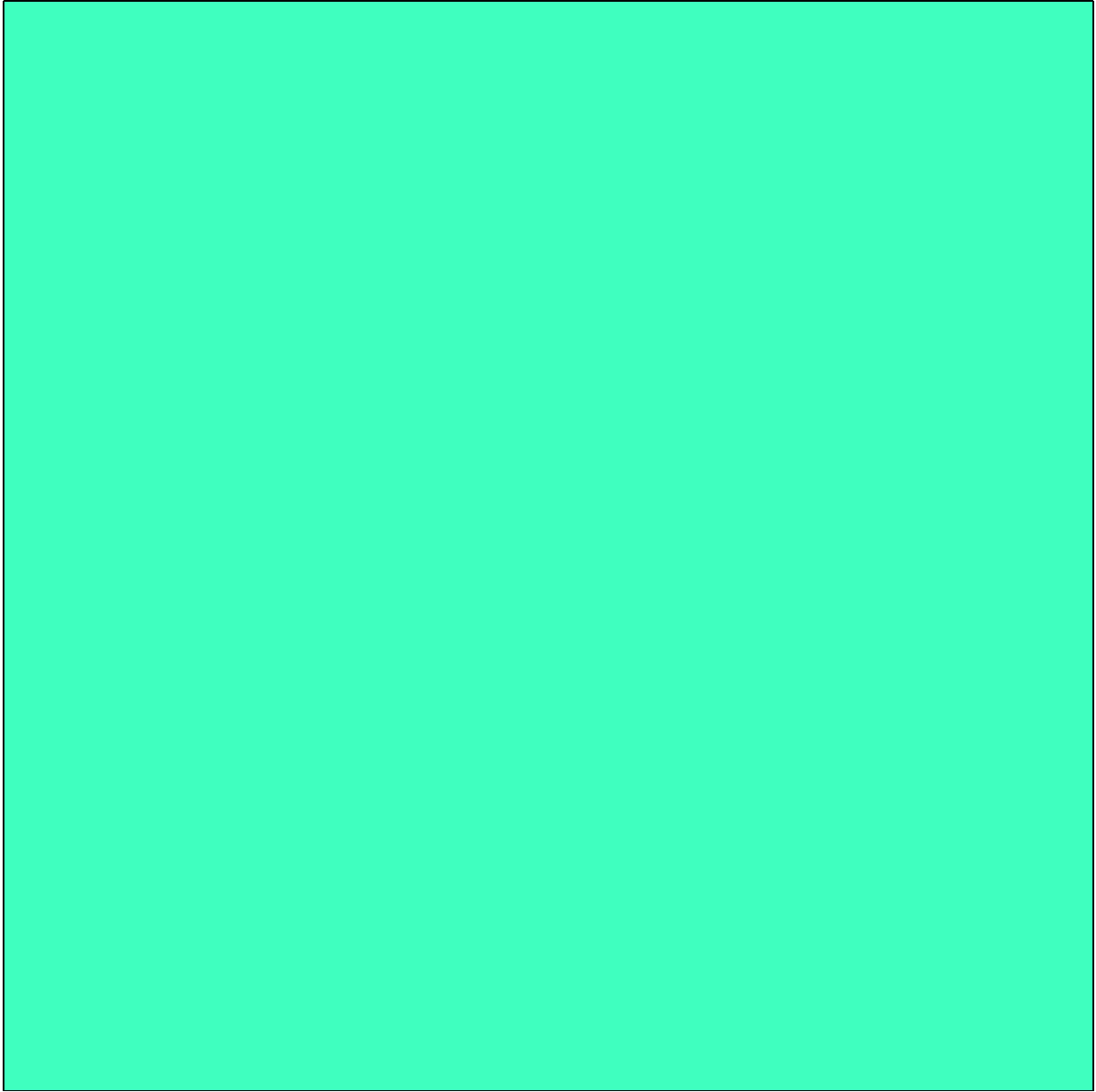} &
\includegraphics[width=1\linewidth]{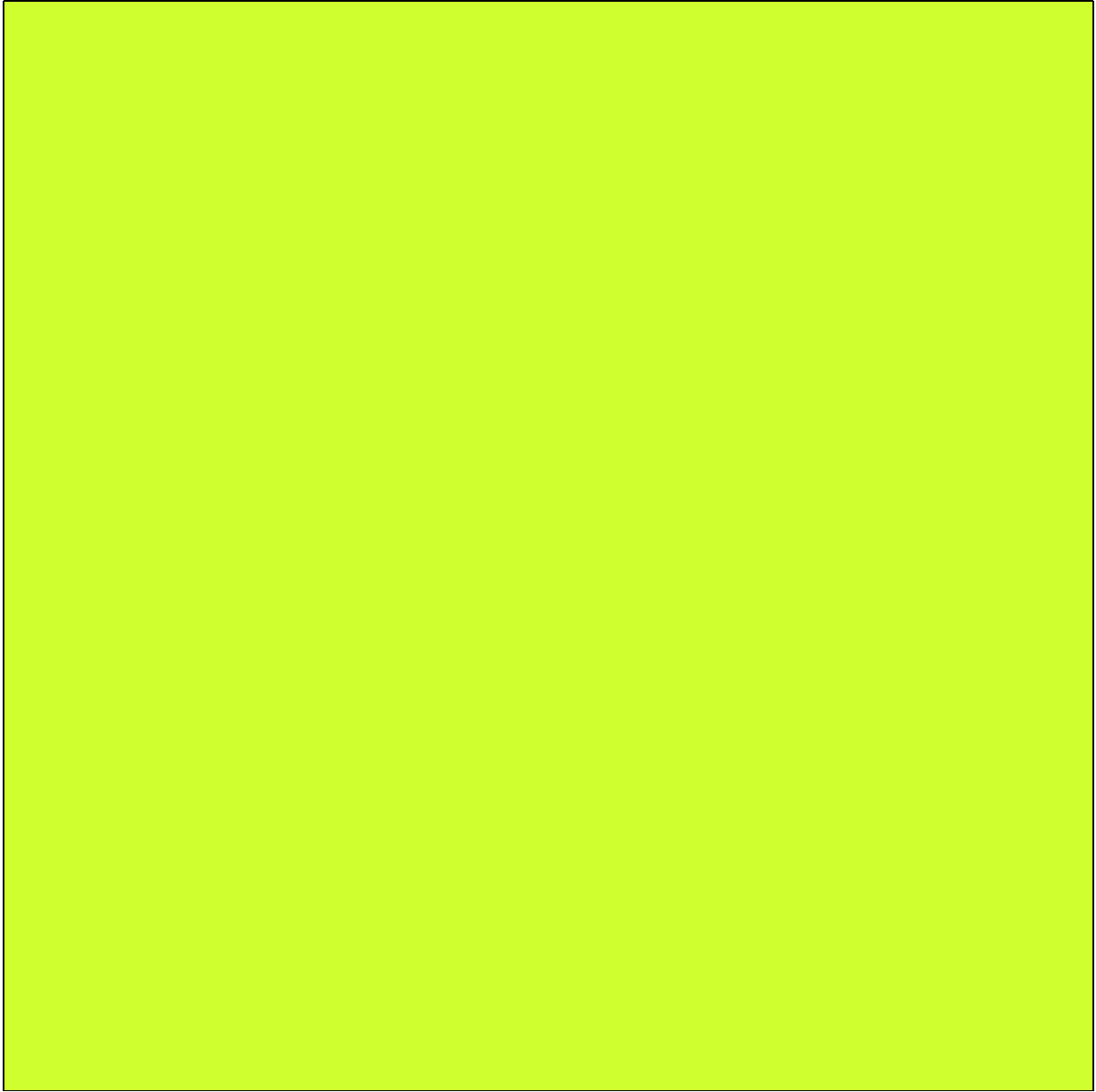} &
\includegraphics[width=1\linewidth]{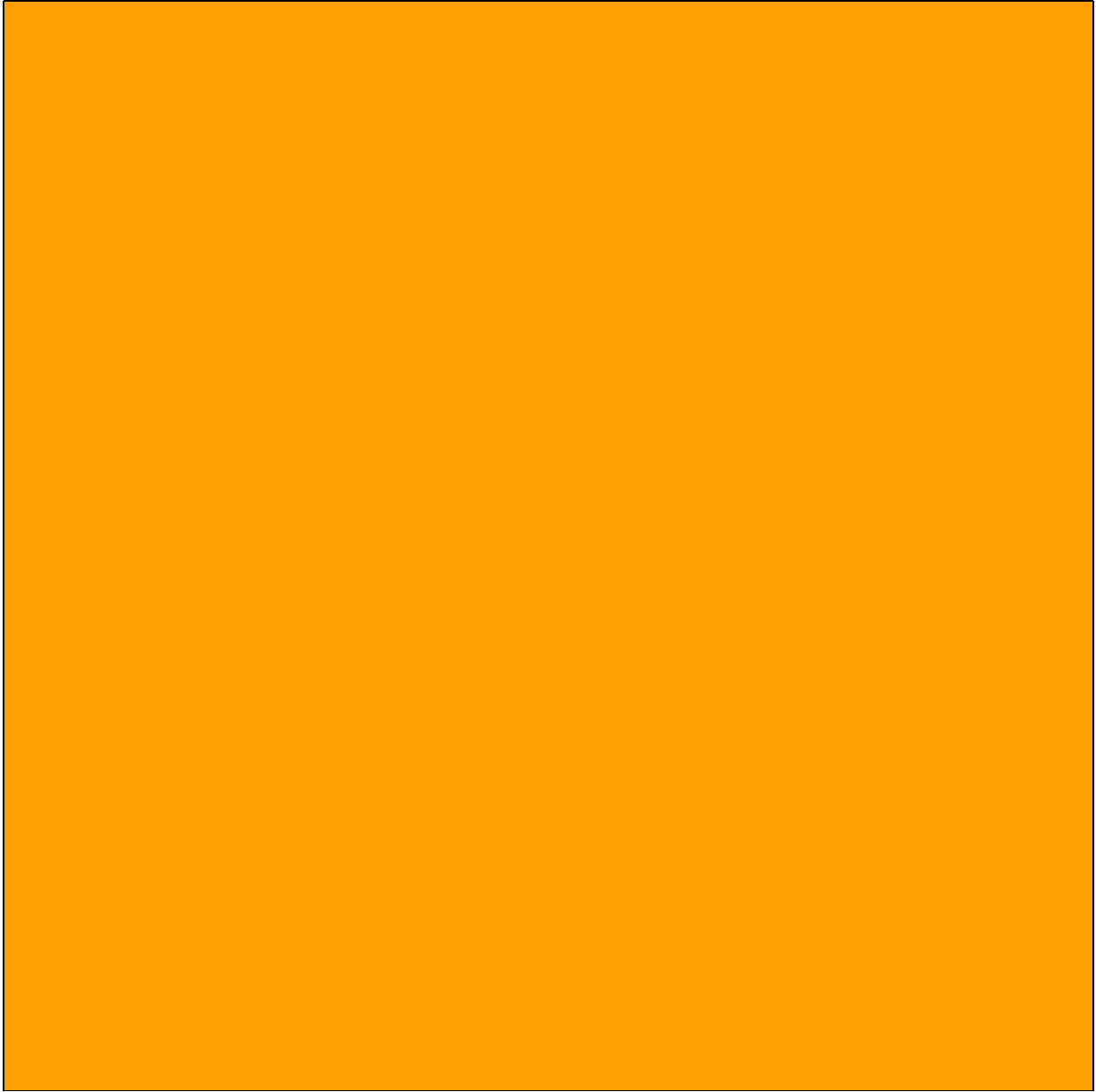} &
\includegraphics[width=1\linewidth]{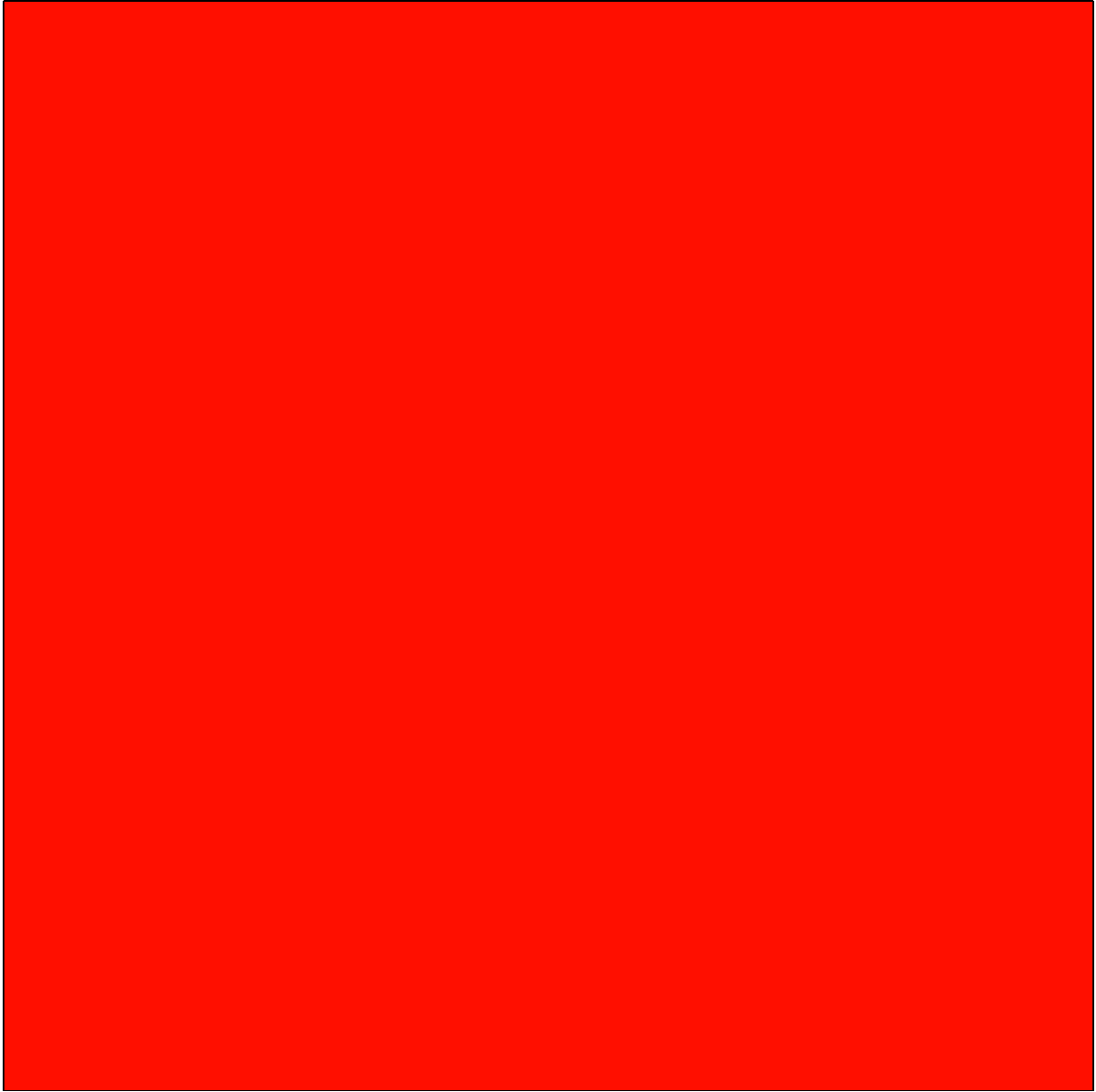} &
\includegraphics[width=1\linewidth]{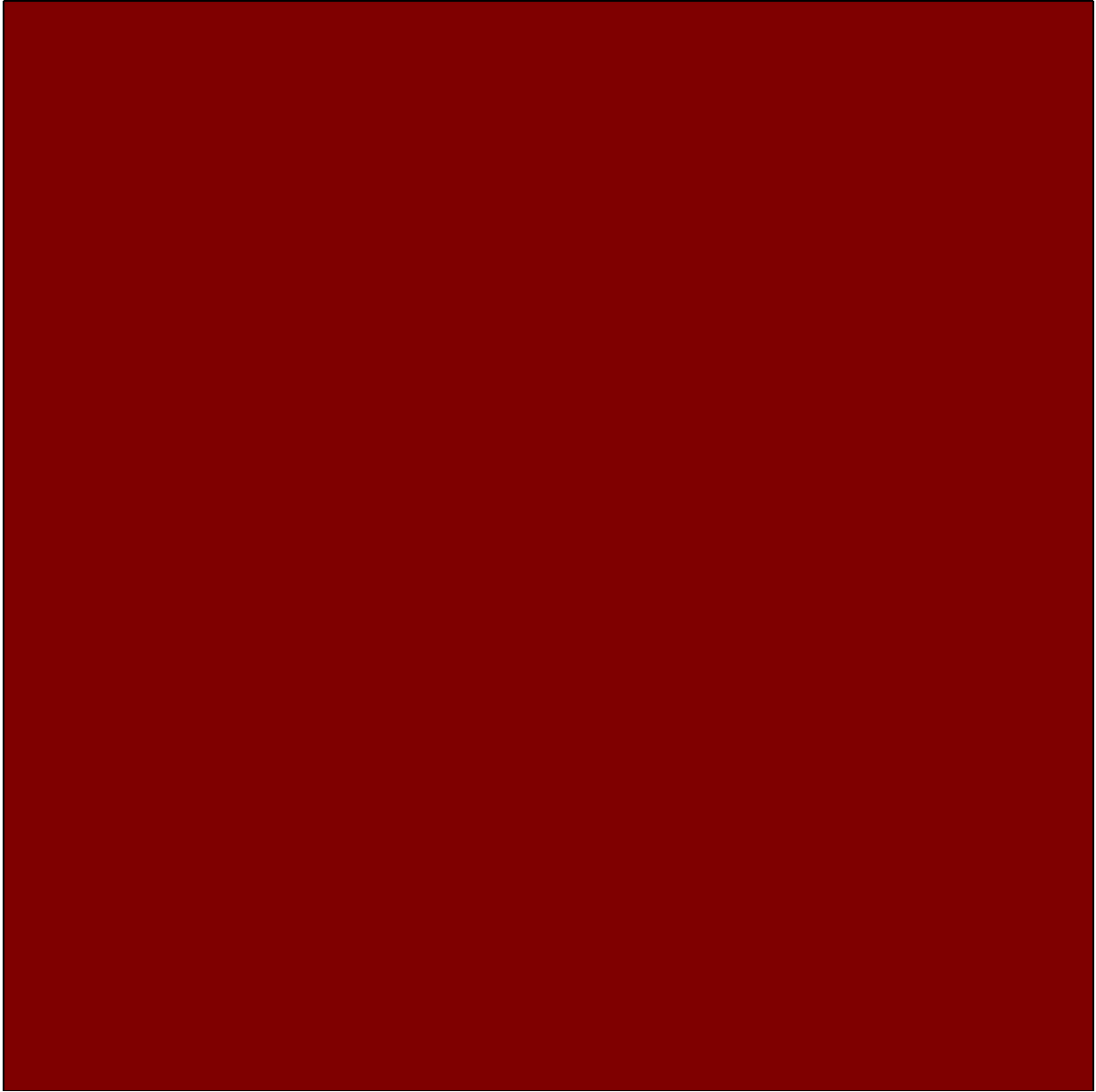} \\
\raggedright
$g=-1$ $h=-1$ $s=-1$ & 
\raggedright
$g=+1$ $h=-1$ $s=-1$& 
\raggedright
$g=-1$ $h=+1$ $s=-1$& 
\raggedright
$g=+1$ $h=+1$ $s=-1$& 
\raggedright
$g=-1$ $h=-1$ $s=+1$& 
\raggedright
$g=+1$ $h=-1$ $s=+1$& 
\raggedright
$g=-1$ $h=+1$ $s=+1$& 
\raggedright
$g=+1$ $h=+1$ $s=+1$\\
\end{tabular}
}
\caption{\label{fig:segDiff} A pixelwise comparison, in the semantic segmentation task \cite{gulshan2010geodesic}, of the ground truth (denoted $g$ in the legend), the prediction from training with Hamming loss (denoted $h$), and the prediction when training with the 8-connected loss (denoted $s$). We note that there are many regions in the set of images where the supermodular loss learns to correctly predict the foreground when Hamming loss fails (orange regions corresponding to $g=+1$, $h=-1$, and $s=+1$). } 
\end{figure*}

\begin{figure*}[]
\centering
\subfigure[]{\includegraphics[width=0.23\linewidth]{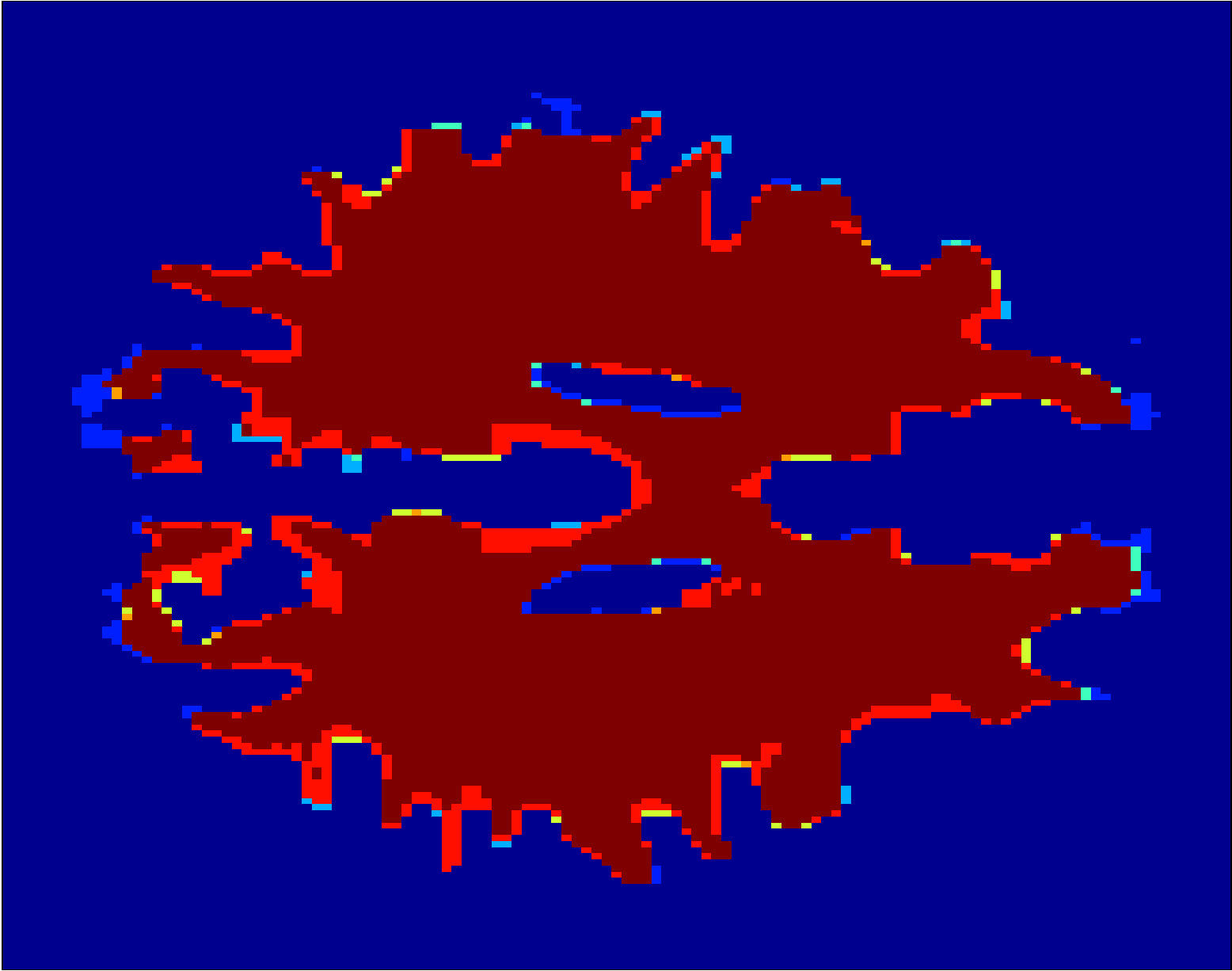}\label{fig:SegmentationComparison10}}
\subfigure[]{\includegraphics[width=0.23\linewidth]{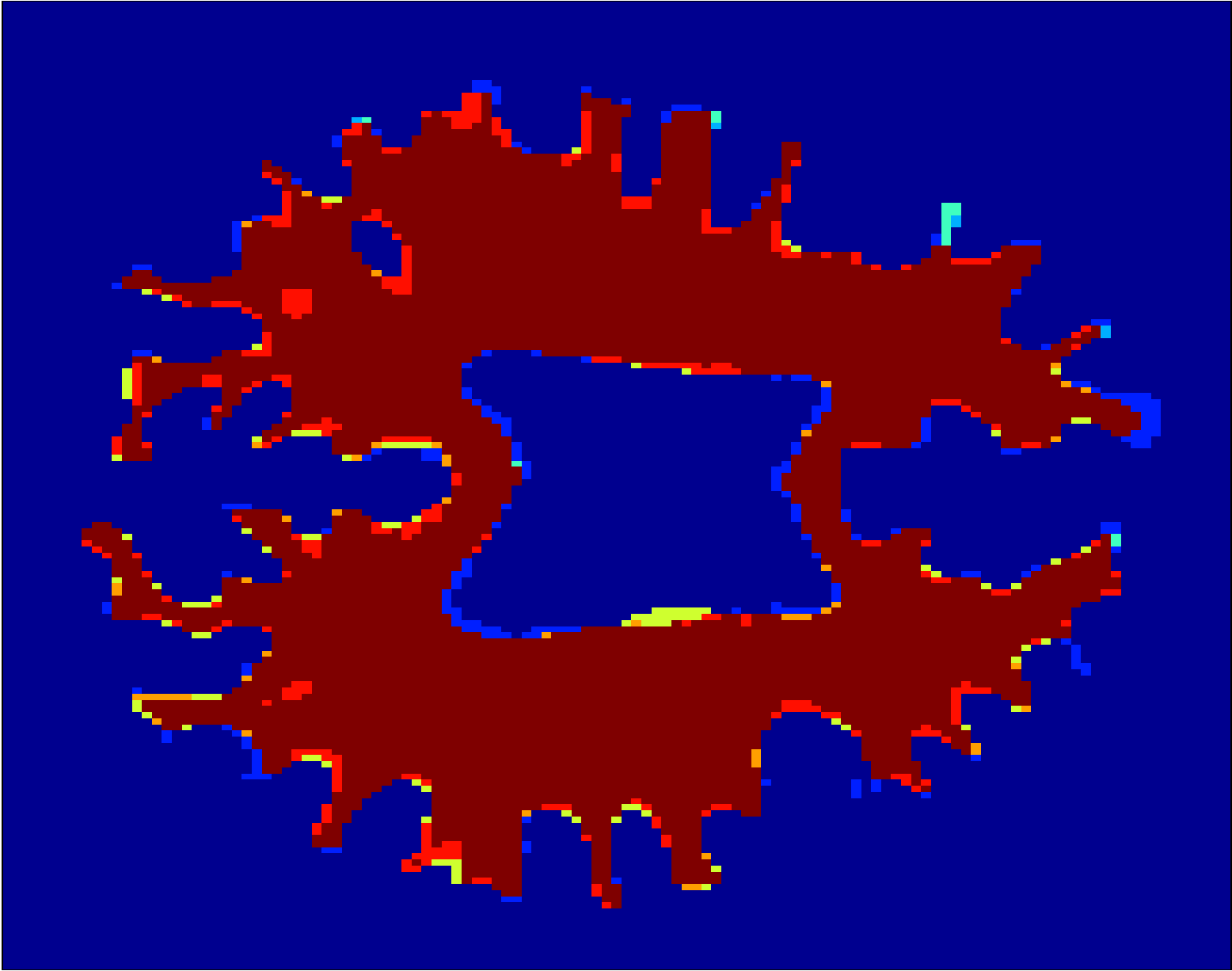}\label{fig:SegmentationComparison11}}
\subfigure[]{\includegraphics[width=0.23\linewidth]{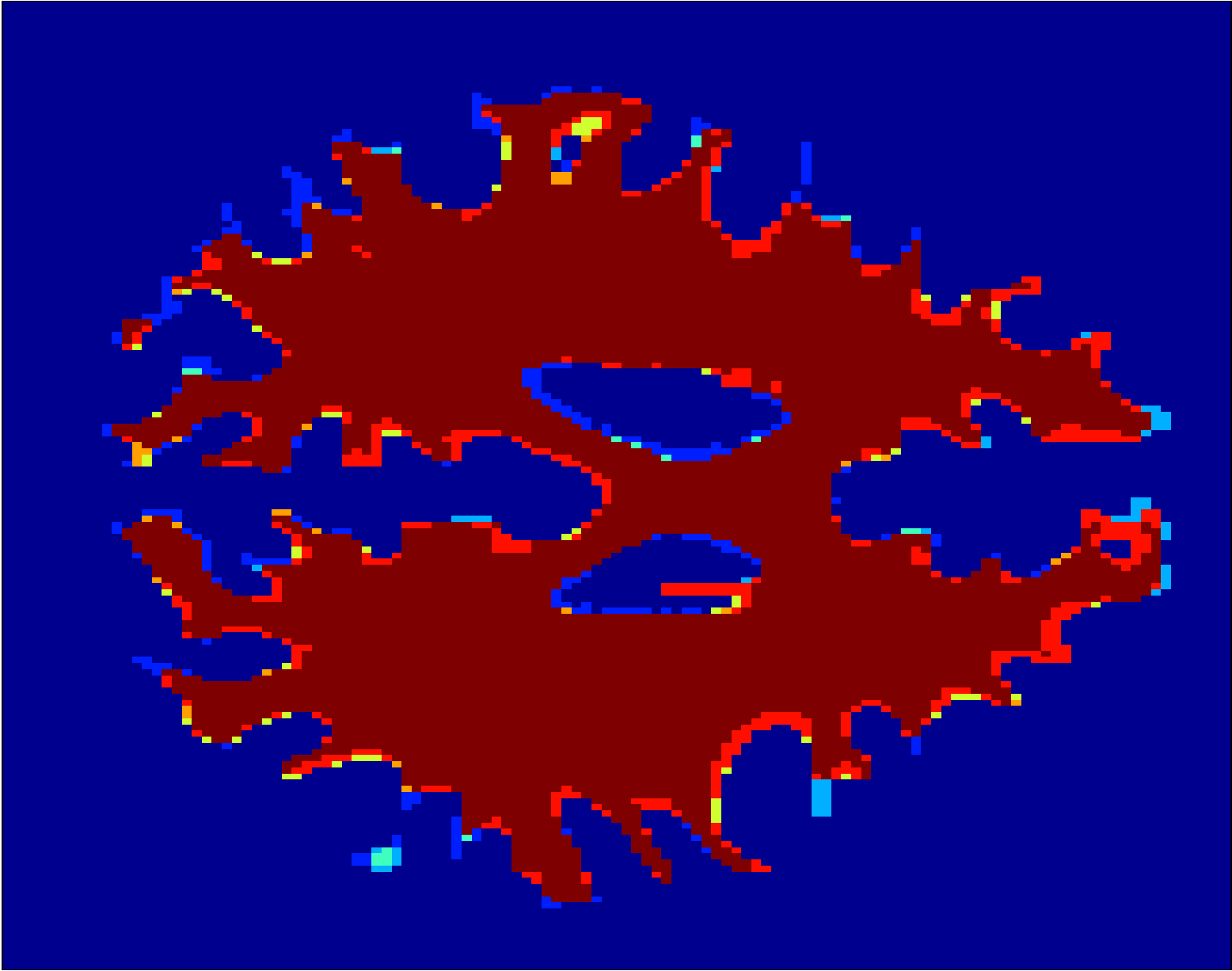}\label{fig:SegmentationComparison12}}
\subfigure[]{\includegraphics[width=0.23\linewidth]{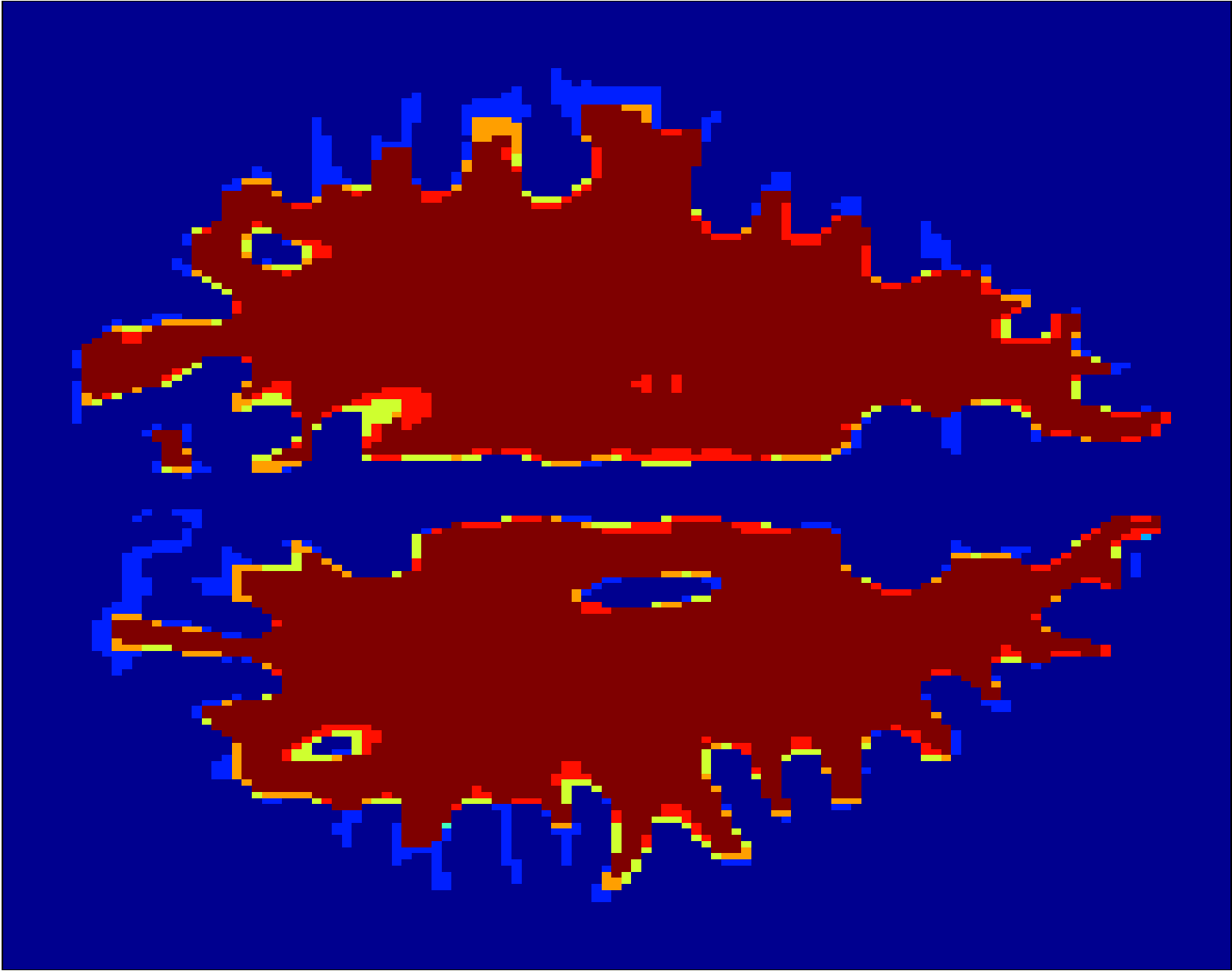}\label{fig:SegmentationComparison13}}
{\tiny
\begin{tabular}{p{0.05\linewidth}p{0.05\linewidth}p{0.05\linewidth}p{0.05\linewidth}p{0.05\linewidth}p{0.05\linewidth}p{0.05\linewidth}p{0.05\linewidth}}
\includegraphics[width=1\linewidth]{SegmentationComparisonLegend0-crop.pdf} &
\includegraphics[width=1\linewidth]{SegmentationComparisonLegend1-crop.pdf} &
\includegraphics[width=1\linewidth]{SegmentationComparisonLegend2-crop.pdf} &
\includegraphics[width=1\linewidth]{SegmentationComparisonLegend3-crop.pdf} &
\includegraphics[width=1\linewidth]{SegmentationComparisonLegend4-crop.pdf} &
\includegraphics[width=1\linewidth]{SegmentationComparisonLegend5-crop.pdf} &
\includegraphics[width=1\linewidth]{SegmentationComparisonLegend6-crop.pdf} &
\includegraphics[width=1\linewidth]{SegmentationComparisonLegend7-crop.pdf} \\
\raggedright
$g=-1$ $h=-1$ $s=-1$ & 
\raggedright
$g=+1$ $h=-1$ $s=-1$& 
\raggedright
$g=-1$ $h=+1$ $s=-1$& 
\raggedright
$g=+1$ $h=+1$ $s=-1$& 
\raggedright
$g=-1$ $h=-1$ $s=+1$& 
\raggedright
$g=+1$ $h=-1$ $s=+1$& 
\raggedright
$g=-1$ $h=+1$ $s=+1$& 
\raggedright
$g=+1$ $h=+1$ $s=+1$\\
\end{tabular}
}
\caption{\label{fig:segDiff2} A pixelwise comparison, in the structural brain segmentation task \cite{braindata}, of the ground truth (denoted $g$ in the legend), the prediction from training with Hamming loss (denoted $h$), and the prediction when training with the 8-connected loss (denoted $s$). We note that there are many regions in the set of images where the supermodular loss learns to correctly predict the foreground when Hamming loss fails (orange regions corresponding to $g=+1$, $h=-1$, and $s=+1$).}
\end{figure*}

\begin{figure*}
\centering
{\includegraphics[width=0.12\linewidth]{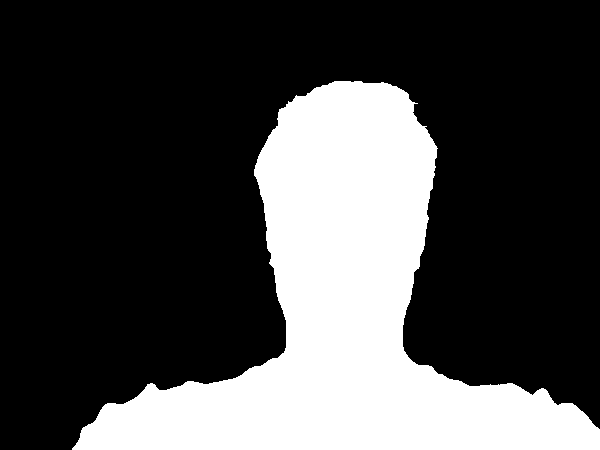}}
{\includegraphics[width=0.12\linewidth]{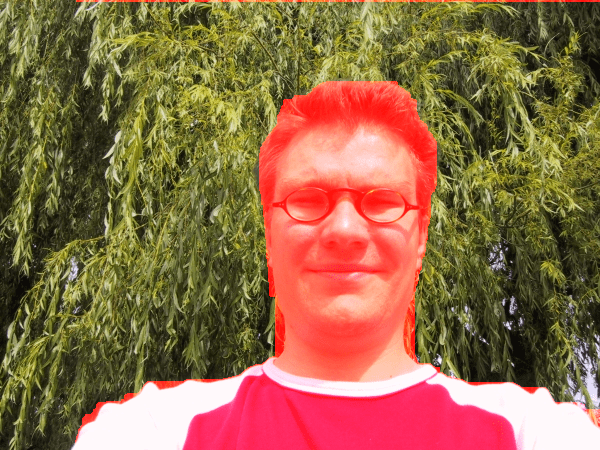}}
{\includegraphics[width=0.12\linewidth]{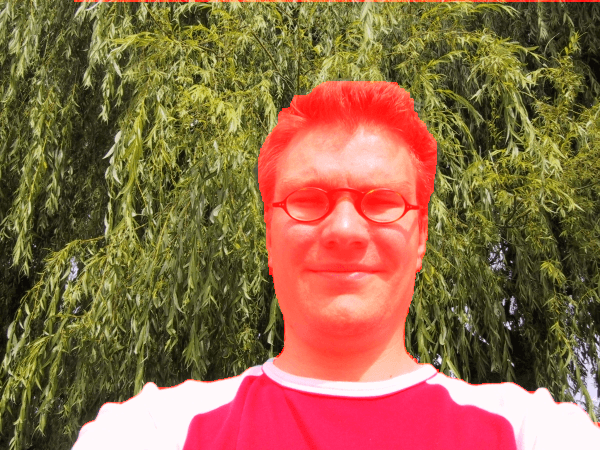}}
{\includegraphics[width=0.12\linewidth]{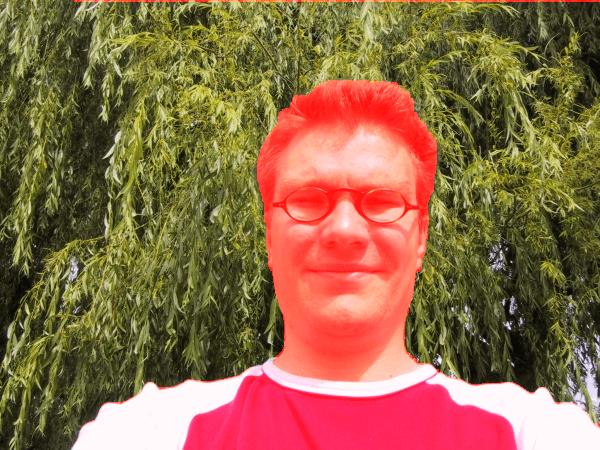}}
{\includegraphics[width=0.12\linewidth]{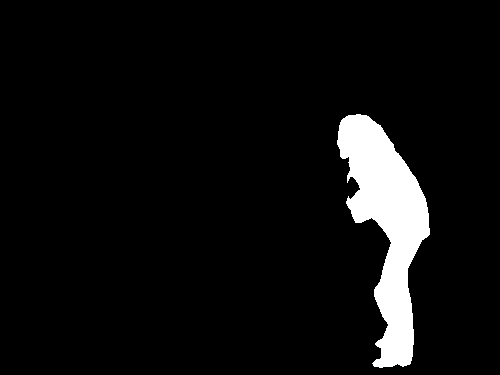}}
{\includegraphics[width=0.12\linewidth]{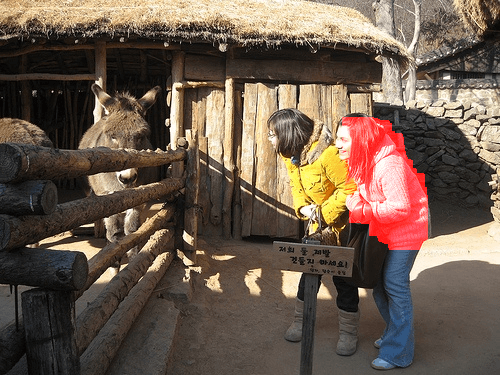}}
{\includegraphics[width=0.12\linewidth]{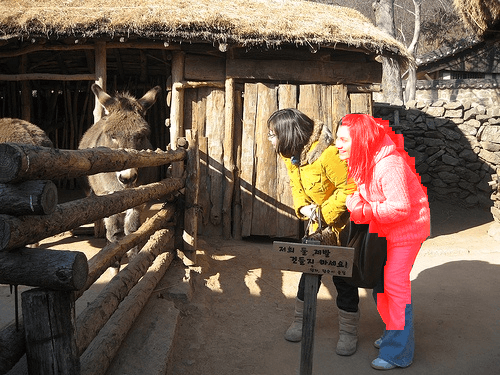}}
{\includegraphics[width=0.12\linewidth]{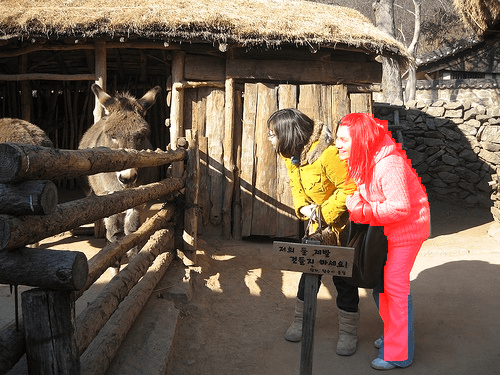}}
{\includegraphics[width=0.12\linewidth]{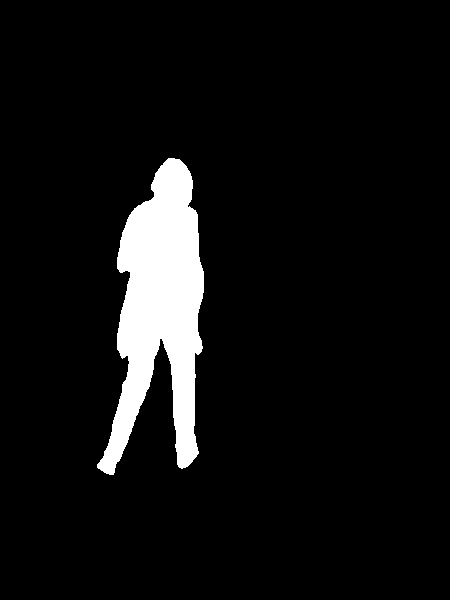}}
{\includegraphics[width=0.12\linewidth]{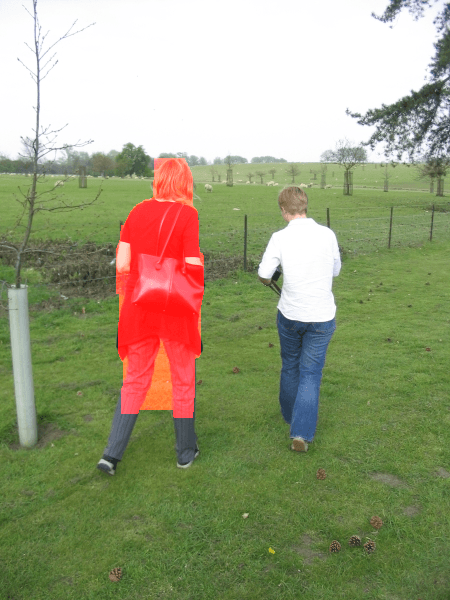}}
{\includegraphics[width=0.12\linewidth]{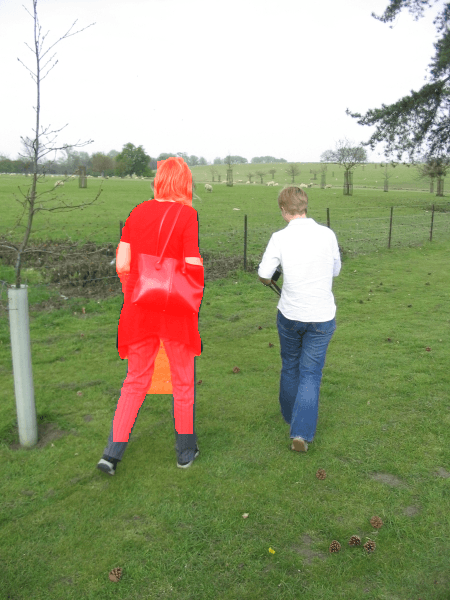}}
{\includegraphics[width=0.12\linewidth]{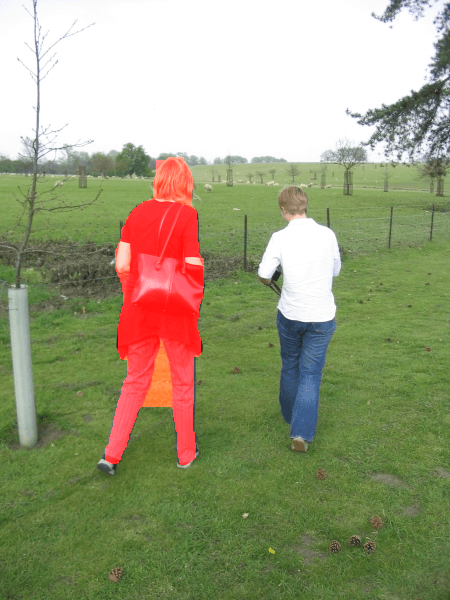}}
{\includegraphics[width=0.12\linewidth]{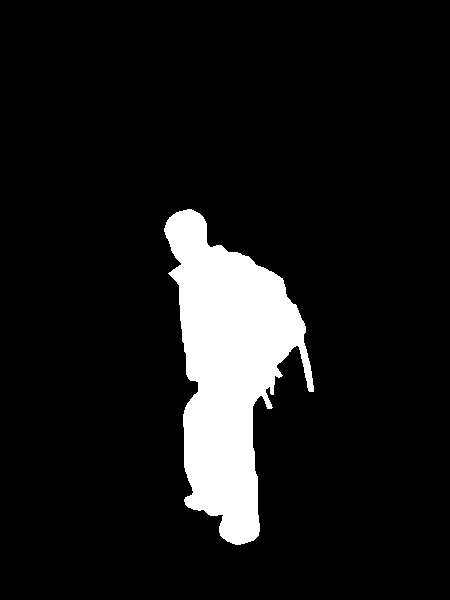}}
{\includegraphics[width=0.12\linewidth]{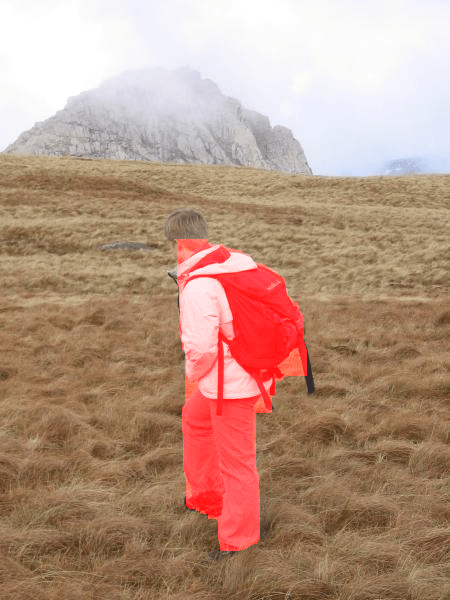}}
{\includegraphics[width=0.12\linewidth]{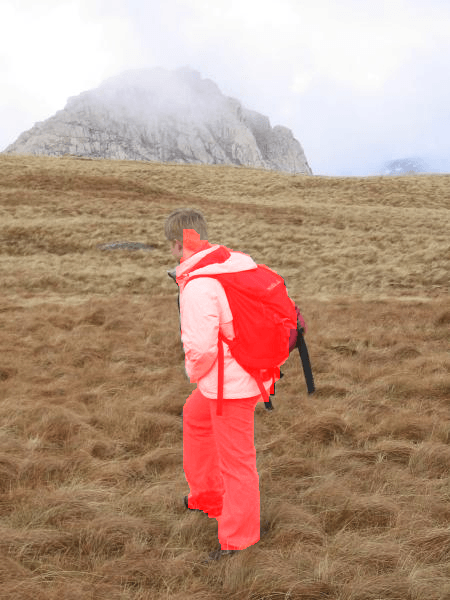}}
{\includegraphics[width=0.12\linewidth]{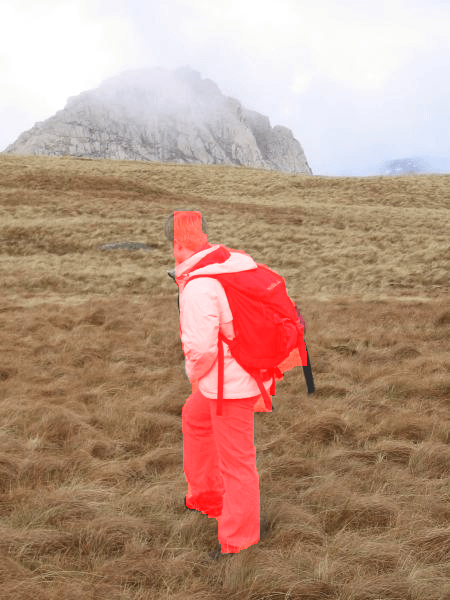}}
{\includegraphics[width=0.12\linewidth]{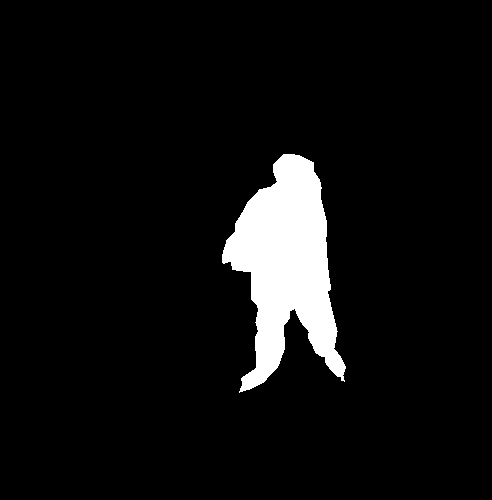}}
{\includegraphics[width=0.12\linewidth]{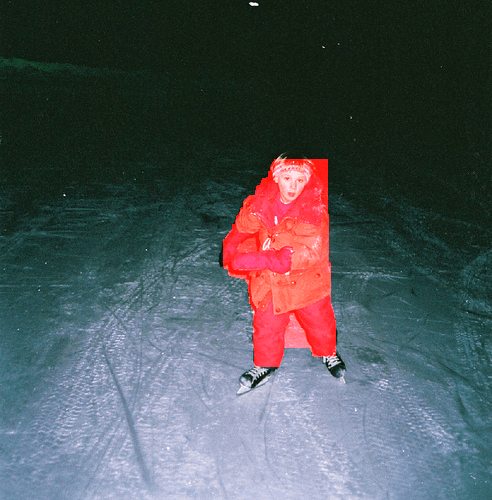}}
{\includegraphics[width=0.12\linewidth]{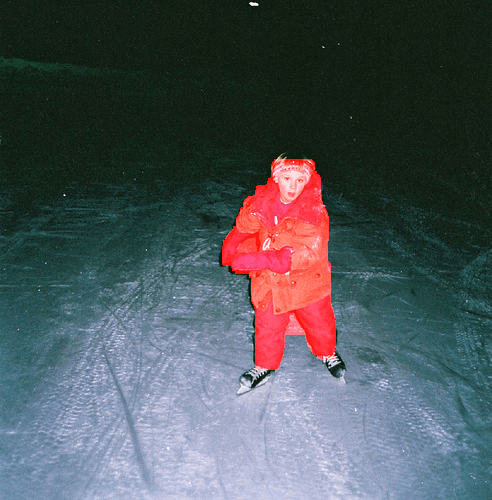}}
{\includegraphics[width=0.12\linewidth]{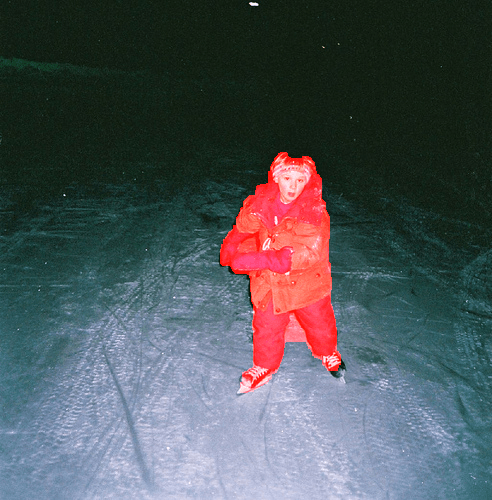}}
{\includegraphics[trim={0 2.5cm 0 2.5cm},clip,width=0.12\linewidth]{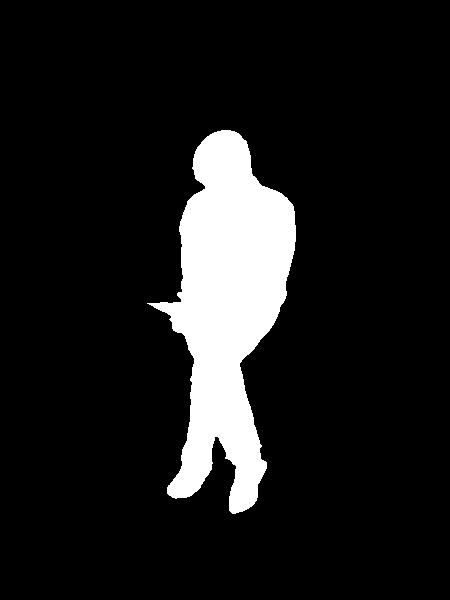}}
{\includegraphics[trim={0 2.5cm 0 2.5cm},clip,width=0.12\linewidth]{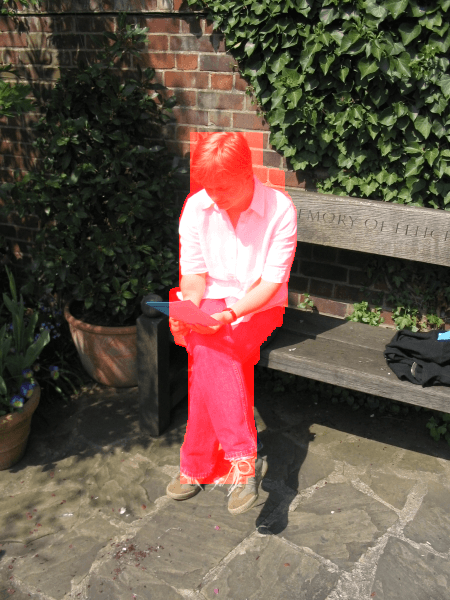}}
{\includegraphics[trim={0 2.5cm 0 2.5cm},clip,width=0.12\linewidth]{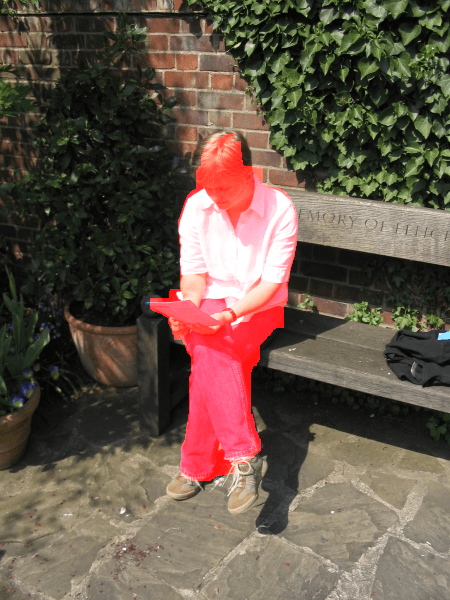}}
{\includegraphics[trim={0 2.5cm 0 2.5cm},clip,width=0.12\linewidth]{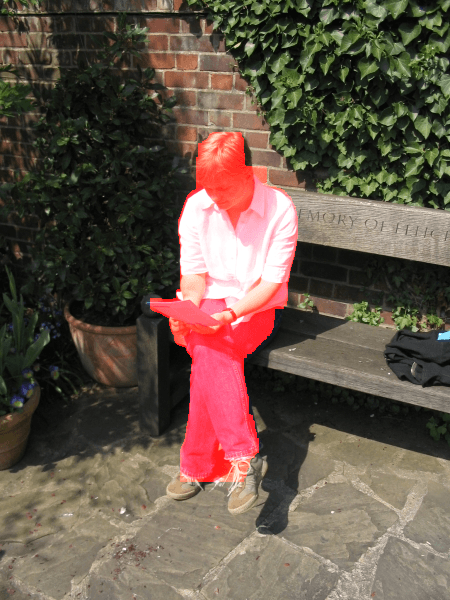}}
\subfigure[groundtruth]{\includegraphics[width=0.12\linewidth]{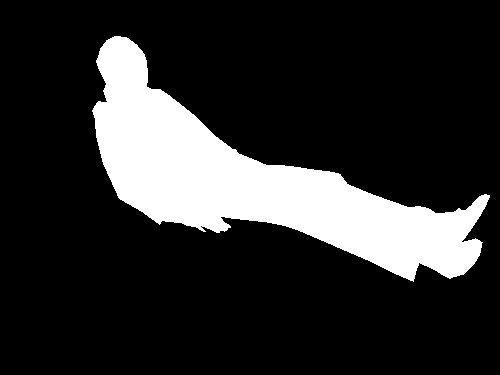}}
\subfigure[Hamming]{\includegraphics[width=0.12\linewidth]{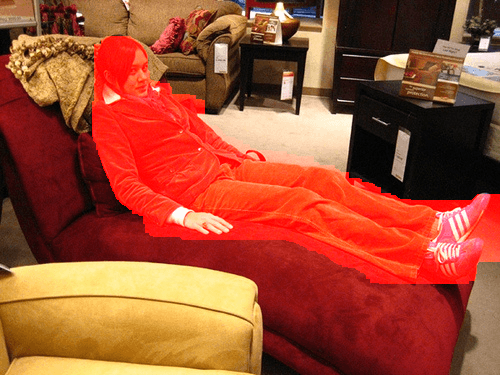}}
\subfigure[Square]{\includegraphics[width=0.12\linewidth]{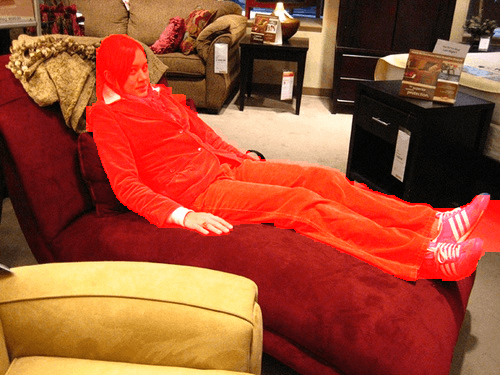}}
\subfigure[Biconvex]{\includegraphics[width=0.12\linewidth]{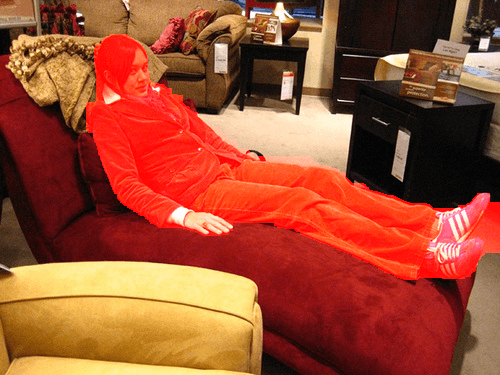}}
\subfigure[groundtruth]{\includegraphics[trim={0 2.5cm 0 0},clip,width=0.12\linewidth]{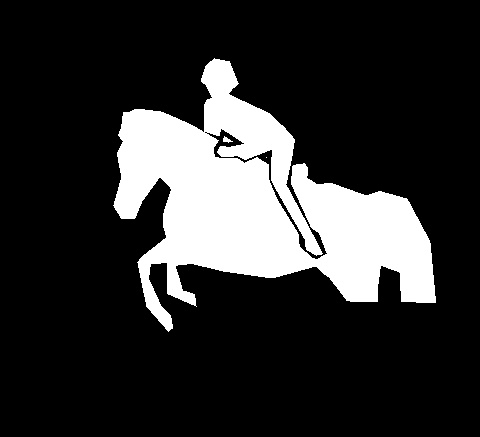}}
\subfigure[Hamming]{\includegraphics[trim={0 2.5cm 0 0},clip,width=0.12\linewidth]{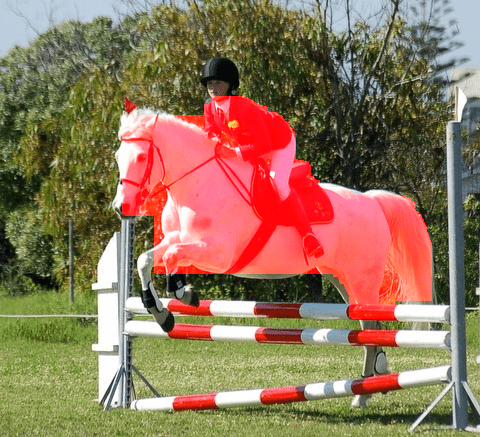}}
\subfigure[Square]{\includegraphics[trim={0 2.5cm 0 0},clip,width=0.12\linewidth]{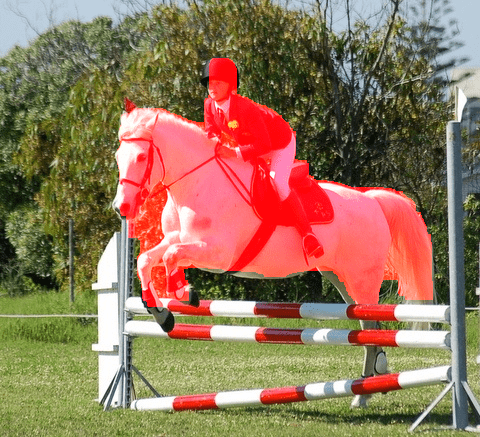}}
\subfigure[Biconvex]{\includegraphics[trim={0 2.5cm 0 0},clip,width=0.12\linewidth]{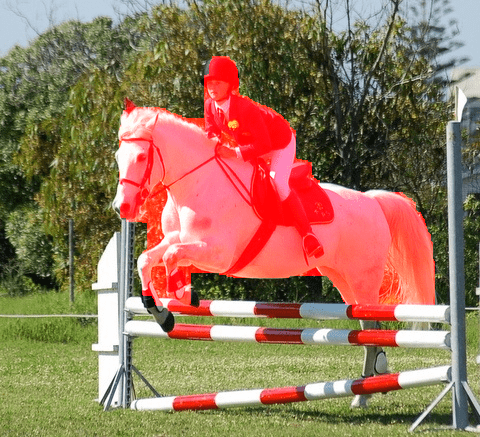}}
\caption{\label{fig:segmentation2} The segmentation results of prediction trained with Hamming loss (columns 2 and 6), the square loss (columns 3 and 7) and the biconvex loss (columns 4 and 8). The supermodular loss functions perform better on foreground object boundary than Hamming loss does, as well as they achieve better prediction on the elongated structure of the foreground object e.g.the heads and the legs.}
\end{figure*}

\begin{figure*}[]
\centering
\subfigure[]{\includegraphics[width=0.225\linewidth]{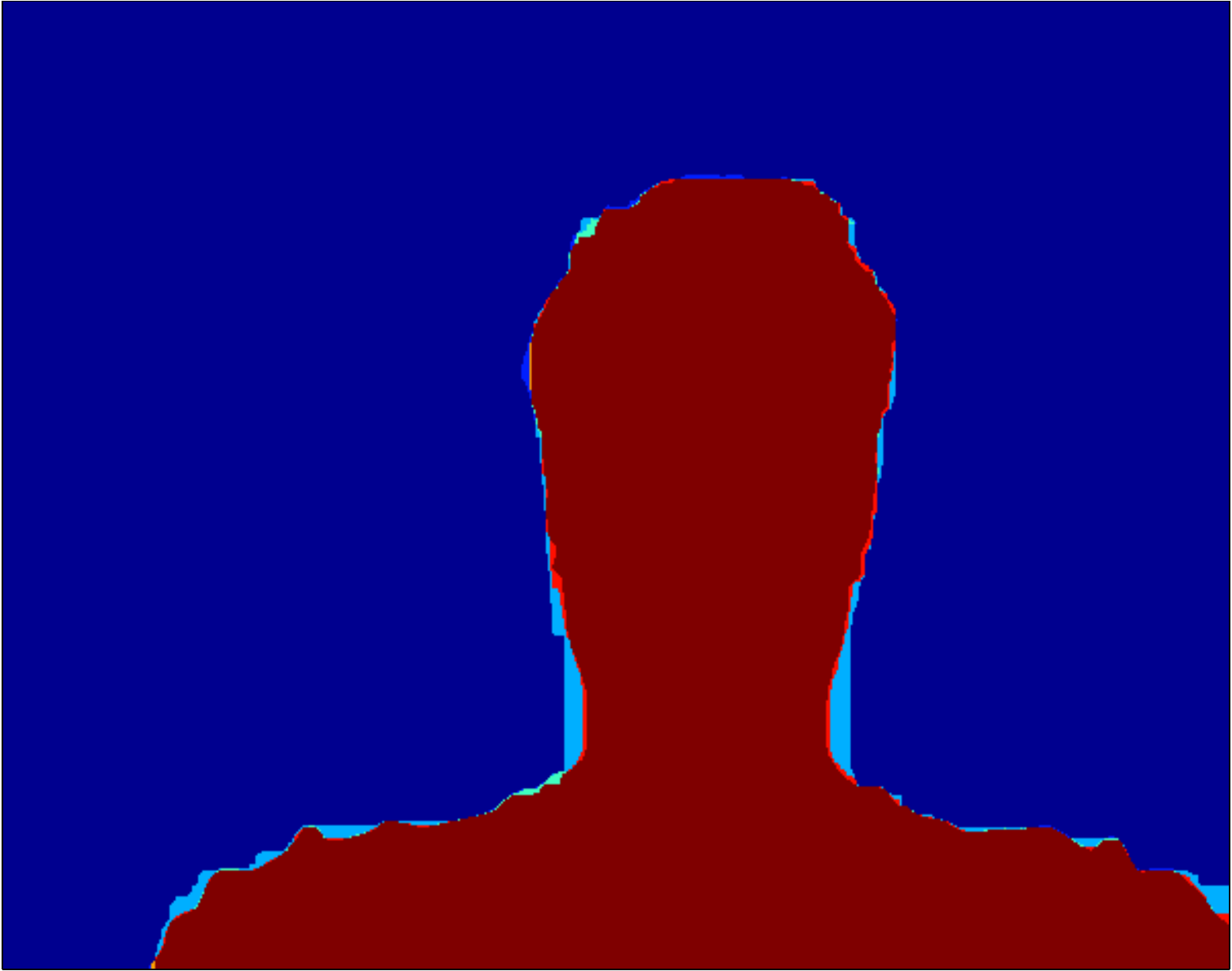}}
\subfigure[]{\includegraphics[width=0.225\linewidth]{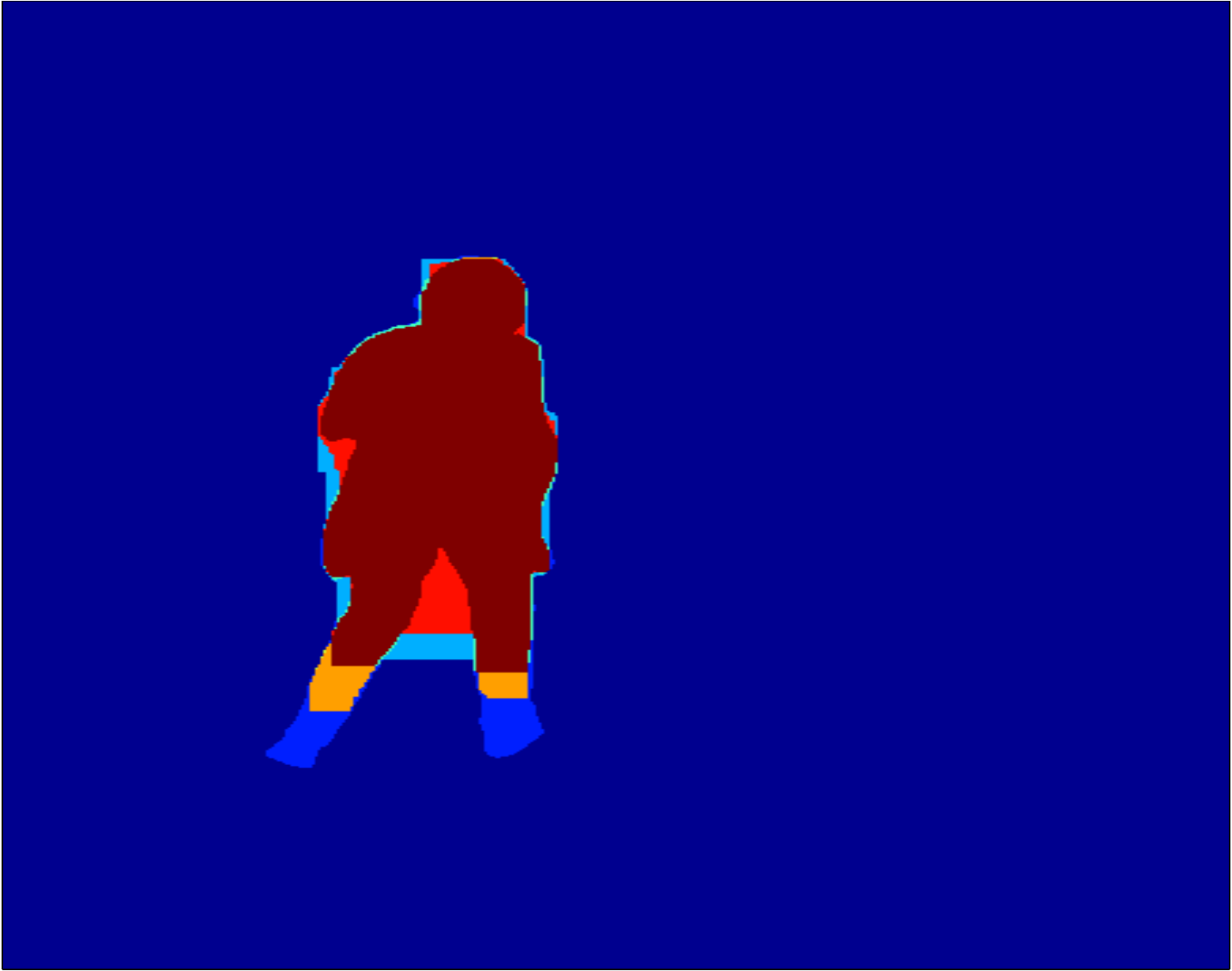}}
\subfigure[]{\includegraphics[width=0.225\linewidth]{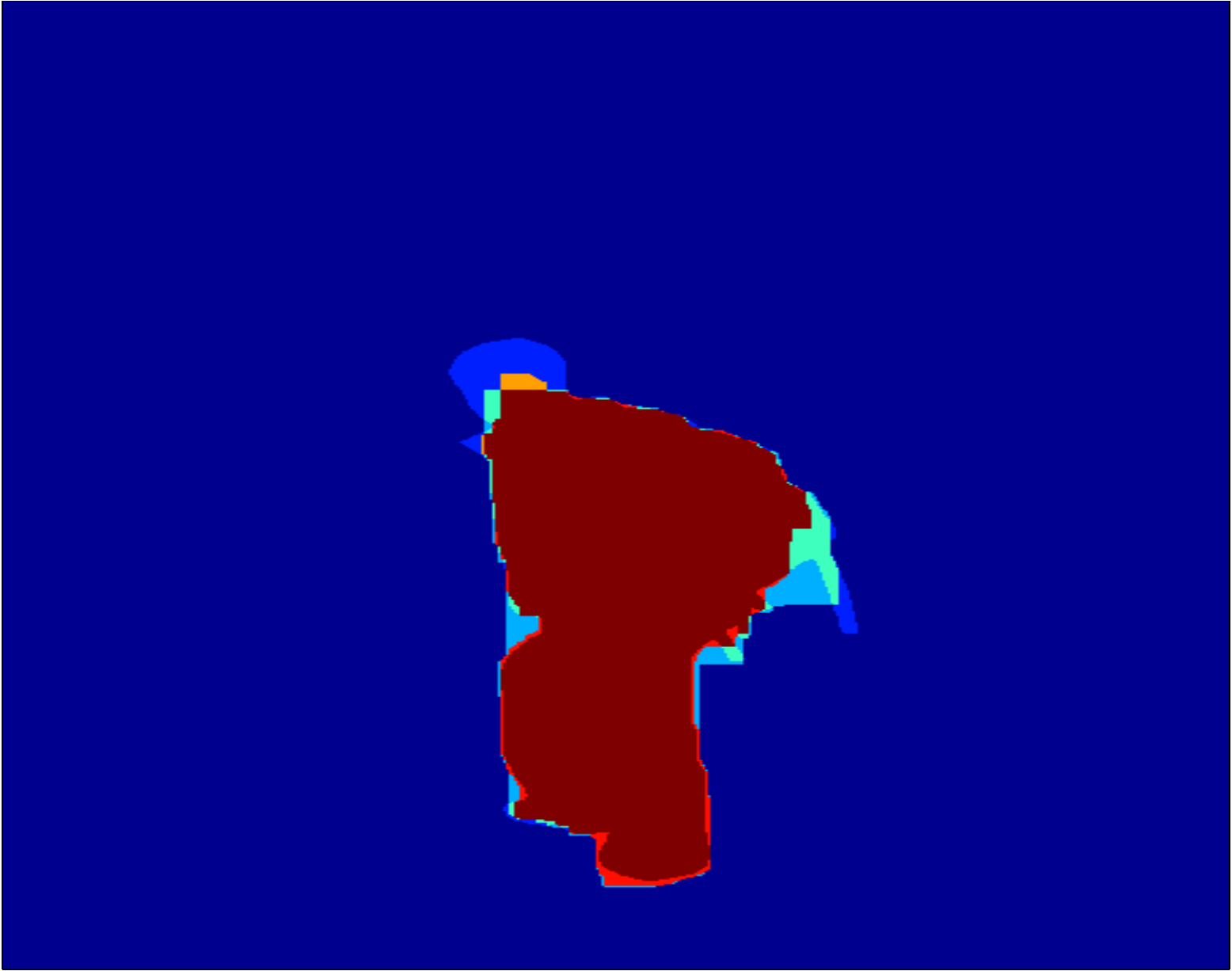}}
\subfigure[]{\includegraphics[width=0.225\linewidth]{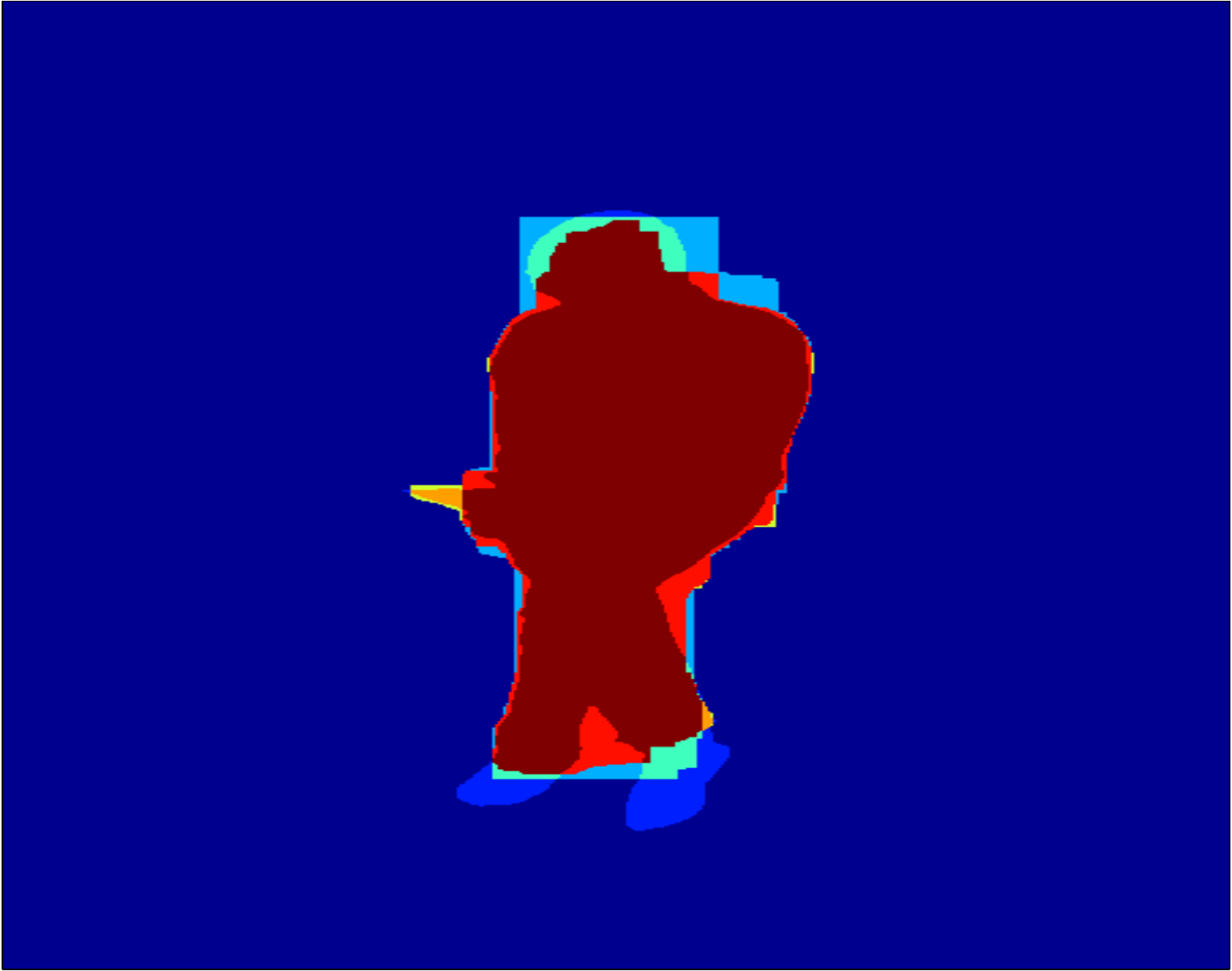}}
\subfigure[]{\includegraphics[width=0.225\linewidth]{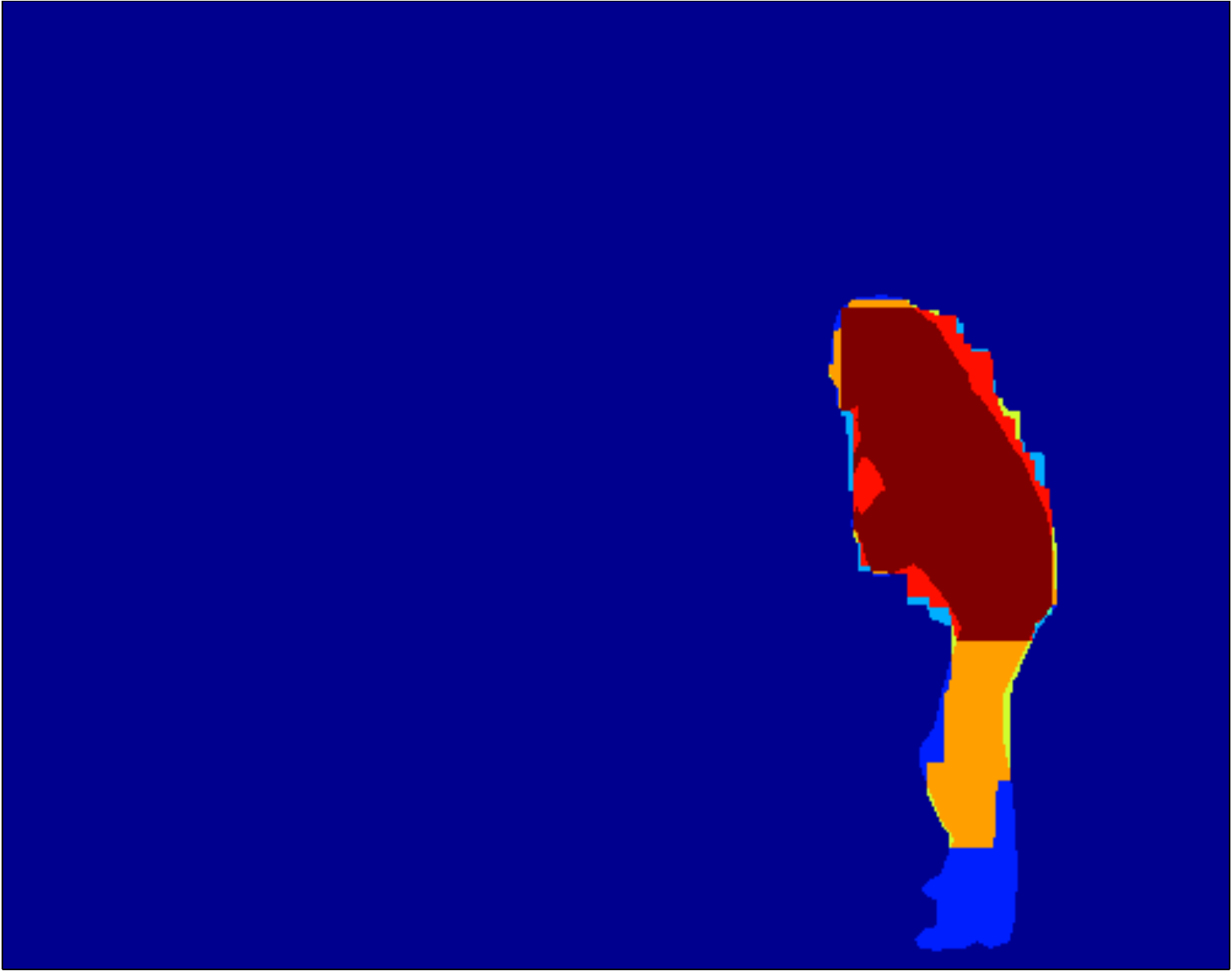}}
\subfigure[]{\includegraphics[width=0.225\linewidth]{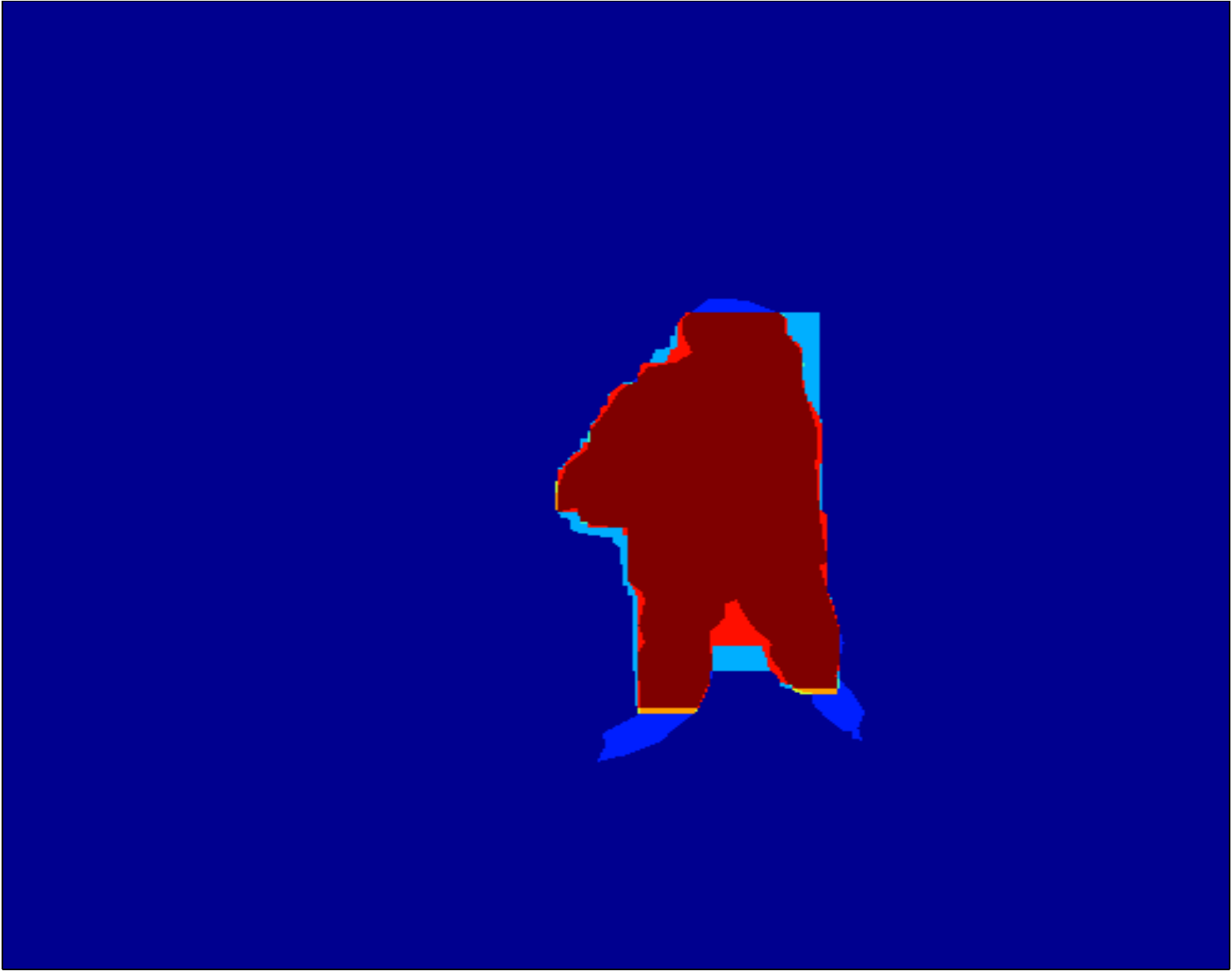}}
\subfigure[]{\includegraphics[width=0.225\linewidth]{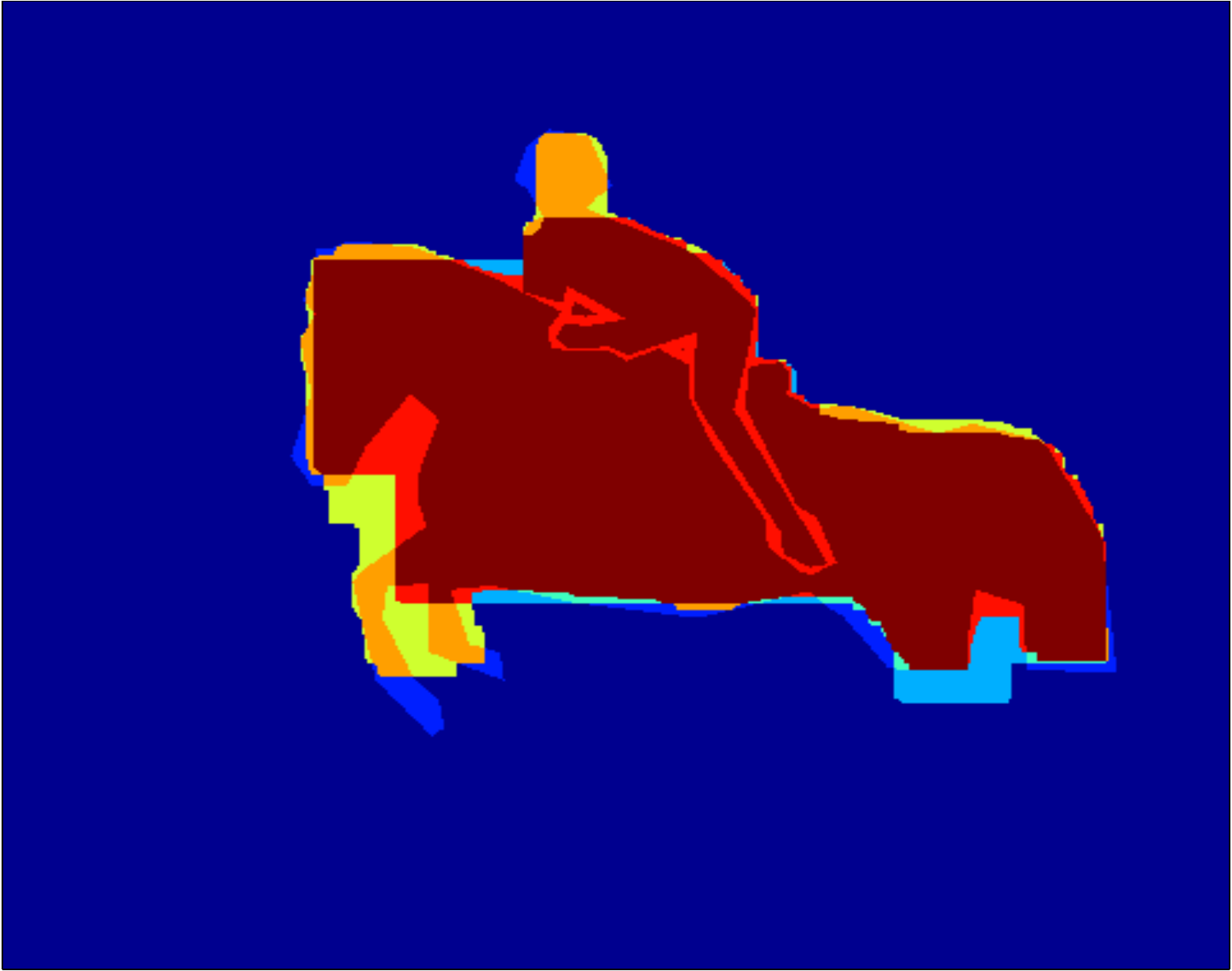}}
\subfigure[]{\includegraphics[width=0.225\linewidth]{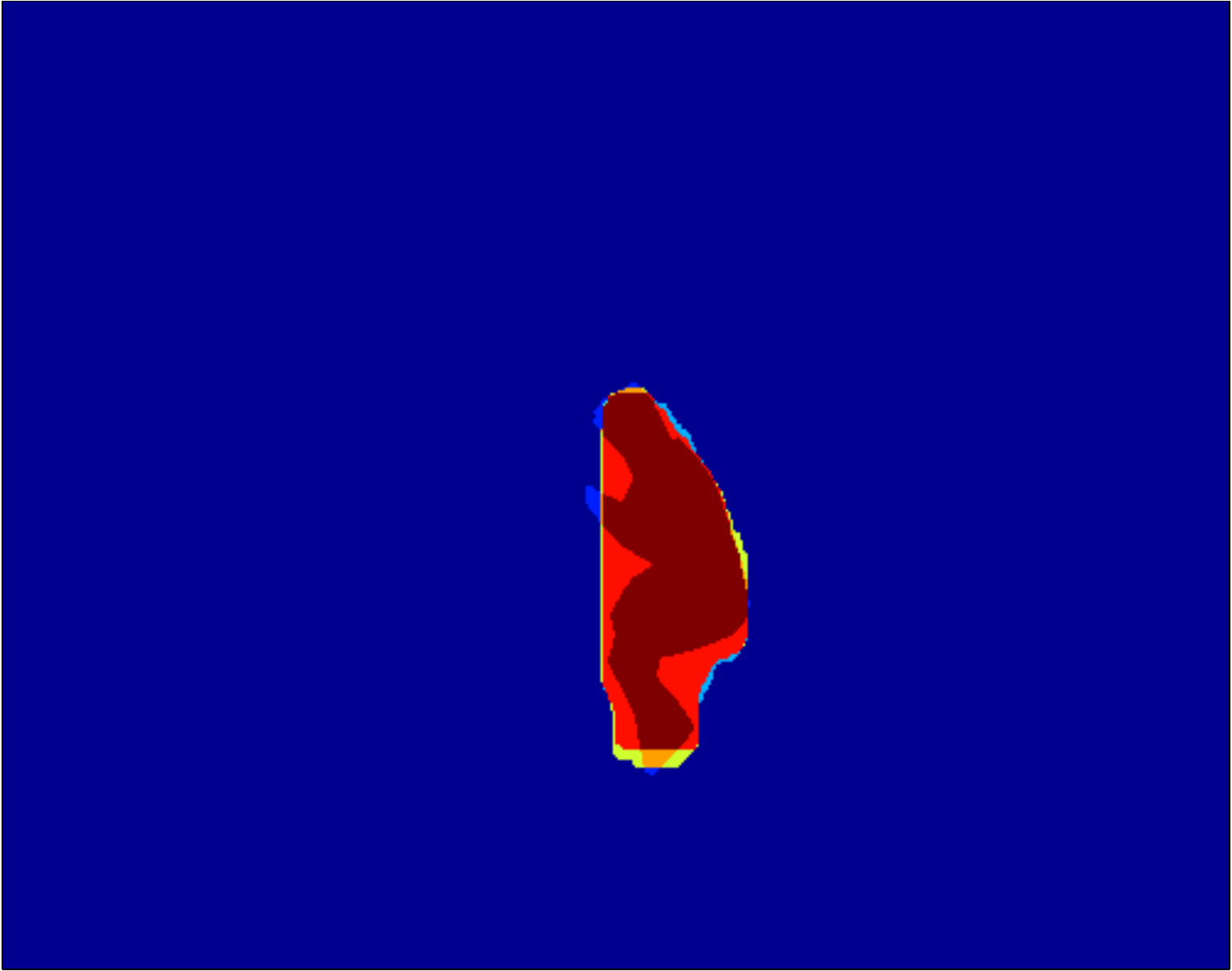}}
{\tiny
\begin{tabular}{p{0.05\linewidth}p{0.05\linewidth}p{0.05\linewidth}p{0.05\linewidth}p{0.05\linewidth}p{0.05\linewidth}p{0.05\linewidth}p{0.05\linewidth}}
\includegraphics[width=1\linewidth]{SegmentationComparisonLegend0-crop.pdf} &
\includegraphics[width=1\linewidth]{SegmentationComparisonLegend1-crop.pdf} &
\includegraphics[width=1\linewidth]{SegmentationComparisonLegend2-crop.pdf} &
\includegraphics[width=1\linewidth]{SegmentationComparisonLegend3-crop.pdf} &
\includegraphics[width=1\linewidth]{SegmentationComparisonLegend4-crop.pdf} &
\includegraphics[width=1\linewidth]{SegmentationComparisonLegend5-crop.pdf} &
\includegraphics[width=1\linewidth]{SegmentationComparisonLegend6-crop.pdf} &
\includegraphics[width=1\linewidth]{SegmentationComparisonLegend7-crop} \\
\raggedright
$g=-1$ $h=-1$ $s=-1$ & 
\raggedright
$g=+1$ $h=-1$ $s=-1$& 
\raggedright
$g=-1$ $h=+1$ $s=-1$& 
\raggedright
$g=+1$ $h=+1$ $s=-1$& 
\raggedright
$g=-1$ $h=-1$ $s=+1$& 
\raggedright
$g=+1$ $h=-1$ $s=+1$& 
\raggedright
$g=-1$ $h=+1$ $s=+1$& 
\raggedright
$g=+1$ $h=+1$ $s=+1$\\
\end{tabular}
}
\caption{\label{fig:segDiffSquare} A pixelwise comparison of the ground truth (denoted $g$ in the legend), the prediction from training with Hamming loss (denoted $h$), and the prediction when training with the square loss in Equation~\eqref{eq:squareloss} (denoted $s$). The {\color{orange} orange} regions ($g=+1$, $h=-1$, and $s=+1$) show where the supermodular loss learns to correctly predict the foreground when Hamming loss fails; the {\color{cyan} cyan} regions ($g=-1$, $h=+1$, and $s=-1$) show where the supermodular loss learns to correctly predict the background when Hamming loss fails.
 } 
\end{figure*}

\begin{figure*}[]
\centering
\subfigure[]{\includegraphics[width=0.225\linewidth]{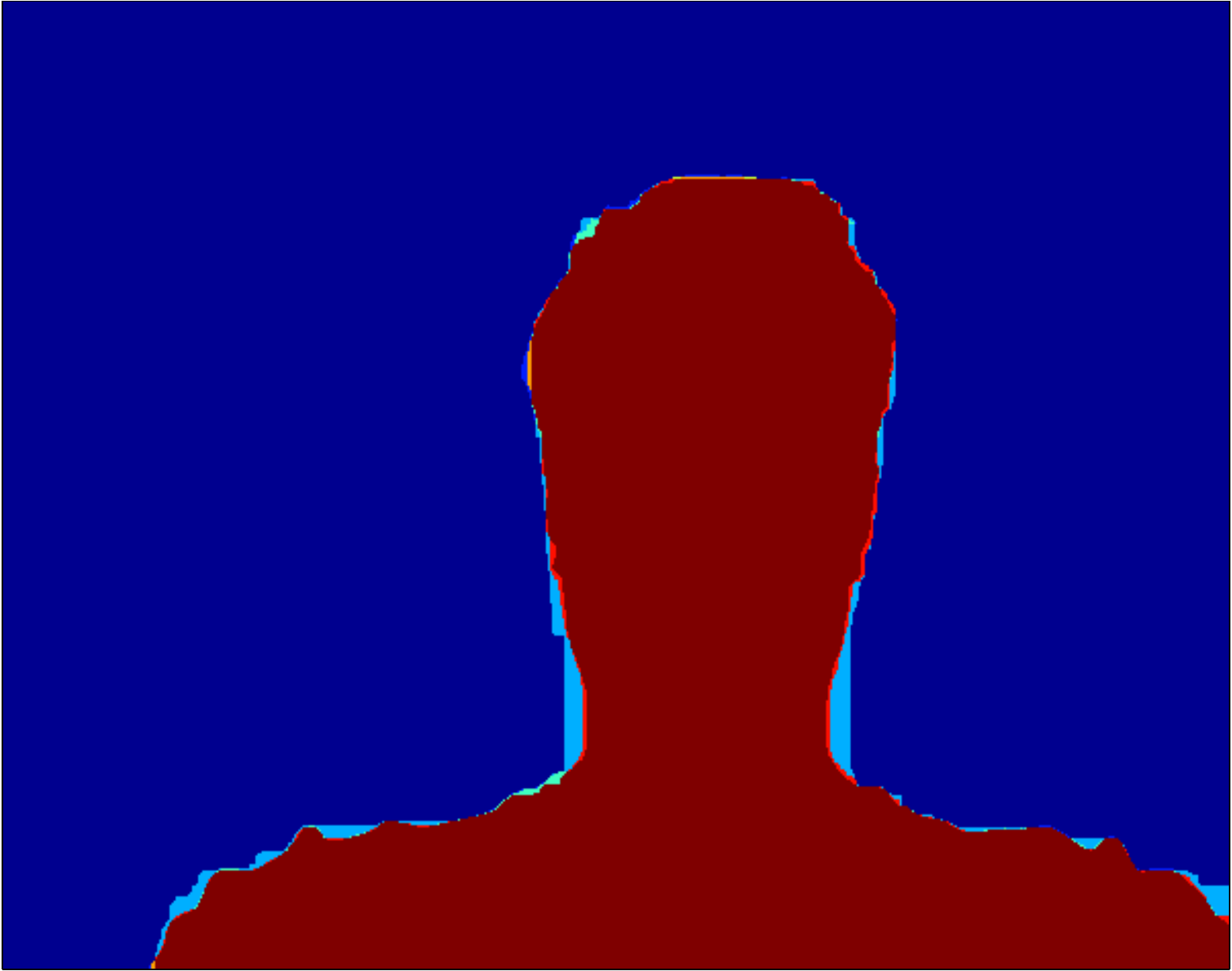}}
\subfigure[]{\includegraphics[width=0.225\linewidth]{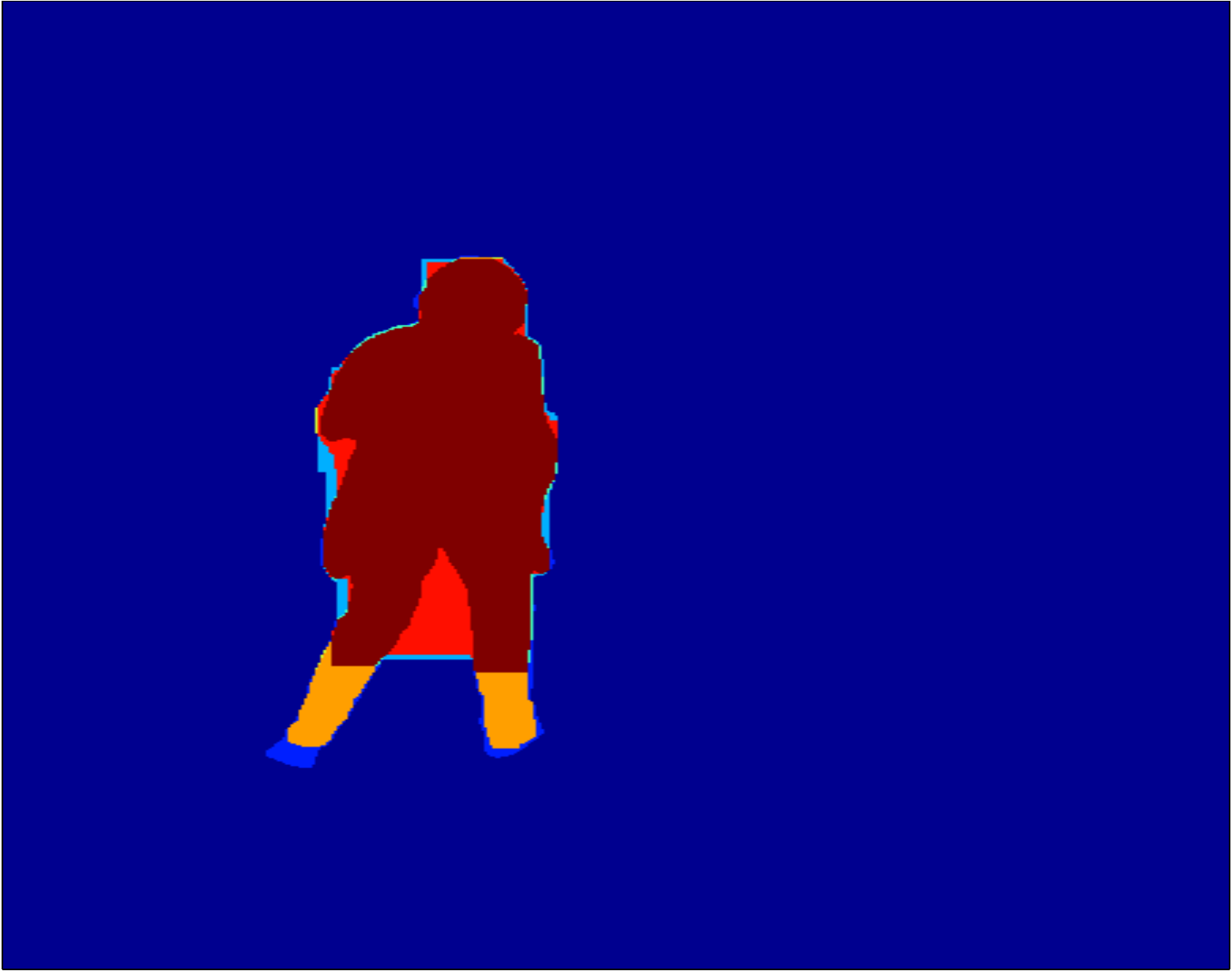}}
\subfigure[]{\includegraphics[width=0.225\linewidth]{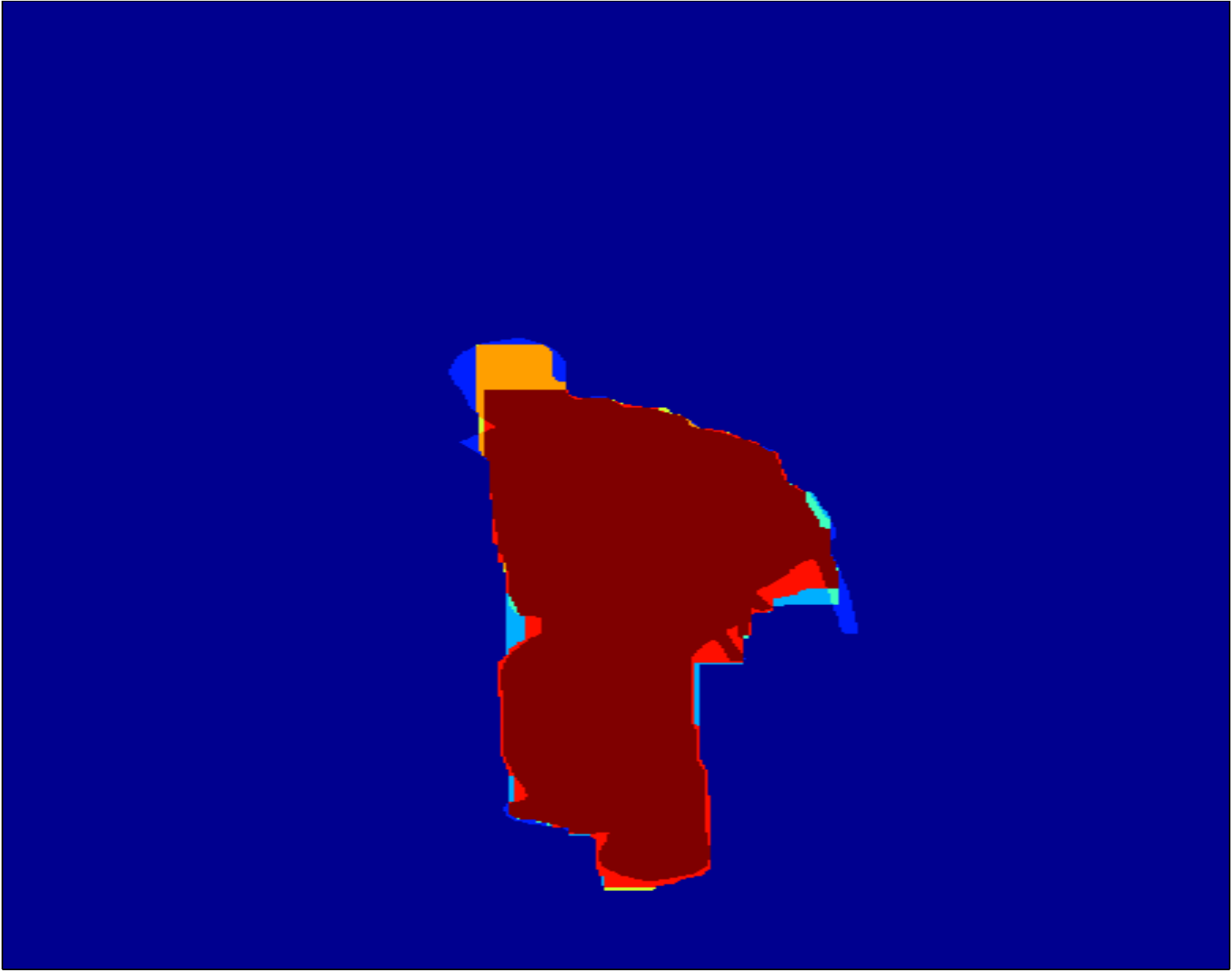}}
\subfigure[]{\includegraphics[width=0.225\linewidth]{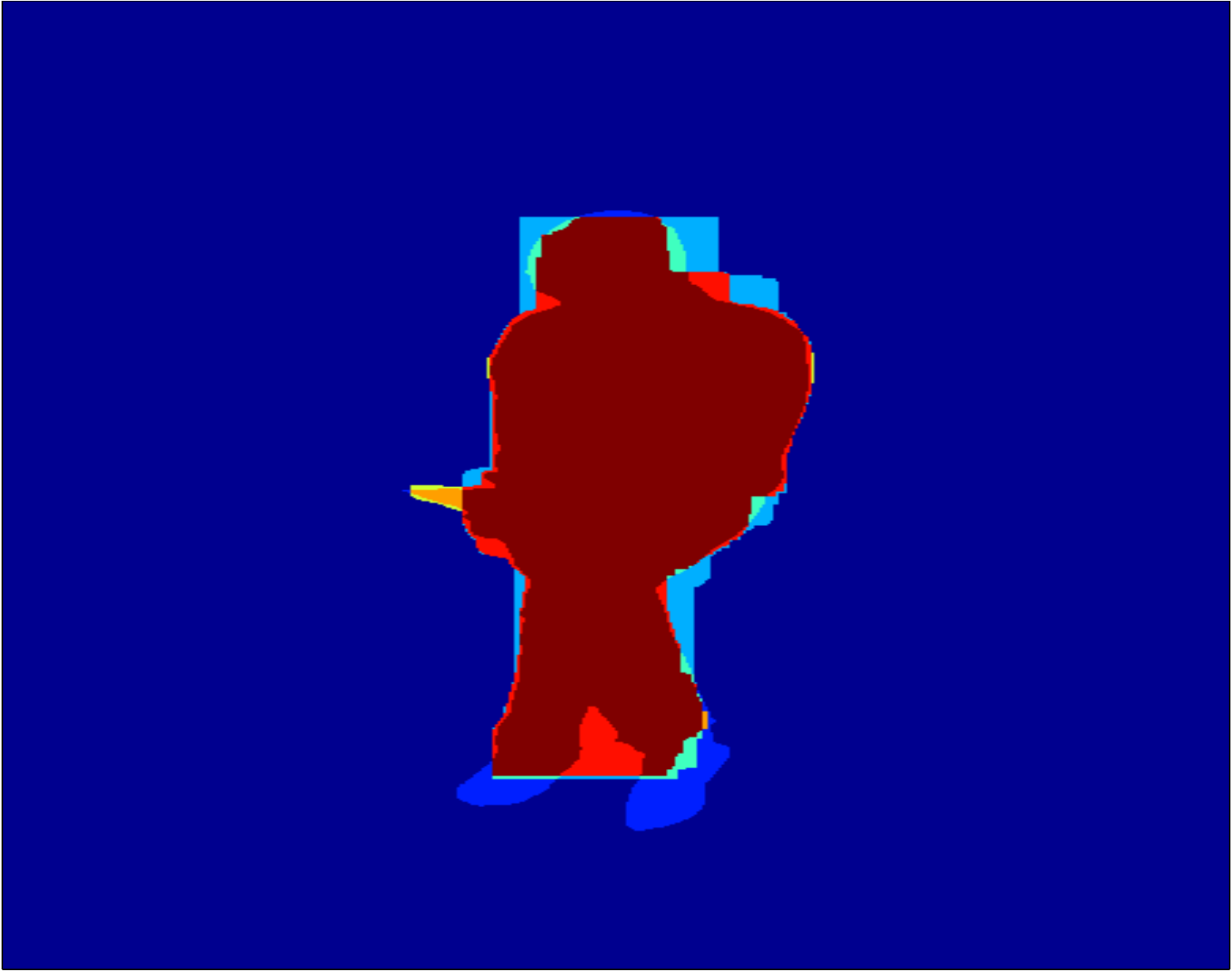}}
\subfigure[]{\includegraphics[width=0.225\linewidth]{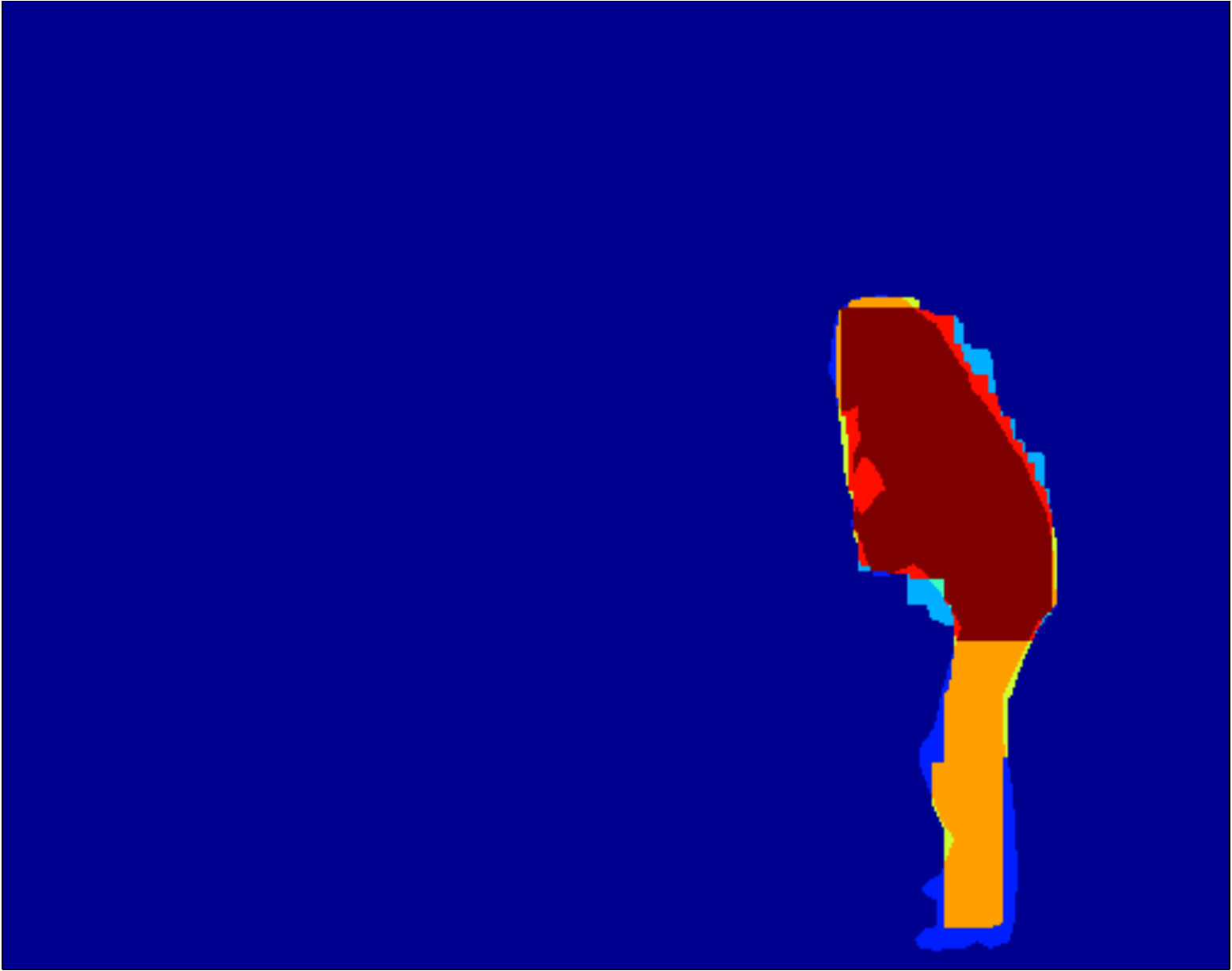}}
\subfigure[]{\includegraphics[width=0.225\linewidth]{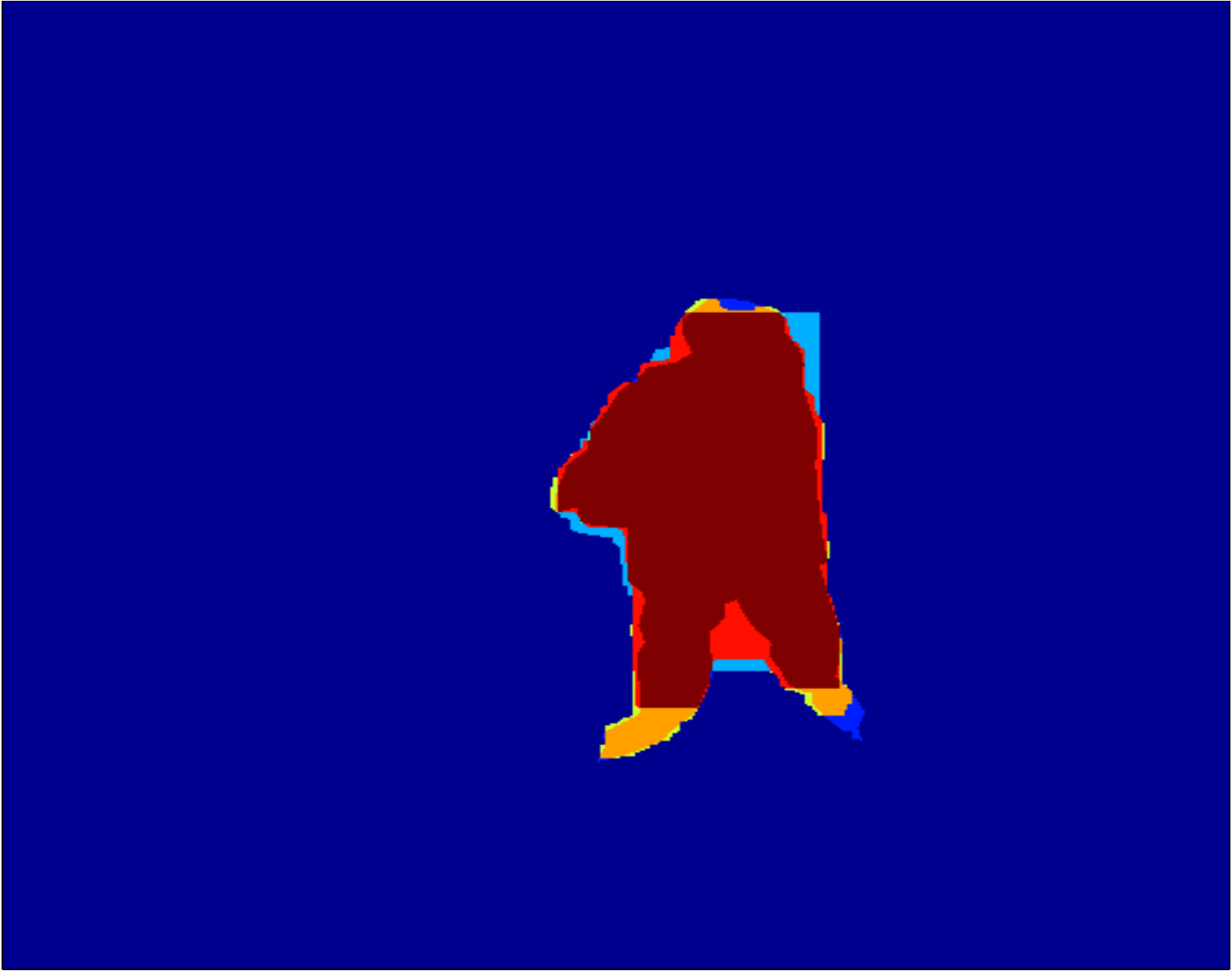}}
\subfigure[]{\includegraphics[width=0.225\linewidth]{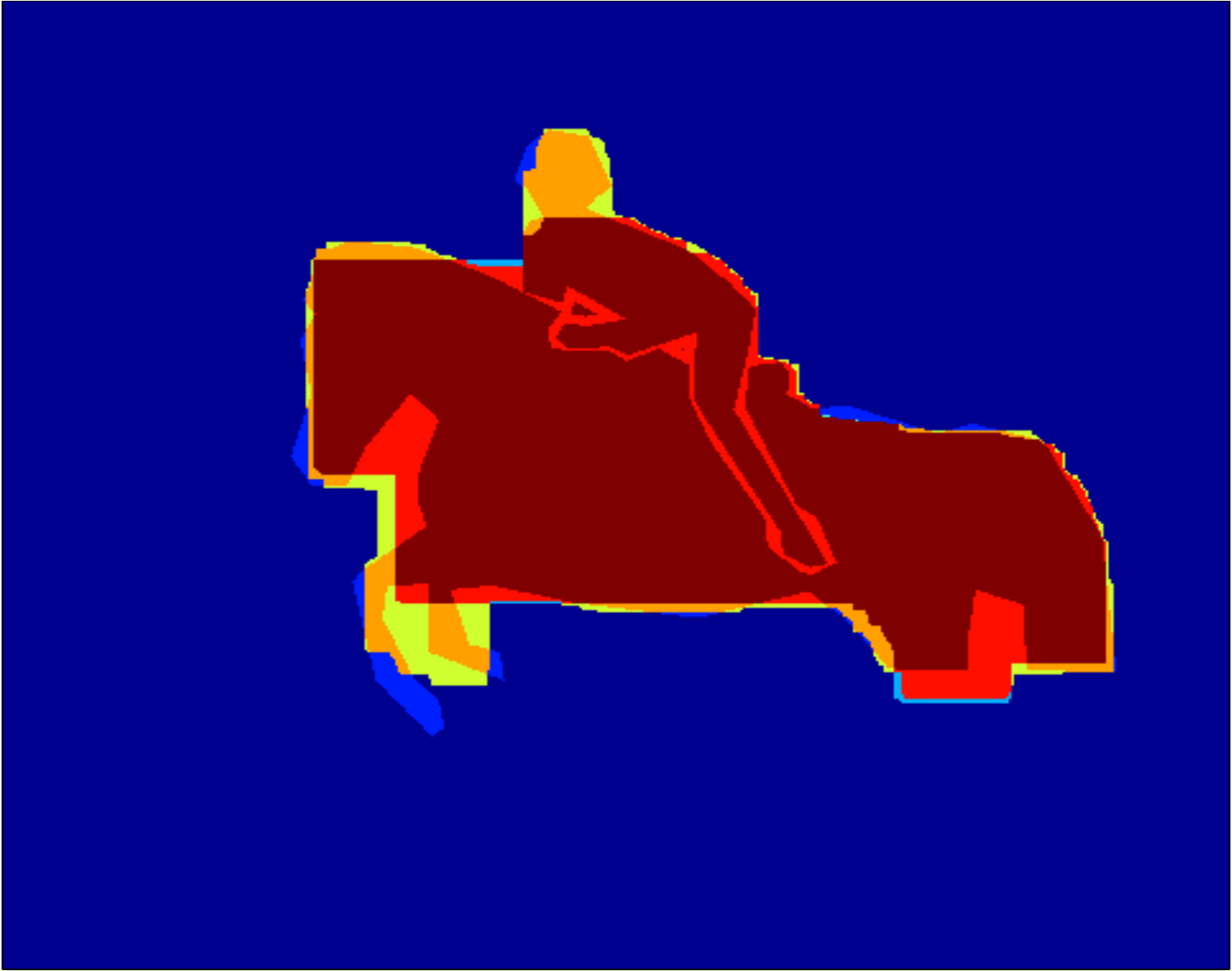}}
\subfigure[]{\includegraphics[width=0.225\linewidth]{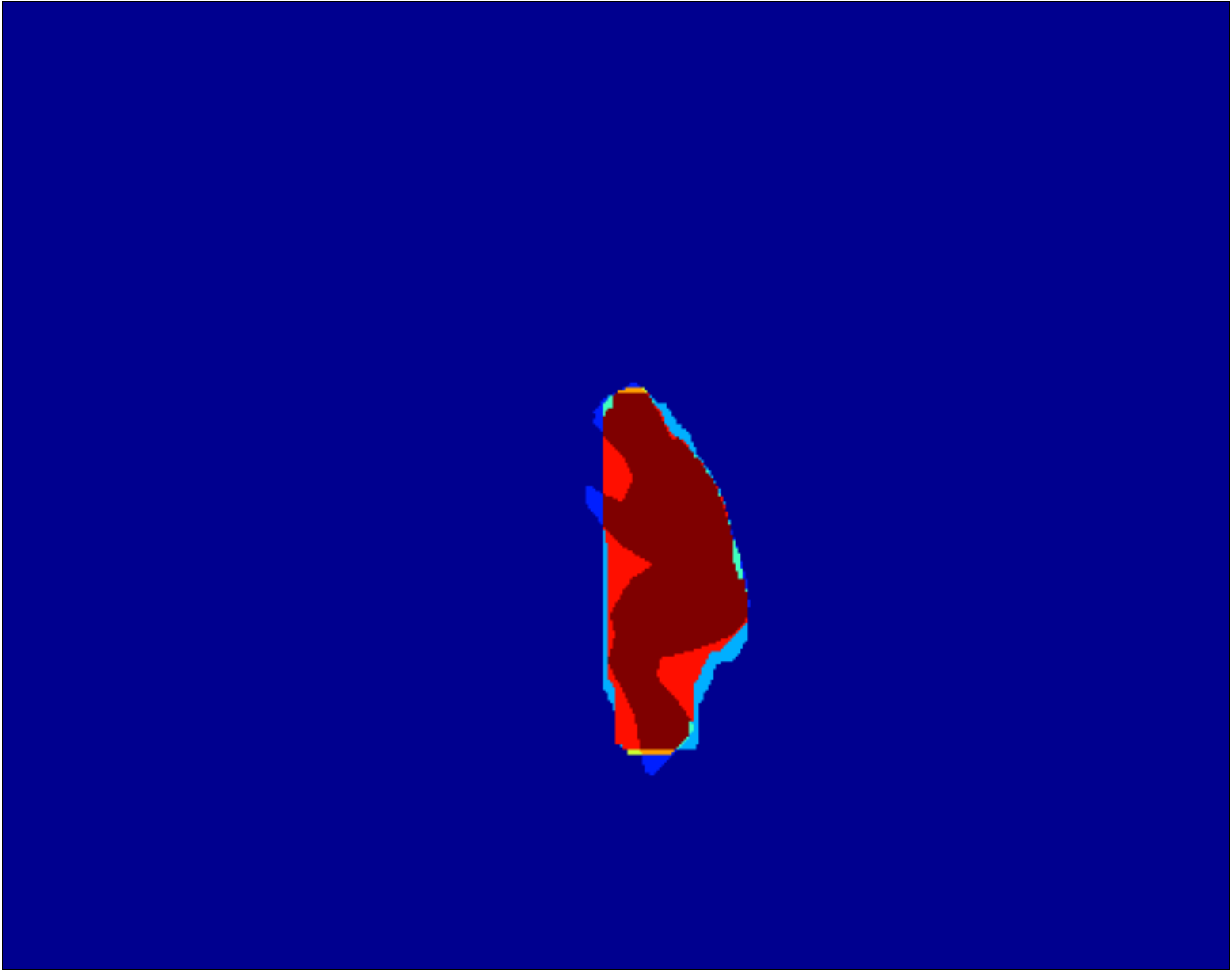}}
{\tiny
\begin{tabular}{p{0.05\linewidth}p{0.05\linewidth}p{0.05\linewidth}p{0.05\linewidth}p{0.05\linewidth}p{0.05\linewidth}p{0.05\linewidth}p{0.05\linewidth}}
\includegraphics[width=1\linewidth]{SegmentationComparisonLegend0-crop.pdf} &
\includegraphics[width=1\linewidth]{SegmentationComparisonLegend1-crop.pdf} &
\includegraphics[width=1\linewidth]{SegmentationComparisonLegend2-crop.pdf} &
\includegraphics[width=1\linewidth]{SegmentationComparisonLegend3-crop.pdf} &
\includegraphics[width=1\linewidth]{SegmentationComparisonLegend4-crop.pdf} &
\includegraphics[width=1\linewidth]{SegmentationComparisonLegend5-crop.pdf} &
\includegraphics[width=1\linewidth]{SegmentationComparisonLegend6-crop.pdf} &
\includegraphics[width=1\linewidth]{SegmentationComparisonLegend7-crop} \\
\raggedright
$g=-1$ $h=-1$ $s=-1$ & 
\raggedright
$g=+1$ $h=-1$ $s=-1$& 
\raggedright
$g=-1$ $h=+1$ $s=-1$& 
\raggedright
$g=+1$ $h=+1$ $s=-1$& 
\raggedright
$g=-1$ $h=-1$ $s=+1$& 
\raggedright
$g=+1$ $h=-1$ $s=+1$& 
\raggedright
$g=-1$ $h=+1$ $s=+1$& 
\raggedright
$g=+1$ $h=+1$ $s=+1$\\
\end{tabular}
}
\caption{\label{fig:segDiffBiconvex} A pixelwise comparison of the ground truth (denoted $g$ in the legend), the prediction from training with Hamming loss (denoted $h$), and the prediction when training with the biconvex loss in Equation~\eqref{eq:biconvexloss} (denoted $s$). The {\color{orange} orange} regions ($g=+1$, $h=-1$, and $s=+1$) show where the supermodular loss learns to correctly predict the foreground when Hamming loss fails; the {\color{cyan} cyan} regions ($g=-1$, $h=+1$, and $s=-1$) show where the supermodular loss learns to correctly predict the background when Hamming loss fails.
 } 
\end{figure*}

\section{Discussion and Conclusion}

A somewhat surprising result in Table~\ref{tab:loss} is that training with the supermodular loss results in better performance \emph{as measured by Hamming loss}.  This has been previously observed with a different loss function by \cite{pletscher2012learning,osokin2014perceptually}, and indicates that in the finite sample regime a supermodular likelihood can result in better generalization performance.  This holds, although the model space and regularizer were identical in both training settings.  We have observed the same effect with the other two supermodular loss functions, $\Delta_S$ and $\Delta_C$, indicating that this may be a broader property of supermodular loss functions.  

Our results in terms of computation time give clear evidence for the superiority of ADMM inference when a specialized optimization procedure is available for the loss function.  As shown in Figure~\ref{fig:timeBox}, the Fujishige-Wolfe minimum norm point algorithm does not scale to typical consumer images (i.e.\ several megapixels), which indicates that loss functions for which a specialized optimization procedure is not available are likely infeasible for pixel level image segmentation without unprecedented improvements in general submodular minimization.  Figure~\ref{fig:timeBox} shows that the log-log slope of the runtime for the min-norm point algorithm is higher than for ADMM, suggesting a worse computational compexity. 
One may wish to employ the result that early termination of the min-norm point algorithm gives a guaranteed approximation of the exact result, but even this is infeasible for images of the size considered here. 
In addition, Table~\ref{tab:ADMMvsLP} suggests that ADMM provides a more efficient strategy without lost of performance compared to using an LP-relaxation.
Joint graph-cuts optimization for loss augmented inference results in non-submodular pairwise potentials and graph-cuts fails to correctly minimize the joint energy.  As a result, a cutting plane optimization of the structured output SVM objective fails catastrophically, and the resulting accuracy is on par with a random weight vector.  Consequently, the ADMM technique yielded the \emph{only} feasible training strategy.


In this work, we propose three novel supermodular loss functions. We have shown that using supermodular loss functions achieves improved performance both in qualitative and quantitative terms on a binary segmentation task. 
We observe that a key advantage of the proposed supermodular losses over modular losses, e.g.\ Hamming loss, is an improved ability to find elongated regions such as heads and legs, or thin articulated structures in medical images.

Previous to our work, specialized inference procedures had to be developed for every model/loss pair, a time consuming process.  By contrast, we have proposed a Lagrangian splitting technique based on ADMM to perform general loss augmented inference.
We demonstrate the feasibility of the ADMM algorithm for loss augmented inference on an interactive foreground/background segmentation task, for which alternate strategies such as the Fujishige-Wolfe minimum norm point algorithm are infeasible.
Our proposed ADMM algorithm provides a strategy to solve the loss augmented inference as two separate subproblems. This provides an alternate API for the structured output SVM framework to that of SVMstruct \cite{tsochantaridis2005large}. We envision that this can be of use in a wide range of application settings, and an open source general purpose toolbox for this efficient segmentation framework with supermodular losses is available for download from \url{https://github.com/yjq8812/efficientSegmentation}.



\bibliographystyle{plain}      
\bibliography{iccvbib}   

\end{document}